\documentclass{article}
\usepackage[utf8]{inputenc}
\usepackage{mathmacros}
\usepackage[numbers]{natbib}
\usepackage{graphicx}
\usepackage{color}
\usepackage{enumerate}
\usepackage[margin=1.0in]{geometry}
\newcommand{\anova}{K_{\mathrm{A}}}
\usepackage{subfig}
\usepackage{url}
\usepackage{multirow}

\title{Factorization Machines with Regularization \\for Sparse Feature Interactions}
\author{Kyohei Atarashi, Satoshi Oyama, and Masahito Kurihara\\
Hokkaido University\\
katarashi0305@gmail.com,  \{oyama, kurihara\}@ist.hokudai.ac.jp}
\date{\today}

\begin{document}

\maketitle
\begin{abstract}%
    Factorization machines (FMs) are machine learning predictive models based on second-order feature interactions and FMs with sparse regularization are called sparse FMs.
    Such regularizations enable \emph{feature selection}, which selects the most relevant features for accurate prediction, and therefore they can contribute to the improvement of the model accuracy and interpretability.
    However, because FMs use second-order feature interactions, the selection of features often causes the loss of many relevant feature interactions in the resultant models.
    In such cases, FMs with regularization specially designed for \emph{feature interaction selection} trying to achieve interaction-level sparsity may be preferred instead of those just for feature selection trying to achieve feature-level sparsity.
    In this paper, we present a new regularization scheme for feature interaction selection in FMs.
    For feature interaction selection, our proposed regularizer makes the feature interaction matrix sparse without a restriction on sparsity patterns imposed by the existing methods.
    We also describe efficient proximal algorithms for the proposed FMs and how our ideas can be applied or extended to feature selection and other related models such as higher-order FMs and the all-subsets model.
    The analysis and experimental results on synthetic and real-world datasets show the effectiveness of the proposed methods.
\end{abstract}

\section{Introduction}
    \label{sec:intro}
    Factorization machines (FMs)~\citep{rendle2010factorization,rendle2012factorization} are machine learning predictive models based on second-order feature interactions, i.e., the products $x_i x_j$ of two feature values.
    One advantage of FMs compared with linear models with a polynomial term (namely, quadratic regression (QR)) and kernel methods is efficiency.
    The computational cost of evaluating FMs is linear with respect to the dimension of the feature vector $d$ and independent of the number of training examples $N$.
    Another advantage is that FMs can learn well even on a sparse dataset because they can estimate the weights for feature interactions that are not observed from the training dataset.
    These advantages are due to the low-rank matrix factorization modeling.
    In FMs, the weight matrix for feature interactions $\bm{W} \in \real^{d \times d}$ is factorized as $\bm{W}=\bm{P}\bm{P}^\top,\ \bm{P} \in \real^{d \times k}$, where $k \in \natu_{>0}$, called the rank hyperparameter, is usually far smaller than $d$.
 
    \emph{Feature selection} methods based on sparse regularization~\citep{tibshirani1996regression,friedman2010note} have been developed to improve the performance and interpretability of FMs~\citep{pan2016sparse,xu2016synergies,zhao2017meta}.
    When one uses feature selection methods, it is assumed that there are irrelevant features.
    However, because FMs use second-order feature interactions, not only feature-level but also interaction-level relevance should be considered.
    Actually, FMs with feature selection can work well only if all relevant interactions are those among a subset of features, but this is not necessarily the case in practice.
    In such cases, FMs with regularization specially designed for \emph{feature interaction selection} trying to achieve interaction-level sparsity may be preferred instead of those just for feature selection trying to achieve feature-level sparsity.
    Technically speaking, the existing methods select features by inducing row-wise sparsity in $\bm{P}$, often leading to undesiable row/column-wise sparsity in $\bm{W}$ as a result.
    We would like to remove such row/column-wise sparsity-pattern restrictions in the interaction-level sparsity modeling.

    In this paper, we present a new regularization scheme for feature interaction selection in FMs.
    The proposed regularizer is intended to make $\bm{P}$ sparse and $\bm{W}$ sparse without row/column-wise sparsity-pattern restrictions, which means that our proposed regularizer induces interaction-level sparsity in $\bm{W}$.
    Our basic objectives are to design regularizers of $\bm{P}$ imposing a penalty on the density (in some sense) of $\bm{W}$ and to develop efficient algorithms for the associated optimization problems.
    We will see that our regularizer comes from mathematical analyses of norms (of matrices) and their squares in association with upper bounds of the $\ell_1$ norm for $\bm{W}$.
    In addition, we will discuss how our ideas can be applied or extended to better feature selection (in terms of the prediction performance) and other related models such as higher-order FMs and the all-subsets model.
    The experiments on synthetic datasets demonstrated that one of the proposed methods (but no other existing ones) succeeded in feature interaction selection in FMs and all the proposed methods performed feature selection more accurately in the cases where all relevant interactions are those among a subset of features.
    Moreover, the experiments on real-world datasets demonstrated that the proposed methods tend to be easier to use and select more relevant interactions and features for predictions than the existing methods.
    
    This paper is organized as follows.
    In~\cref{sec:background}, we review FMs and sparse FMs.
    \cref{sec:proposed} presents our basic idea and analyses for constructing regularizers that make $\bm{W}$ sparse.
    In accordance with them, we propose a new regularizer for feature interaction selection in~\cref{sec:ti} and then another regularizer for better feature selection (in terms of the prediction performance) in~\cref{sec:cs}, as well as efficient proximal optimization methods for these proposed regularizers.
    We extend the proposed regularizers to related models: higher-order FMs and the all-subsets model~\citep{blondel2016higher} in~\cref{sec:extension}.
    In~\cref{sec:related}, we discuss related work. 
    In~\cref{sec:experiments}, we provide the experimental results on two synthetic and four real-world datasets before concluding in~\cref{sec:conclusion}.
    
    \begin{table}[t]
        \centering
        \caption{A summary of sparse regularizers for FMs. $\ell_{2,1}$~\citep{xu2016synergies,zhao2017meta} and $\ell_1$~\citep{pan2016sparse} are those proposed in the existing methods, and $\tilde{\ell}_{1,2}^2$ (TI) and $\ell_{2,1}^2$ (CS) are those proposed in this paper.
        The TI regularizer is for feature interaction selection and the CS regularizer is for accurate feature selection. TI is an abbreviation of triangular inequality, and CS is an abbreviation of Cauchy-Schwarz}.
        \label{tab:summary}
        \begin{tabular}{c|lcc}
             \multirow{2}{*}{Regularizer} & \multicolumn{1}{c}{\multirow{2}{*}{Formula}} & Feature & Feature\\
             & & interaction selection & selection \\ \hline
             $\ell_{2,1}$ & $\norm{\bm{P}}_{2,1}=\sum_{j=1}^d \norm{\bm{p}_{j}}_2$ & No & Yes\\
             $\ell_1$ & $\norm{\bm{P}}_1 = \sum_{j=1}^d \sum_{s=1}^k \abs{p_{j,s}}$ & No & Yes\\
             $\tilde{\ell}_{1,2}^2$ (TI) & $\norm{\bm{P}^\top}_{1,2}^2 = \sum_{s=1}^k \norm{\bm{p}_{:, s}}_1^2$ & Yes & Yes\\
             $\ell_{2,1}^2$ (CS) & $\norm{\bm{P}}_{2,1}^2=\left(\sum_{j=1}^d \norm{\bm{p}_{j}}_2\right)^2$ & No & Yes
        \end{tabular}
    \end{table}
    \Cref{tab:summary} summarizes existing and proposed methods (regularizers).
    
    \paragraph{Notation.} We denote $\{1, 2, \ldots, n\}$ as $[n]$.
    We use $\circ$ for the element-wise product (a.k.a Hadamard product) of the vector and matrix.
    We denote the $\ell_p$ norm for vector and matrix as $\norm{\cdot}_p$.
    Given a matrix $\bm{X}$, we use $\bm{x}_i$ for the $i$-th row vector and $\bm{x}_{:, i}$ for the $i$-th column vector.
    Given a matrix $\bm{X} \in \real^{n \times m}$, we denote the $\ell_q$ norm of the vector $(\norm{\bm{x}_1}_p, \ldots, \norm{\bm{x}_n}_p)^\top$ by $\norm{\bm{X}}_{p, q} \coloneqq \left(\sum_{i=1}^{n}\norm{\bm{x}_i}^{q}_{p}\right)^{1/q}$ and call it $\ell_{p, q}$ norm.
    We use the terms $\tilde{\ell}_p$ and $\tilde{\ell}_{p, q}$ norm for $\ell_p$ and $\ell_{p, q}$ norm for the transpose matrix, i.e., $\norm{\bm{P}^\top}_p$ and $\norm{\bm{P}^\top}_{p, q}$, respectively.
    For the number of non-zero elements in vectors ($|\{i: x_i \neq 0\}|$) and matrices ($|\{(i, j): x_{i, j} \neq 0\}|$), we use $\nnz{\cdot}$.
    We define $\supp(\bm{x})$, called the support for $\bm{x}$,  as the indices of non-zero elements in $\bm{x} \in \real^d$: $\{i \in [d] : x_i \neq 0\}$.
    We define $\absop(\bm{x}): \bm{x} \in \real^n \to \real^{n}_{\ge 0}$ as $\absop(\bm{x}) = (\abs{x_1}, \ldots, \abs{x_n})^\top$.
    
\section{Factorization Machines and Sparse Factorization Machines}
    \label{sec:background}
    In this section, we briefly review FMs~\citep{rendle2010factorization,rendle2012factorization}, sparse FMs~\citep{pan2016sparse,xu2016synergies,zhao2017meta}, and a sparse and low-rank quadratic regression.

\subsection{Factorization Machines}
    FMs~\citep{rendle2010factorization,rendle2012factorization} are models for supervised learning based on second-order feature interactions.
    For a given feature vector $\bm{x} \in \real^d$, FMs predict the target of $\bm{x}$ as 
    \begin{align}
        f_{\mathrm{FM}}(\bm{x}; \bm{w}, \bm{P}) \coloneqq \inner{\bm{w}}{\bm{x}} + \sum_{j_2>j_1} \inner{\bm{p}_{ j_1}}{\bm{p}_{j_2}}x_{j_1}x_{j_2} = \inner{\bm{w}}{\bm{x}} + \frac{1}{2}\sum_{j_1=1}^d \sum_{j_2 \in [d]\setminus \{j_1\}} \inner{\bm{p}_{ j_1}}{\bm{p}_{j_2}}x_{j_1}x_{j_2},
        \label{eq:fm} 
    \end{align}
    where $\bm{w} \in \real^d$ and $\bm{P} \in \real^{d \times k}$ are learnable parameters, and $k \in \natu_{>0}$ is the rank hyperparameter.
    The first term in~\eqref{eq:fm} represents the linear relationship, and the second term represents the second-order polynomial relationship between the input and target.
    For a given training dataset $\mathcal{D} = \{(\bm{x}_n, y_n)\}_{n=1}^N$, the objective function of the FM is
    \begin{align}
        L_{\mathrm{FM}}(\bm{w}, \bm{P}; \lambda_{w}, \lambda_{p}) \coloneqq \frac{1}{N}\sum_{n=1}^N \ell (f_{\mathrm{FM}}(\bm{x}_n), y_n) + \lambda_{w} \norm{\bm{w}}_2^2 + \lambda_{p} \norm{\bm{P}}_2^2, 
        \label{eq:objective_fm}
    \end{align}
    where $\ell: \real \times \real \to \real_{\ge 0}$ is the $\mu$-smooth (i.e., its derivative is a $\mu$-Lipschitz) convex loss function, and $\lambda_{w}, \lambda_{p} \ge 0$ are the regularization-strength hyperparameters.

    The inner product of the $j_1$-th and $j_2$-th row vectors in $\bm{P}$, $\inner{\bm{p}_{j_1}}{\bm{p}_{j_2}}$, corresponds to the weight for the interaction between the $j_1$-th and $j_2$-th features in the FM. Thus, FMs are equivalent to the following linear model with a second-order polynomial term (we call it (distinct) quadratic regression (QR) in this paper) with factorization of the feature interaction matrix $\bm{W}=\bm{P}\bm{P}^\top$:
    \begin{align}
        f_{\mathrm{QR}}(\bm{x}; \bm{w}, \bm{W}) = \inner{\bm{w}}{\bm{x}} + \sum_{j_2 > j_1} w_{j_1, j_2}x_{j_1}x_{j_2} = \inner{\bm{w}}{\bm{x}} + \frac{1}{2}\sum_{j_1=1}^d\sum_{j_2 \in [d] \setminus \{j_1\}} w_{j_1, j_2}x_{j_1}x_{j_2},
        \label{eq:quadratic}
    \end{align}
    where $\bm{W} \in \real^{d \times d}$ is the feature interaction matrix.
    The computational cost for evaluating FMs is $O(\nnz{\bm{x}}k)$, i.e., it is linear w.r.t the dimension $d$ of feature vectors, because the second term in Equation~\eqref{eq:fm} can be rewritten as
    \begin{align}
        \sum_{j_2>j_1} \inner{\bm{p}_{ j_1}}{\bm{p}_{j_2}}x_{j_1}x_{j_2} =  \sum_{s=1}^k(\inner{\bm{p}_{:, s}}{\bm{x}}^2 - \inner{\bm{p}_{:, s}\circ \bm{p}_{:, s}}{\bm{x}\circ \bm{x}})/2.
    \end{align}
    On the other hand, the QR clearly requires $O\left(\nnz{\bm{x}}^2\right)$ time and $O\left(d^2\right)$ space for storing $\bm{W}$, which is prohibitive for a high-dimensional case.
    Moreover, this factorized representation enables FMs to learn the weights for unobserved feature interactions but the QR does not learn such weights~\citep{rendle2010factorization}.
    
    The objective function in Equation~\eqref{eq:objective_fm} is differentiable, so \citet{rendle2010factorization} developed a stochastic gradient descent (SGD) algorithm for minimizing~\eqref{eq:objective_fm}.
    Although the objective function is non-convex w.r.t $\bm{P}$, it is multi-convex w.r.t $\bm{p_{j}}$ for all $j \in [d]$.
    It can thus be efficiently minimized by using a coordinate descent (CD) (a.k.a alternating least squares) algorithm~\citep{rendle2011fast,blondel2016polynomial}.
    Both the SGD and CD algorithms require $O(\nnz{\bm{X}}k)$ time per epoch (using all instances at one time in the SGD algorithm and updating all parameters at one time in the CD algorithm), where $\bm{X} \in \real^{N \times d}$ is the design matrix.
    It is linear w.r.t both the number of training examples $N$ and the dimension of feature vector $d$.
    
\subsection{Sparse Factorization Machines}
    Feature selection methods based on sparse regualization~\citep{tibshirani1996regression,yuan2006model,bach2011optimization} have been developed to improve the performance and interpretability of FMs~\citep{pan2016sparse,xu2016synergies,zhao2017meta}.
    Selecting features necessarily means making the weight matrix $\bm{W}$ row/column-wise sparse.
    
    \citet{xu2016synergies} and \citet{zhao2017meta} proposed using $\norm{\cdot}_{2, 1}$ regularization, it is well-known as group-lasso regularization~\citep{friedman2010note,yuan2006model}.
    We call FMs with this regularization $\ell_{2,1}$-sparse FMs.
    The objective function of this FM is $L_{\mathrm{FM}}(\bm{w}, \bm{P}; \lambda_{w}, \lambda_{p}) + \tilde{\lambda}_{p} \norm{\bm{P}}_{2, 1}$, where $\tilde{\lambda}_p \ge 0$ is the regularization hyperparameter.~\footnote{There are several differences between our formulations of $\ell_{2,1}$-sparse FMs and the original ones. First, the original formulations~\citep{xu2016synergies,zhao2017meta} did not introduce the standard $\ell_2^2$ regularization $\lambda_{p} \norm{P}_2^2$ while ours do because setting $\lambda_{p}$ to zero reproduces the original formulations.
    Moreover, they modified the models or assumed some additional information according to some domain specific knowledge.
    We do not modify the models and do not assume such information since we do not specify any application.}
    \citet{xu2016synergies} and \citet{zhao2017meta} proposed the proximal block coordinate descent (PBCD) and proximal gradient descent (PGD) algorithms respectively, for minimizing this objective function.
    In our setting, at each iteration, the PBCD algorithm updates the $j$-th row vector $\bm{p}_j$ by
    \begin{align}
        \bm{p}_j \leftarrow \prox_{\eta\tilde{\lambda}_{p}\norm{\cdot}_2}(\bm{p}_j - \eta \nabla_{\bm{p}_j}L)&=\argmin_{\bm{q}}\frac{1}{2}\norm{\bm{q} - (\bm{p}-\eta \nabla_{\bm{p}_j}L)}_2^2 + \eta \tilde{\lambda}_{p} \norm{\bm{q}}_2\\
        &= \max (1 - \eta\tilde{\lambda}_{p}/\norm{\bm{p}'_j}_2, 0) \cdot \bm{p}'_j,
        \label{eq:prox_l21}
    \end{align}
    where $\eta > 0$ is the step size parameter and $\bm{p}'_j \coloneqq \bm{p}_j - \eta \nabla_{\bm{p}_j}L_{\mathrm{FM}}(\bm{P})$.
    This proximal algorithm produces row-wise sparse parameter matrix $\bm{P}$.
    When $\bm{p}_j=\bm{0}$, $\inner{\bm{p}_j}{\bm{p}_i}$ clearly equals zero for all $i \in [d]$.
    This means that the feature interaction matrix $\bm{W}=\bm{P}\bm{P}^\top$ is row/column-wise sparse, so the FM ignores all feature interactions that involve $j$-th feature, i.e., $\ell_{2,1}$ regularizer enables feature selection in FMs.
    
    \citet{pan2016sparse} proposed using $\ell_1$ (=$\ell_{1,1}$) regularization for $\bm{P}$. 
    We call FMs with $\ell_1$ regularized objective function $\ell_1$-sparse FMs~\footnote{Strictly speaking,~\citet{pan2016sparse} considered probabilistic FMs and proposed the use of a Laplace prior. It corresponds to a $\ell_1$ regularization in our non-probabilistic formulation.}.
    The PGD update for $p_{j, s}$ in $\ell_1$-sparse FMs is given by
    \begin{align}
        p_{j, s} \leftarrow \prox_{\eta\tilde{\lambda}_{p}\abs{\cdot}}(p'_{j, s}) &=\argmin_{q} \frac{1}{2}(q - p'_{j,s})^2 + \eta \tilde{\lambda}_{p}\abs{q} \\
        &=\sign(p'_{j, s}) \cdot \max (\abs{p'_{j,s}} - \eta\tilde{\lambda}_{p}, 0).
        \label{eq:prox_l1}
    \end{align}
    $\ell_1$-sparse FMs are intended to make not feature interaction matrix $\bm{W}=\bm{P}\bm{P}^\top$ row/column-wise sparse but $\bm{P}$ sparse.
    However, they practically work well for feature selection in FMs~\citep{pan2016sparse}.

\section{Proposed Scheme for Feature Interaction Selection}
    \label{sec:proposed}
    \paragraph{Feature Interaction Selection: its Motivation.} The existing sparse regularizers~\citep{pan2016sparse,zhao2017meta,xu2016synergies} can improve the performance and interpretability of FMs by selecting only relevant features.
    However, because FMs use second-order feature interactions, not feature-level but interaction-level relevance should be considered, in other words, feature interaction selection is preferable to feature selection.
    Assume that a subset $S$ of $[d]$ is selected as a set of relevant features.
    Then, for all $j \in [d] \setminus S$, all feature interactions with $j$-th feature are lost but they can contain important feature interactions.
    Moreover, FMs use all feature interactions from $S$ and they can contain irrelevant feature interactions.
    Therefore, the existing methods tend to produce all-zeros $\bm{W}$ to remove all irrelevant interactions or all-non-zeros (dense) $\bm{W}$ to select all relevant interactions.
    Many relevant interactions are lost in the former case and many irrelevant interactions are used in the latter case.
    Actually, FMs with feature selection can work well only if all relevant interactions are those among a subset of features, but this is not necessarily the case in practice.
    
    In this section, firstly, we briefly verify whether the existing regularizers based on sparsity-inducing norms can select feature interactions or not, experimentally.
    We next introduce a preferable but hard to optimize regularizer $\Omega_*$ for feature interaction selection in FMs.
    Then, we present a relationship between norms and $\Omega_*$.
    We next present a relationship between \emph{squares} of norms and $\Omega_*$, and finalize our scheme: using the square of a \emph{sparsity-inducing} (quasi-)norm.

\subsection{Can $\ell_1$ and $\ell_{2,1}$ regularizers select feature interactions?}
\label{subsec:verification}
    \begin{figure}[t]
        \centering
        \subfloat{
        \includegraphics[width=80mm]{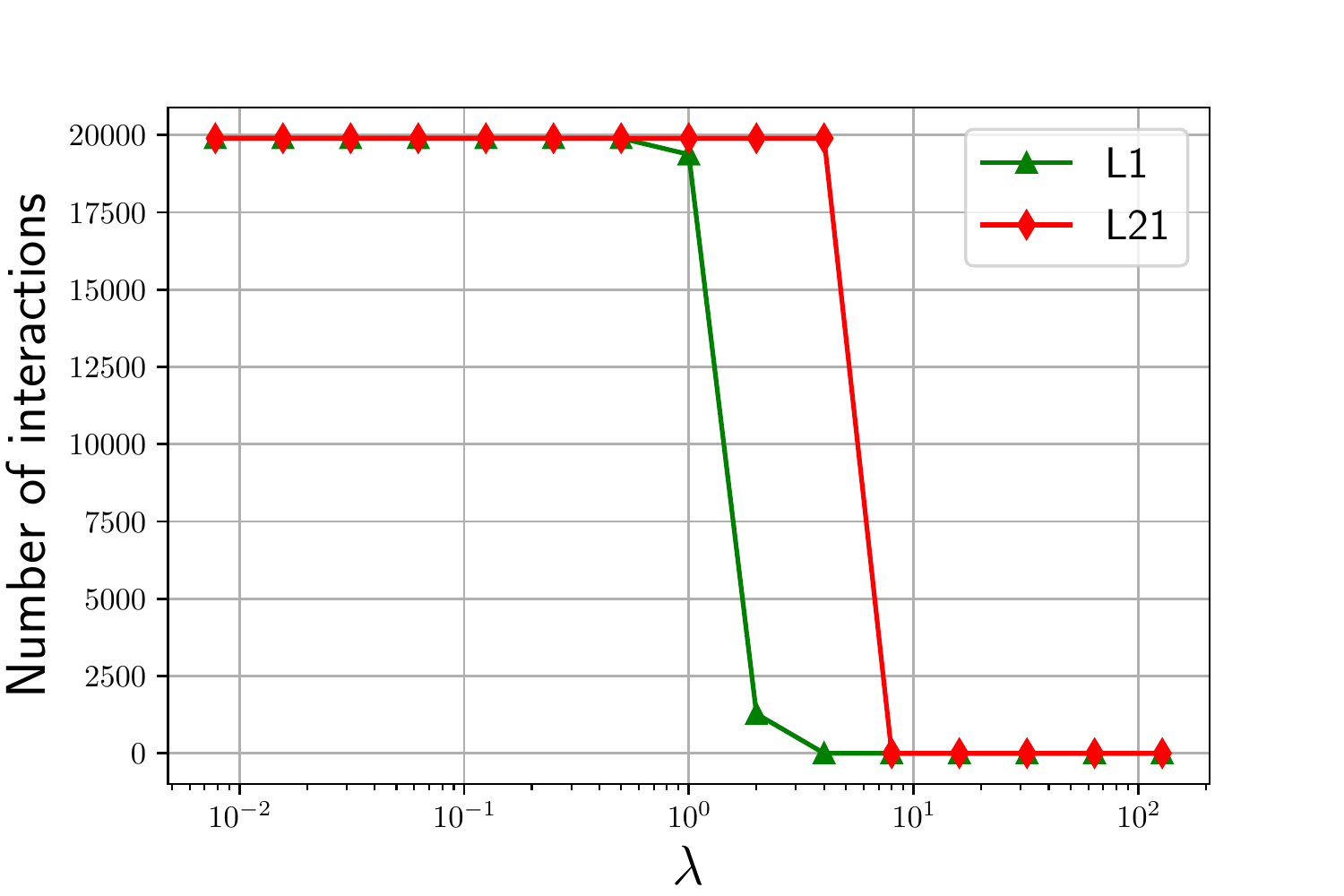} 
        \includegraphics[width=80mm]{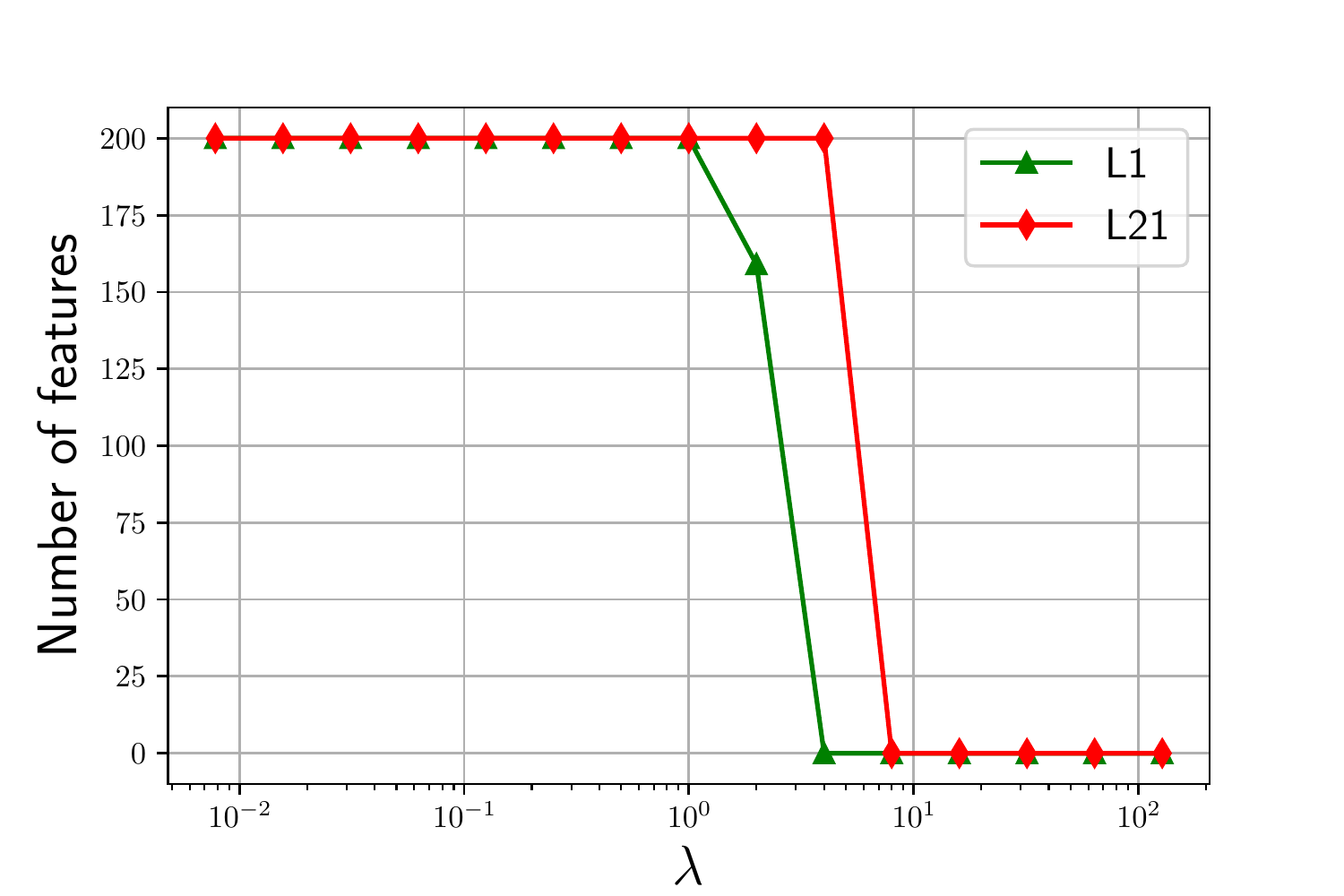}}
        \caption{Comparison of proximal operators associated with \textbf{L1}, and \textbf{L21} regularizer. We evaluated the proximal operators at a randomly sampled $\bm{P}$ with various $\lambda$. Left graph shows the number of used feature interactions and right graph shows the number of used features in $\bm{Q}^*(\bm{Q}^*)^\top$, where $\bm{Q}^*$ is the output of the proximal operator: $\bm{Q}^* = \prox_{\lambda \Omega(\cdot)}(\bm{P})$ and $\Omega$ is $\ell_1$ (\textbf{L1}~\citep{pan2016sparse}) or $\ell_{2,1}$ (\textbf{L21}~\citep{xu2016synergies,zhao2017meta}).}
        \label{fig:toy}
    \end{figure}
    We verify whether the existing regularizers can select feature interactions or not.
    Because the objective functions with the existing regularizers are typically optimized by the PGD algorithm, we compared the output of their proximal operators for verification.
    We sampled $\bm{P} \in \real^{200 \times 30}$ with $p_{j,s} \sim \mathcal{N}(0, 1.0)$ for all $j \in [200]$, $s \in [30]$, and evaluated proximal operators $\bm{Q}^* = \prox_{\lambda \Omega(\cdot)}(\bm{P})$ with various $\Omega$: $\ell_1$ (\textbf{L1}~\citep{pan2016sparse}) and $\ell_{2,1}$ (\textbf{L21}~\citep{xu2016synergies,zhao2017meta}).
    Their corresponding proximal operators are~\eqref{eq:prox_l1} and~\eqref{eq:prox_l21} respectively.
    Regarding $\bm{P}$ as the parameter of FMs, we computed the number of used interactions (i.e., $\abs{\{(j_1, j_2) : \inner{\bm{q}_{j_1}^*}{\bm{q}_{j_2}^*}\neq 0, j_2>j_1\}})$ and the number of used features (i.e., the number of non-zero rows in $\bm{Q}^*$).
    We set $\lambda$ to be $2^{-7}, 2^{-6}, \ldots$, and $2^{7}$.
    
    Results are shown in~\cref{fig:toy}.
    Both \textbf{L1} and \textbf{L21} tended to produce completely dense (all-non-zeros) or all-zeros feature interaction matrices.
    This result indicates that it is difficult for the existing methods to select feature interactions.
    We will see later in~\cref{subsec:comparison_proximal_operator}, the proximal operators of the proposed regularizers can produce moderately sparse feature interaction matrices and therefore more useful for feature interaction selection and feature selection in FMs.
    Moreover, we will show that one of the proposed regularizers can select relevant feature interactions and all of the proposed regularizers can select relevant features in~\cref{subsec:synthetic}.

\subsection{$\ell_1$ Norm for Feature Interaction Weight Matrix}
    \label{subsec:exactl1}
    We here introduce a preferable but hard to optimize regularizer for feature interaction selection in FMs.
    
    Selecting feature interactions necessarily means making the feature interaction weight matrix $\bm{W}$ sparse, so our goal is learning sparse $\bm{W}=\bm{P}\bm{P}^\top$ in FMs.
    Although the existing regularizers are intended to make $\bm{P}$ sparse, the sparsity of $\bm{P}$ does not necessarily imply the sparsity of $\bm{W}$.
    Thus, our basic idea is to use a regularization inducing sparsity $\bm{W}$ rather than $\bm{P}$.
    Especially, we propose  using $\ell_1$ regularization~\citep{tibshirani1996regression} for the strictly upper triangular elements (or equivalently, non-diagonal elements) in $\bm{W}$, i.e.,
    \begin{align}
        \sum_{j_2>j_1}\abs{w_{j_1, j_2}} = \sum_{j_2>j_1}\abs{\inner{\bm{p}_{j_1}}{\bm{p}_{j_2}}} \eqqcolon \Omega_*(\bm{P}),
    \end{align}
    because $\ell_1$ regularization is the well-known and one of the promising regularization for inducing sparsity.
    The corresponding objective function is
    \begin{align}
        L_{\mathrm{FM}}^*(\bm{w}, \bm{P}; \lambda_{w}, \lambda_{p}, \tilde{\lambda}_p) \coloneqq L_{\mathrm{FM}}(\bm{w}, \bm{P}; \lambda_{w}, \lambda_{p}) + \tilde{\lambda}_{p} \Omega_*(\bm{P}).
        \label{eq:objective_omegastar}
    \end{align}
    Unfortunately, this objective function is hard to optimize w.r.t $\bm{P}$.
    In the following, we introduce three well-known algorithms for minimizing a sum of a differentiable loss and a non-smooth regularization like~\eqref{eq:objective_omegastar} and show that the use of them is unrealistic.
    
    \paragraph{Subgradient Descent Algorithm.}
    Consider the use of the subgradient descent (SubGD) algorithm for minimizing~\eqref{eq:objective_omegastar}.
    $\Omega_*$ is non-convex and thus its subdifferential cannot be defined.
    Fortunately, $\norm{\bm{P}\bm{P}^\top}_1/2 = \Omega_*(\bm{P}) + \norm{\bm{P}}_2^2 / 2$ is convex, so its subdifferential can be defined.
    Therefore, consider the minimization of $L_{\mathrm{FM}}^*(\bm{w}, \bm{P}; \lambda_{w}, \lambda_{p} + \tilde{\lambda}_p/2, \tilde{\lambda}_p)$, i.e.,
    \begin{align}
         L_{\mathrm{FM}}^*(\bm{w}, \bm{P}; \lambda_{w}, \lambda_{p} + \tilde{\lambda}_p/2, \tilde{\lambda}_p) = L_{\mathrm{FM}}(\bm{w}, \bm{P}; \lambda_{w}, \lambda_{p}) + \frac{\tilde{\lambda}_p}{2}\norm{\bm{P}\bm{P}^\top}_1.
        \label{eq:objective_omegastar_convex}
    \end{align}
    At each iteration, the SubGD algorithm for minimizing~\eqref{eq:objective_omegastar_convex} picks a subgradient $\bm{G} \in  \partial_{\bm{P}} \norm{\bm{P}\bm{P}^\top}_1$ and updates the parameter $\bm{P}$ as
    \begin{align}
        \bm{P} \leftarrow \bm{P} - \eta \left(\nabla_{\bm{P}} L + \frac{\tilde{\lambda}_p}{2}\bm{G}\right).
        \label{eq:subgd}
    \end{align}
    The subdifferential of $\norm{\bm{P}\bm{P}^\top}_1$ is defined as~\citep{li2020nonconvex}
    \begin{align}
        \partial\norm{\bm{P}\bm{P}^\top}_1 = \{2\bm{Z} \bm{P}: \bm{Z} \in \partial_{\bm{S}} \norm{\bm{S}}_1, \bm{S}=\bm{P}\bm{P}^\top\}.
    \end{align}
    Therefore, picking a subgradient $\bm{G} \in \partial_{\bm{P}} \norm{\bm{P}\bm{P}^\top}_1$ requires $O(d^2k)$ computational cost (for computing $\bm{P}\bm{P}^\top$), so it might be prohibitive to use the SubGD algorithm for a high-dimensional case.
    To be more precise, the computational cost of the SubGD algorithm at each iteration is $O(T(\nnz{\bm{X}}k + d^2k))$, where $T$ is the number of line search iterations.
    Moreover, in general, the SubGD algorithm cannot produce a sparse solution and therefore it is not suitable for feature interaction selection~\citep{bach2011optimization}.
    
    \paragraph{Inexact PGD Algorithm.}
    To obtain a sparse $\bm{W}$, we consider the use of a PGD algorithm for~\eqref{eq:objective_omegastar_convex}.
    At each iteration, the PGD algorithm for minimizing~\eqref{eq:objective_omegastar_convex} requires the evaluation of the following proximal operator
    \begin{align}
        \prox_{\lambda \norm{\cdot \cdot^\top}_1}(\bm{P})\coloneqq \argmin_{\bm{Q} \in \real^{d\times k}}\frac{1}{2}\norm{\bm{P}-\bm{Q}}_2^2 + \lambda \norm{\bm{Q}\bm{Q}^\top}_1
        \label{eq:prox_omegastar}
    \end{align}
    with some $\lambda > 0$.
    This proximal problem~\eqref{eq:prox_omegastar} is convex but unfortunately cannot be evaluated analytically.
    The PGD algorithm with \textit{inexact} evaluation of proximal operator is called inexact PGD algorithm and we here consider the use of the inexact PGD algorithm~\citep{gu2018inexact,yao2017efficient,schmidt2011convergence}.
    Because~\eqref{eq:prox_omegastar} is a $(d \times k)$-dimensional non-smooth convex optimization problem, for example, it can be optimized using the SubGD algorithm~\citep{gu2018inexact}.
    Unfortunately, the SubGD algorithm for~\eqref{eq:prox_omegastar} requires $O(d^2 k)$ computational cost at each iteration, which is free from $\bm{X}$ but depends on $d$ quadratically.
    Thus, one iteration of this inexact PGD algorithm for~\eqref{eq:objective_omegastar_convex} using the SubGD algorithm for inexact evaluation of~\eqref{eq:prox_omegastar} takes $O(T(\nnz{\bm{X}}k + T'd^2k))$ time, where $T$ is the number of line search iteration and $T'$ is the number of iterations of the SubGD algorithm for~\eqref{eq:prox_omegastar}.
    Moreover, the precision of the inexact proximal operator should be high in practice and must be controlled carefully for convergence.
    Furthermore, the convergence rate of the SubGD algorithm for a convex optimization problem is $O(1/\epsilon^2)$~\citep{nesterov2018lectures}.
    Thus, we must set $T'$ to be a large value for a good solution, and then it is also not practical for~\eqref{eq:objective_omegastar_convex} to use the inexact PGD algorithm for a high-dimensional case.
    
    \paragraph{Inexact PCD Algorithm.}
    If we assume the use of proximal CD (PCD) algorithm for the regularized objective~\eqref{eq:objective_omegastar}, the proximal operator evaluated at $p'_{j, s} \coloneqq p_{j, s} - \eta \partial L_{\mathrm{FM}}(\bm{P})/\partial p_{j, s}$ is 
    \begin{align}
        \argmin_{q} \Biggl\{\frac{1}{2}\left(q - p'_{j, s}\right)^2 + \eta\tilde{\lambda}_{p} \sum_{i \in [d]\setminus\{j\}}\Bigl|qp_{i, s} + r_{j, i, s}\Bigr|\Biggr\}
        \label{eq:prox_l1_for_interaction_weights},
    \end{align}
    where $r_{j,i,s}=\sum_{s' \in [k]\setminus\{s\}}p_{j,s'}p_{i,s'}$.
    The second term in Equation~\eqref{eq:prox_l1_for_interaction_weights} takes the form of the sum of the absolute deviations and can be rewritten as a $d$-dimensional linear programming problem with inequality constraints~\citep{boyd2004convex}.
    Thus, the optimization problem in this proximal operator is a typical quadratic programming problem and can be solved by some well-known methods (e.g., an interior-point method).
    Alternatively, one can also solve~\eqref{eq:prox_l1_for_interaction_weights} using the SubGD algorithm or the alternating direction of direction method of multipliers (ADMM) algorithm~\citep{boyd2011distributed}.
    However, in any case, $r_{i, j, s}$ must be computed for all $i \in [d] \setminus \{j\}$ and it requires $O(dk)$ computational cost.
    Thus, the inexact PCD algorithm for~\eqref{eq:objective_omegastar} requires $\Omega(dk(dk))=\Omega(d^2k^2)$ computational cost only for evaluating~\eqref{eq:prox_l1_for_interaction_weights} per epoch.
    It might be prohibitive for a high-dimensional case.

\subsection{Upper Bound Regularizers of $\Omega_*$}
    \label{subsec:squares_of_norms}
    As described above, $\Omega_*$ regularizer seems appropriate for feature interaction selection but unfortunately it is hard to optimize.
    Thus, we consider the use of an upper bound regularizer being easy to optimize and ensuring sparsity to $\bm{W}$.
    The use of an easy-to-optimize upper bound is a common approach for minimizing a hard-to-optimize objective function~\citep{zhou2012learning,liu2011learning,chen2019rafm}.
    
\subsubsection{Non-equivalence of Norms and $\Omega_*$}
    \label{subsubsec:non_equivalence}
    Firstly, are existing $\ell_{2, 1}$ and $\ell_1$ regularizers upper bounds of $\Omega_{*}$?
    Unfortunately, not only them but also any norm on $\real^{d \times k}$ can be neither an upper bound nor a lower bound; i.e., all norms are not equivalent to $\Omega_{*}$.
    \begin{thm}
        \label{thm:non-equivalence_norm}
        Let $\norm{\cdot}$ be a norm on $\real^{d \times k}$. Then, for any $C > 0$, there exists $\bm{P}, \bm{Q} \in \real^{d \times k}$ such that 
        \begin{align}
            \Omega_*(\bm{P}) < C\norm{\bm{P}}, \Omega_*(\bm{Q}) > C\norm{\bm{Q}}.
            \label{eq:non-equivalence_norm}
        \end{align}
    \end{thm}
    
    \begin{proof}
        Since $\norm{\cdot}$ is a norm on $\real^{d \times k}$, it is absolutely homogeneous $\norm{a \bm{P}} = \abs{a}\norm{P}$ for all $a \in \real$ and $\norm{\bm{P}}=0 \iff \bm{P}=\bm{0}$.
        On the other hand, $\Omega_{*}$ is $2$-homogeneous: 
        \begin{align}
            \Omega_{*}(a \bm{P}) = \sum_{j_1=1}^d\sum_{j_2>j_1}\abs{\inner{a\bm{p}_{j_1}}{a\bm{p}_{j_2}}}=a^2\Omega_{*}(\bm{P}),
        \end{align}
        and $\Omega_{*}(\bm{P})=0 \iff \inner{\bm{p}_{j_1}}{\bm{p}_{j_2}} = 0$ for all $j_1\neq j_2 \in [d]$. Thus, we can take $\bm{P'} \in \real^{d\times k}$ such that $\norm{\bm{P'}}\neq 0$ and $\Omega_{*}(\bm{P'}) \neq 0$. Given $C>0$, we take a positive number $a$ such that $0 < a < C \norm{P'} / \Omega_*(\bm{P'})$. Then, $\Omega_*(a \bm{P'}) = a^2 \Omega_{*}(\bm{P'}) < a\left[C \norm{P'} / \Omega_*(\bm{P'})\right]\Omega_{*}(\bm{P'}) = C\norm{a\bm{P'}}$, which is surely the first inequality in~\eqref{eq:non-equivalence_norm} ($\bm{P}=a\bm{P}'$). Similarly, if we take $a > C \norm{\bm{P}'} / \Omega_*(\bm{P'})$, we can derive the second inequality in~\eqref{eq:non-equivalence_norm}.
    \end{proof}
    In some cases, the fact that any norm cannot be an upper bound of $\Omega_{*}$ is crucial.
    Suppose that one wants FMs with $\bm{P}$ such that $\Omega_{*}(\bm{P}) \le \lambda$; i.e., one solves the constrained minimization problem.
    Since this problem is also hard to optimize, one can replace $\Omega_{*}$ with $\norm{\cdot}$, and the revised problem may be easier to optimize.
    However, it is not guaranteed that the solution $\bm{P}$ satisfies $\Omega_{*}(\bm{P}) \le \lambda$ because $\norm{\cdot}$ cannot be an upper bound of $\Omega_{*}$.
    
    The existing methods using sparsity-inducing norms produce completely dense (all-non-zeros) or all-zeros feature interaction matrices as shown in~\cref{fig:toy}.
    This phenomenon can be explained by~\cref{thm:non-equivalence_norm}. 
    From the proof of~\cref{thm:non-equivalence_norm}, we have $\norm{\bm{P}} \gg \Omega_*(\bm{P})$, i.e, the regularization strength of norm is much greater than that of $\Omega_*$, if the absolute value of each element in $\bm{P}$ is sufficiently small.
    Thus, when $\lambda$ is large, the existing methods using norm regularizers can produce all-zeros matrices.
    Similarly, we have $\norm{\bm{P}} \ll \Omega_*(\bm{P})$ if the absolute value of each element in $\bm{P}$ is sufficiently large.
    Thus, when $\lambda$ is small, the existing methods using norm regularizers can produce completely dense matrices.

\subsubsection{Upper Bound Regularizers of $\Omega_*$ by Squares of (Quasi-)norms}
    In this section, we present how to construct an upper bound of $\Omega_*$.
    We first define $m$-homogeneous quasi-norms.
    \begin{defi}
        \label{def:m-homogeneous_quasi-norm}
        We say a function $\Omega: \real^{d \times k} \to \real_{\ge 0}$ is an $m$-homogeneous quasi-norm if, for all $\bm{P}, \bm{Q} \in \real^{d \times k}$, (i) $\Omega(\bm{P}) \ge 0$, $\Omega(\bm{P}) = 0 \iff \bm{P} = \bm{0}$, (ii) there exists $m\in \natu_{>0}$ for all $a \in \real$ such that $\Omega(a\bm{P}) = \abs{a}^m \Omega(\bm{P})$, and (iii) there exists $K>0$ ($K \ge 2^{m-1}$) such that $\Omega(\bm{P} + \bm{Q}) \le K(\Omega(\bm{P}) + \Omega(\bm{Q}))$.
        Note that $m=1$ implies that $\Omega$ is a quasi-norm.
    \end{defi}
    
    There is an important relationship between $\Omega_*$ and $2$-homogeneous quasi-norms: unlike norms, any $2$-homogeneous quasi-norm can be an upper bound of $\Omega_*$.
    \begin{thm}
        \label{thm:bound_by_square_of_norm}
        For any $2$-homogeneous quasi-norm $\Omega$, there exists $C > 0$ such that $\Omega_{*}(\bm{P}) \le C \Omega(\bm{P})$ for all $\bm{P} \in \real^{d \times k}$.
    \end{thm}
    Moreover, one can construct an $m$-homogeneous quasi-norm by $m$-th power of a (quasi-)norm.
    \begin{restatable}{thm}{powofquasinorm}
        \label{thm:pow_of_quasi_norm}
        $\Omega: \real^{d \times k} \to \real_{\ge 0}$ is an $m$-homogeneous quasi-norm if and only if there exists a quasi-norm $\norm{\cdot}'$ such that $\Omega(\cdot) = (\norm{\cdot}')^{m}$.
    \end{restatable}
    Thus, one can construct an upper bound of $\Omega_*$ by the square of a (quasi-)norm.
    The regularizer in canonical FMs, $\ell_{2}^2$, is clearly the square of the norm but it does not produce sparse feature interaction matrix $\bm{W}=\bm{P}\bm{P}^\top$ (as indicated in our experimental results in~\cref{sec:experiments}).
    It means that using an upper bound of $\Omega_*$ is not sufficient for feature interaction selection.
    We thus propose especially using the square of a \emph{sparsity-inducing} (quasi-)norm, which can make $\bm{W}=\bm{P}\bm{P}^\top$ sparse by making $\bm{P}$ sparse since $\supp(\bm{p}_{j_1}) \cap \supp(\bm{p}_{j_2}) = \emptyset$ implies $w_{j_1, j_2} = 0$.
    As described above, the sparsity of $\bm{P}$ does not necessarily imply the sparsity of $\bm{W}$.
    However, the square of a sparsity-inducing (quasi-)norm can be more useful for feature interaction selection than a sparsity-inducing norm because the square of a norm can be an upper bound of $\Omega_*$.
    In the following sections, we present such regularizers.
    
\subsection{Comparison of Norm and Squared Norm}
    Before presenting the proposed regularizers, we discuss relationships between regularizations based on squared norms and those based on norms.
    Consider the optimization problem with the regularization based on the squared norm
    \begin{align}
        \min_{\bm{w}, \bm{P}} L_{\mathrm{FM}}(\bm{w}, \bm{P}; \lambda_{w}, \lambda_{p}) + \lambda \norm{\bm{P}}^2
        \label{eq:squared_norm_optimization_problem}
    \end{align}
    and one of its stationary points $\{\hat{\bm{w}}, \hat{\bm{P}}\}$.
    Then, $\{\hat{\bm{w}}, \hat{\bm{P}}\}$ is also one of the stationary points of the optimization problem with the regularization based on the (non-squared) norm
    \begin{align}
        \min_{\bm{w}, \bm{P}} L_{\mathrm{FM}}(\bm{w}, \bm{P}; \lambda_{w}, \lambda_{p}) + \left(2 \lambda \norm{\hat{\bm{P}}}\right) \norm{\bm{P}}
        \label{eq:norm_optimization_problem}
    \end{align}
    since $\partial \norm{\bm{P}}^2 = 2 \norm{\bm{P}} \partial \norm{\bm{P}}$ (from~\cref{lemm:subdifferential_pow_norm} in~\cref{subsec:subdifferential}).
    Therefore, one might consider that the use of the squared norm $\norm{\cdot}^2$ is essentially equivalent to the use of the (non-squared) norm $\norm{\cdot}$.
    However, the optimal solutions of~\eqref{eq:squared_norm_optimization_problem} and~\eqref{eq:norm_optimization_problem} are not necessarily equivalent to each other because $L_{\mathrm{FM}}(\bm{w}, \bm{P}; \lambda_{w}, \lambda_{p})$ is non-convex w.r.t $\bm{P}$.
    To the best of our knowledge, there are no known relationships between the optimal solutions of the existing methods~\citep{pan2016sparse,xu2016synergies,zhao2017meta} and those of the QR with the $\ell_1$ regularization for $\bm{W}$.
    On the other hand, under some conditions, the optimal solutions of one of the proposed methods are equivalent to those of the QR with the $\ell_1$ regularization for $\bm{W}$, which is shown in~\cref{sec:ti}.

    As described above, the squared norm regularization is not necessarily equivalent to the norm regularization.
    On the other hand, the squared norm regularization can be interpreted as the norm regularization with an adaptive regularization-strength since $\lambda \norm{\bm{P}}^2 = (\lambda \norm{\bm{P}})\norm{\bm{P}}$.
    Indeed, as we will see later in~\cref{sec:ti} and~\cref{sec:cs}, the proposed TI (CS) regularizer based on squared norms shrink elements (row vectors) of relatively small absolute value ($\ell_2$ norm) to $0$ ($\bm{0}$).
    On the other hand, the existing methods shrink absolutely small elements to $0$ (i.e., elements that are smaller than a fixed threshold are shrunk to $0$).
    Therefore, regularizers based on squared norms (the proposed methods) can be less sensitive to the choice of the regularization-strength hyperparameter than those based on (non-squared) norms (the existing methods).
    
\section{TI Upper Bound Regularizer}
    \label{sec:ti}
    We first propose using $\tilde{\ell}_{1,2}^2$ as an upper bound regularizer of $\Omega_*$.
    We call this regularizer $\tilde{\ell}_{1,2}^2$ regularizer or triangular inequality (TI) regularizer and call FMs with this regularization $\tilde{\ell}_{1,2}^2$-sparse FMs or TI-sparse FMs since this regularization can also be derived using the triangle inequality:
    \begin{align}
         \frac{1}{2}\norm{\bm{P}^\top}_{1,2}^2  = \frac{1}{2} \sum_{s=1}^k \norm{\bm{p}_{:, s}}_1^2 &= \sum_{j_2>j_1}\sum_{s=1}^k \abs{p_{j_1, s}p_{j_2, s}} + \frac{1}{2}\norm{\bm{P}}_2^2 \\
         &\eqqcolon \Omega_{\mathrm{TI}}(\bm{P}) + \frac{1}{2}\norm{\bm{P}}_2^2\ge \Omega_{*}(\bm{P}) + \frac{1}{2}\norm{\bm{P}}_2^2.
        \label{eq:ti}
    \end{align}
    Because $\norm{\bm{P}^\top}_{1,2}^2 = 2\Omega_{\mathrm{TI}}(\bm{P}) + \norm{\bm{P}}_2^2$ and $\norm{\bm{P}}_2^2$ can be taken into $L_{\mathrm{FM}}$, we will sometimes discuss not TI-sparse FMs but rather $\Omega_{\mathrm{TI}}$-sparse FMs (FMs with $\Omega_{\mathrm{TI}}$ regularization).

    We here discuss the relationship between $\Omega_{\mathrm{TI}}$-sparse FMs and the QR~\eqref{eq:quadratic} with $\ell_2^2$ regularization for $\bm{w}$ and $\ell_{1}$ regularization for $\bm{W}$.
    \Cref{thm:equivalence_tisfm_qr} states that the optimal $\Omega_{\mathrm{TI}}$-sparse FMs are equivalent (or better in the sense of the objective value) to the optimal QR with such regularizations when the rank hyperparameter $k$ is sufficiently large.
    We also obtain a similar relationship between $\Omega_{\mathrm{TI}}$ -sparse FMs and $\Omega_*$-sparse FMs.
    It is shown in~\cref{subsec:proof_analysis}.
    Note that $\Omega_{\mathrm{TI}}$-sparse FMs can be regarded as TI-sparse FMs ($\tilde{\ell}_{1,2}^2$-sparse FMs) when $\lambda_p \ge \tilde{\lambda}_p / 2$.
    
    \begin{restatable}{thm}{equivalenceqr}
        \label{thm:equivalence_tisfm_qr}
        $L_{\mathrm{QR}}(\bm{w}, \bm{W}; \lambda_{w})$ be the objective function of the QR with $\ell_2^2$ regualrization for $\bm{w}$:
        \begin{align}
            L_{\mathrm{QR}}(\bm{w}, \bm{W}; \lambda_{w}) \coloneqq \sum_{n=1}^N \ell (f_{\mathrm{QR}}(\bm{x}_n), y_n)/N + \lambda_{w} \norm{\bm{w}}_2^2.
        \end{align}
        Then, for any $\lambda_{w}, \lambda_{p}, \tilde{\lambda}_{p} \ge 0$, there exists $k' \le d(d-1)/2$ such that for all $k \ge k'$,
        \begin{align}
            & \min_{\bm{w}\in \real^d, \bm{P} \in \real^{d \times k}} L_{\mathrm{FM}}(\bm{w}, \bm{P}; \lambda_{w}, \lambda_{p}) + \tilde{\lambda}_{p}\Omega_{\mathrm{TI}}(\bm{P}) \le \min_{\bm{w}\in \real^d, \bm{W} \in \real^{d \times d}} L_{\mathrm{QR}}(\bm{w}, \bm{W}; \lambda_{w}) + (\tilde{\lambda}_{p}+2\lambda_p) \norm{\bm{W}}_{1}.
            \label{eq:equivalence_tisfm_qr}
        \end{align}
        Moreover, if $\lambda_{p}=0$, the equality holds, and $f_{\mathrm{FM}}(\bm{x}; \bm{w}^*_{\mathrm{TI}}, \bm{P}^*_{\mathrm{TI}}) = f_{\mathrm{QR}}(\bm{x}; \bm{w}^*_{\mathrm{QR}}, \bm{W}^*_{\mathrm{QR}})$ for all $\bm{x} \in \real^d$, where $\{\bm{w}^*_{\mathrm{TI}}, \bm{P}^*_{\mathrm{TI}}\}$ and $\{\bm{w}^*_{\mathrm{QR}}, \bm{W}^*_{\mathrm{QR}}\}$ are the solutions of the left- and right-hand sides, respectively, of Equation~\eqref{eq:equivalence_tisfm_qr}.
    \end{restatable}
    
    We next consider TI ($\tilde{\ell}_{1,2}^2$) regularized objective function, i.e., $L_{\mathrm{FM}}(\bm{w}, \bm{P}; \lambda_{w}, \lambda_{p}) + \tilde{\lambda}_{p} \norm{\bm{P}^\top}_{1,2}^2$.
    Since $\norm{\bm{P}^\top}_{1,2}^2 = \norm{\bm{P}}_{2}^2 + 2 \Omega_{\mathrm{TI}}(\bm{P})$, it can be written as the $\Omega_{\mathrm{TI}}$ regularized objective function: $L_{\mathrm{FM}}(\bm{w}, \bm{P}; \lambda_{w}, \lambda_{p}+\tilde{\lambda}_{p}) + 2\tilde{\lambda}_{p} \Omega_{\mathrm{TI}}(\bm{P})$.
    The optimization problem of the TI regularized objective function can be written as
    \begin{align}
        &\min_{\bm{w}\in \real^d, \bm{P} \in \real^{d \times k}} L_{\mathrm{FM}}(\bm{w}, \bm{P}; \lambda_{w}, \lambda_{p}) + \tilde{\lambda}_{p} \norm{\bm{P}^\top}_{1,2}^2
        \label{eq:squaredl12_regularized_problem}\\
        &= \min_{\bm{w}\in \real^d, \bm{P} \in \real^{d \times k}} L_{\mathrm{QR}}\left(\bm{w}, \bm{P}\bm{P}^\top; \lambda_{w}\right) +  \sum_{s=1}^k \left\{\lambda_{p} \norm{\bm{p}_{:, s}}_{2}^2 + \tilde{\lambda}_{p}\norm{\bm{p}_{:, s}}_{1}^2\right\},\\
        & = \min_{\bm{w}\in \real^d, \bm{W} \in \real^{d \times d}} L_{\mathrm{QR}}(\bm{w}, \bm{W}; \lambda_{w}) + \Omega_{\mathrm{MF}}\left(\bm{W}; k, \lambda_p, \tilde{\lambda}_p\right),
    \end{align}
    where
    \begin{align}
         \Omega_{\mathrm{MF}}\left(\bm{W}; k, \lambda_p, \tilde{\lambda}_p\right) \coloneqq \inf_{r\in [k]} \inf_{\substack{\bm{U}, \bm{V} \in \real^{d \times r},\\ \bm{U}\bm{V}^\top=\bm{W}}} \sum_{s=1}^r \left\{\lambda_{p} \norm{\bm{u}_{:, s}}_{2}^2 + \tilde{\lambda}_{p}\norm{\bm{u}_{:, s}}_{1}^2 + I_{\{\bm{u}_{:,s}\}}(\bm{v})\right\},
    \end{align}
    and $I_C: \real^d \to \{0, \infty\}$ is the indicator function on $C \subseteq \real^d$.
    When the rank hyperparameter $k$ is sufficiently large, $\Omega_{\mathrm{MF}}(\bm{W}; k, \lambda_p, \tilde{\lambda}_p)$ satisfies the condition as matrix factorization regularizer~\citep{haeffele2019structured}.
    Thus, from Proposition 2, Theorem 1, and Theorem 2 in~\citep{haeffele2019structured}, we conclude that all local minima of the TI regularized objective function are global minima when $k$ is sufficiently large.
    \begin{thm}[\citet{haeffele2019structured}]
        When $k \ge d^2$, all local minima of~\eqref{eq:squaredl12_regularized_problem} are global minima.
    \end{thm}
    Note that for any $m \in \natu_{> 0}$ all local minima of $\tilde{\ell}_{m,2}^2$ regularized optimization problem are global minima under the same condition.

\subsection{PGD/PSGD-based Algorithm for TI Regularizer}
    \label{subsec:pgd_ti}
    Because the TI regularizer is continuous but non-differentiable, we consider the use of the PGD algorithm for optimizing TI-sparse FMs similarly to $\ell_1$-sparse FMs and $\ell_{2,1}$-sparse FMs.
    The PGD algorithm for TI-sparse FMs solves the following optimization problem at $t$-th iteration:
    \begin{align}
        \min_{\bm{Q}\in \real^{d\times k}} \inner{\nabla\ell_{\mathcal{D}}^{(t-1)}}{\bm{Q}-\bm{P}^{(t-1)}} + \frac{1}{2\eta}\norm{\bm{Q}-\bm{P}^{(t-1)}}_2^2+ \lambda_{p} \norm{\bm{Q}}_2^2 + \tilde{\lambda}_{p} \norm{\bm{Q}^\top}_{1,2}^2,
        \label{eq:pgd_ti}
    \end{align}
    where $\nabla\ell_{\mathcal{D}}^{(t-1)} \coloneqq \sum_{n=1}^N\nabla\ell\left(f_{\mathrm{FM}}\left(\bm{x}_n; \bm{P}^{(t-1)}\right), y_n\right) / N$ and $\bm{P}^{(t-1)}$ is the parameter matrix $\bm{P}$ after $(t-1)$-th iteration (we omit $\bm{w}^{(t-1)}$ for simplicity).
    Fortunately, the TI regularizer is separable w.r.t each column vector: $\norm{\bm{P}^\top}_{1,2}^2 = \sum_{s=1}^k \norm{\bm{p}_{:, s}}_1^2$.
    We thus can separate the optimization problem~\eqref{eq:pgd_ti} w.r.t each column as
    \begin{align}
        \min_{\bm{q}_{:, s}\in \real^{d}} \inner{\nabla_{\bm{p}_{:, s}}\ell_{\mathcal{D}}^{(t-1)}}{\bm{q}_{:,s}-\bm{p}_{:, s}^{(t-1)}} + \frac{1}{2\eta}\norm{\bm{q}_{:,s}-\bm{p}_{:, s}^{(t-1)}}_2^2+ \lambda_{p} \norm{\bm{q}_{:,s}}_2^2 + \tilde{\lambda}_{p} \norm{\bm{q}_{:, s}}_{1,2}^2,
        \label{eq:pgd_ti_separate}
    \end{align}
    where $\nabla_{\bm{p}_{:,s}}\ell_{\mathcal{D}}^{(t-1)} \coloneqq \sum_{n=1}^N\nabla_{\bm{p}_{:,s}} \ell(f_{\mathrm{FM}}(\bm{x}_n; \bm{P}^{(t-1)}), y_n) / N$ and solve this separated problem for all $s \in [k]$ at $t$-th iteration.
    Therefore, the PGD algorithm for TI-sparse FMs surely solves the following type of optimization problem for all $s \in [k]$:
    \begin{align}
        \prox_{\lambda \norm{\cdot}_1^2}(\bm{p}) = \argmin_{\bm{q} \in \real^{d}} \frac{1}{2} \norm{\bm{q}-\bm{p}}_2^2 + \lambda \norm{\bm{q}}_1^2 \eqqcolon \bm{q}^{*},
        \label{eq:prox_ti}
    \end{align}
    where $\lambda \ge 0$.
    The following theorem gives us the insight needed for constructing an algorithm for computing this proximal operator~\citep{filipe2011online}.
    \begin{thm}[\citet{filipe2011online}]
        \label{thm:solution_prox_ti}
        Assume that $\bm{p}\in \real^d$ is sorted in descending order by absolute value: $\abs{p_1} \ge \abs{p_2} \ge \cdots \ge \abs{p_d}$.
        Then, the solution to the proximal problem~\eqref{eq:prox_ti} $\bm{q}^* \in \real^d$ is
        \begin{align}
            q^*_j = \sign(p_j)\max\{\abs{p_j} - 2\lambda S_{\theta}, 0\} \ \forall j \in [d],
            \label{eq:solution_prox_ti}
        \end{align}
        where $\theta = \max\{j : \abs{p_j} - 2\lambda S_j \ge 0\}$, and $S_j = \sum_{i=1}^j \abs{p_i}/(1+2\lambda j)$.
    \end{thm}

    This theorem states that, for arbitrary $\bm{p} \in \real^d$, the proximal operator~\eqref{eq:prox_ti} can be computed in $O(d \log d)$ time by first sorting $\bm{p}$ by absolute value and then computing $S_j$ for all $j \in [d]$ and $\theta$.
    In fact, this proximal operator can be evaluated in $O(d)$ time in expectation by randomized-median-finding-like method~\citep{duchi2008efficient}.
    For more detail, please see our appendix (\cref{alg:prox_ti_sort} and~\cref{alg:prox_ti_randomize} in~\cref{sec:impl}).
    The proximal operator of $\ell_1$ regularizer~\eqref{eq:prox_l1} shrinks each element in $\bm{p}_{:,s}$: $q^{*}_{j,s} = 0$ if $\abs{p_{j,s}} \le \lambda$ for all $j \in [d]$ and the threshold $\lambda$ does not depend on $\bm{P}$.
    Similarly, the proximal operator of $\tilde{\ell}_{1,2}^{2}$~\eqref{eq:prox_ti} also shrinks each element in $\bm{p}_{:,s}$.
    However, its threshold depends on $\bm{p}_{:,s}$: $q^{*}_{j,s} = 0$ if $\abs{p_{j,s}} \le 2\lambda S_{\theta}$ and $S_{\theta}$ depends on $\bm{p}_{:,s}$.
    That is, intuitively, $\tilde{\ell}_{1,2}^2$ regularizer shrinks elements of relatively small absolute value among $p_{1, s}, \ldots, p_{d, s}$ to $0$.

    Clearly, one can construct a proximal SGD (PSGD) algorithm by replacing $\nabla\ell_{\mathcal{D}}^{(t-1)}$ in~\eqref{eq:pgd_ti} with a stochastic gradient.
    The PSGD-based algorithms are typically more useful than the PGD-based (i.e., batch) algorithms when the number of instances $N$ is large~\citep{bottou2012stochastic,bottou2018optimization,allen2018natasha}.
    For the (not necessarily convex) smooth optimization problem $\min_{\bm{z}\in \real^d} \sum_{n=1}^N f_i(\bm{z})/N, f_i: \real^d \to \real$ for all $i \in [N]$, the GD and the SGD require respectively $O\left(1/\varepsilon^2\right)$ and $O\left(V/\varepsilon^4\right)$ iterations to get an $\varepsilon$-approximate stationary point~\citep{allen2018natasha}, where $V$ is the upper bound of the variance of the stochastic gradients.
    In most cases, the GD requires to compute $N$ gradients while the SGD requires one gradient at each iteration.
    Thus, when $N$ is large and $\varepsilon$ is moderate, the SGD is superior to the GD.
    Unfortunately, the PSGD algorithm and its variants with~\cref{alg:prox_ti_randomize} cannot leverage the sparsity of data: they require $O(dk)$ time for each iteration and $O(Ndk)$ time for each epoch even if a dataset $\bm{X}$ is sparse.
    The $O(dk)$ computational cost is due to the evaluation of the proximal operator (\Cref{alg:prox_ti_randomize}).
    A workaround for this issue is the use of mini-batches.
    The use of mini-batch reduces the variance of the stochastic gradient and it hence reduces the number of iteration for convergence, but in general it increases the computational cost for each iteration.
    In our setting, if one chooses a mini-batch such that its number of non-zero elements is $O(d)$, the mini-batch PSGD also runs in $O(dk)$ for one iteration, that is, the use of appropriate-size-mini-batches can reduce the variance of the stochastic gradients without changing the computational complexity for one iteration.
    However, it does not solve the issue completely: the cost for one iteration, $O(dk)$, is not so improved compared to the that of the PGD algorithm $O(\nnz{\bm{X}}k)$ when $\bm{X}$ is very sparse.
    Thus, the (mini-batch) PSGD algorithm should be used only when $\nnz{\bm{X}}/d$ is large.

\subsection{Efficient PCD Algorithm for TI Regularizer}
    \label{subsec:pcd_ti}
    Here, we present an efficient PCD algorithm for TI-sparse FMs, which is often used for minimizing the objective function with non-smooth regularization and has several advantages compared to the PGD algorithm introduced in~\cref{subsec:pgd_ti}.
    Firstly, it does not require tuning nor using line search technique for the step size, and this is its most important advantage compared with the PGD/PSGD-based algorithms.
    Secondly, it can leverage sparsity of data: it runs in $O(\nnz{\bm{X}}k)$ for one epoch.
    Thirdly, it is easy to implement: its implementation is simple and almost the same as the CD algorithm for canonical FMs and the PCD algorithm for $\ell_1$-sparse FMs.
    Fourthly, it can be easily extended to other related models as shown in~\cref{sec:extension}.
    Strictly speaking, $\tilde{\ell}_{1,2}^2$ is not separable w.r.t $p_{1,1}, \ldots, p_{d, k}$ and thus the convergence of the PCD algorithm is not guaranteed.
    However, actually, it doesn't matter much that the convergence of the PCD algorithm is not guaranteed.
    Because the objective function of FMs is non-convex, a global minimum solution cannot be obtained even if the PGD algorithm is used.
    
    Because $\norm{\bm{P}^\top}_{1,2}^2 = 2\Omega_{\mathrm{TI}}(\bm{P}) + \norm{\bm{P}}_2^2$ and $\norm{\bm{P}}_2^2$ can be taken into $L_{\mathrm{FM}}$, for simplify we consider $\Omega_{\mathrm{TI}}$ regularized objective function $L_{\mathrm{FM}}(\bm{w}, \bm{P}; \lambda_{w}, \lambda_{p})+\tilde{\lambda}_{p}\Omega_{\mathrm{TI}}(\bm{P})$, focusing on one parameter $p_{j,s}$ as the optimized parameter.
    Then, the $\Omega_{\mathrm{TI}}$ regularizer can be regarded as a $\ell_1$ regularizer such that its regularization strength is $\sum_{i\in [d] \setminus \{j\}}|p_{i, s}|$:
    \begin{align}
        \Omega_{\mathrm{TI}}(\bm{P}) = \left(\sum_{i\in [d] \setminus \{j\}}|p_{i, s}|\right)|p_{j, s}| + \mathrm{const}.
    \end{align}
    Therefore, given this regularization strength, $\sum_{i\in [d] \setminus \{j\}}\abs{p_{i, s}}$, the procedure of the PCD algorithm for the $\Omega_{\mathrm{TI}}$ (and of course TI ($\tilde{\ell}_{1,2}^2$)) regularized objective function is the same as that for the $\ell_1$ regularized objective function.
    Fortunately, by caching $c_{s} \coloneqq \sum_{j=1}^d \abs{p_{j, s}}$ before updating $p_{1, s}\ldots p_{d, s}$, the algorithm can compute $\sum_{i\in [d] \setminus \{j\}}\abs{p_{i, s}}$ in $O(1)$ time at each iteration. Given $c_s$, one can compute $\sum_{i\in [d] \setminus \{j\}}\abs{p_{i, s}}$ by using $c_s - \abs{p_{j, s}}$, and $c_s$ for the next iteration (i.e., for updating $p_{j+1, s}$) can be computed by using $c_s - \abs{p_{j, s}^{\mathrm{old}}} + \abs{p_{j, s}^{\mathrm{new}}}$.
    \Cref{alg:pcd_ti} shows the procedure for the PCD algorithm for the TI regularized objective function.
    The computational cost of~\cref{alg:pcd_ti} for one epoch is $O(\nnz{\bm{X}}k)$, which is the same as those of the CD and SGD algorithms for canonical FMs.
    For more detail about the CD algorithm for canonical FMs, please see~\citep{blondel2016higher,rendle2011fast,rendle2012factorization} or~\cref{sec:impl}.
    
    \begin{algorithm}[t]
        \caption{PCD algorithm for TI-sparse FMs}
        \label{alg:pcd_ti}
        \begin{algorithmic}[1]
            \Input{$\{(\bm{x}_n, \bm{y}_n) \}_{n=1}^{N}$, $k \in \natu_{>0}$, $\lambda_{w}, \lambda_{p}, \tilde{\lambda}_{p} \ge 0$}
            \State{Initialize $\bm{P} \in \real^{d\times k}$, $\bm{w} \in \real^d$;}
            \State{Compute caches as in canonical FM;}
            \While{not convergence}
                \State{Optimize $\bm{w}$ and update caches as in canonical FM;}
                \For{$s=1, \ldots, k$}
                    \State{\textcolor{blue}{$c_s \leftarrow \sum_{j=1}^d \abs{p_{j, s}}$};}\Comment{Cache for $\Omega_{\mathrm{TI}}$}
                    \For{$j=1, \ldots, d$}
                        \State{\textcolor{blue}{$c_s \leftarrow c_s - \abs{p_{j, s}}$};}\Comment{$c_s = \sum_{i \in [d] \setminus \{j\}}\abs{p_{i,s}}$}
                        \State{$\eta \leftarrow (\mu \sum_{n \in \supp(\bm{x}_{:, j})}(\partial f_{\mathrm{FM}}(\bm{x}_n)/\partial p_{j,s})^2/N + 2\lambda_{p})^{-1}$;}
                        \State{Update $p_{j, s}$ as in canonical FM;}
                        \State{\textcolor{blue}{$p_{j, s} \leftarrow \prox_{\eta \tilde{\lambda}_{p} c_s \abs{\cdot}}(p_{j, s})=\sign(p_{j,s})\max (\abs{p_{j, s}} - \eta\tilde{\lambda}_{p} c_s, 0)$};}
                        \State{Update caches as in canonical FM;}
                        \State{\textcolor{blue}{$c_s \leftarrow c_s + \abs{p_{j, s}}$};} \Comment{Update cache for $p_{j+1, s}$}
                    \EndFor
                \EndFor
            \EndWhile
            \Output{Learned $\bm{P}$ and $\bm{w}$}
        \end{algorithmic}
    \end{algorithm}

\section{CS Upper Bound Regularizer}
    \label{sec:cs}
    We next propose using $\ell_{2,1}^2$ as an upper bound regularizer of $\Omega_*$.
    Because $\ell_{2,1}^2$ is an upper bound of $\Omega_*$ and its corresponding proximal operator outputs row-wise sparse $\bm{P}$ (we will see it in~\cref{subsec:pgd_cs}), it can be better for feature selection in FMs (i.e., it can select better features in terms of the prediction performance) than the existing regularizers and TI regularizer.
    We call this regularizer $\ell_{2,1}^2$ regularizer or Cauchy–Schwarz (CS) regularizer and call FMs with this regularization $\ell_{2,1}^2$-sparse FMs or CS-sparse FMs since this regularization is derived using the Cauchy–Schwarz inequality:
    \begin{align}
         \frac{1}{2}\norm{\bm{P}}_{2,1}^2 &= \sum_{j_2>j_1}\norm{\bm{p}_{j_2}}_2\norm{\bm{p}_{j_1}}_2 + \frac{1}{2}\norm{\bm{P}}_2^2 \\
         &\eqqcolon \Omega_{\mathrm{CS}}(\bm{P}) + \frac{1}{2}\norm{\bm{P}}_2^2 \ge \Omega_{*}(\bm{P}) + \frac{1}{2}\norm{\bm{P}}_2^2.
        \label{eq:cs}
    \end{align}
    
\subsection{PGD/PSGD-based Algorithm for CS Regularizer}
    \label{subsec:pgd_cs}
    Similarly to TI regularizer, we consider the use of the PGD algorithm for optimizing CS-sparse FMs.
    The PGD algorithm for CS-sparse FMs requires to compute the following proximal operator for a $d \times k$ matrix:
    \begin{align}
        \prox_{\lambda \norm{\cdot}_{2,1}^2}(\bm{P}) = \argmin_{\bm{Q} \in \real^{d \times k}} \frac{1}{2} \norm{\bm{Q}-\bm{P}}_2^2 + \lambda \norm{\bm{Q}}_{2,1}^2.
        \label{eq:prox_cs}
    \end{align}
    The following theorem states that the proximal operator~\eqref{eq:prox_cs} can be computed in $O(dk + d \log d)$ by~\cref{alg:prox_ti_sort} or $O(dk)$ time in expectation by~\cref{alg:prox_ti_randomize}.
    \begin{thm}
        \label{thm:solution_prox_cs}
       The solution to the proximal problem~\eqref{eq:prox_cs} $\bm{Q}^* \in \real^{d \times k}$ is
        \begin{align}
            \bm{q}^*_j = \begin{cases}
                \frac{c^*_j}{\norm{\bm{p}_j}_2}\bm{p}_j & \norm{\bm{p}_j}_2 \neq 0,\\
                \bm{0} & \norm{\bm{p}_j}_2 = 0,
            \end{cases}
            \label{eq:solution_prox_cs}
        \end{align}
        where $\bm{c}^*=\prox_{\lambda \norm{\cdot}_1^2}\left(\left(\norm{\bm{p}_1}_2, \ldots, \norm{\bm{p}_d}_2\right)^\top\right)$.
    \end{thm}
    The proximal operator of $\ell_{2,1}$~\eqref{eq:prox_l21} shrinks row vectors in $\bm{P}$: $\bm{q}^*_j = \bm{0}$ if $\norm{\bm{p}_j} \le \lambda$ and the threshold $\lambda$ does not depend on $\bm{P}$, so it sets row vectors of absolutely small $\ell_2$ norm to be $\bm{0}$.
    Since $\bm{c}^* = \prox_{\lambda \norm{\cdot}_1^2}\left(\left(\norm{\bm{p}_1}_2, \ldots, \norm{\bm{p}_d}_2\right)^\top\right)$ can be sparse, $\ell_{2,1}^{2}$~\eqref{eq:prox_cs} also shrinks row vectors in $\bm{P}$: if $\norm{\bm{p}_j}_2$ is smaller than the threshold, then $\bm{q}^*_j = \bm{0}$.
    However, the threshold depends on $\bm{P}$: $\bm{q}^{*}_{j} = \bm{0}$ if $\norm{\bm{p}_j}_2 \le 2\lambda S_{\theta}$ and $S_{\theta}$ depends on $\norm{\bm{p}_1}_2, \ldots, \norm{\bm{p}_{d}}_2$ as described in~\cref{thm:solution_prox_ti}.
    That is, $\ell_{2,1}^2$ regularizer shrinks row vectors of relatively small $\ell_2$ norm among $\bm{p}_1, \ldots, \bm{p}_{d}$ to $\bm{0}$.
    Clearly, CS-sparse FMs also cannot achieve feature interaction selection like $\ell_{2,1}$-sparse FMs~\citep{xu2016synergies}: they produce a row-wise sparse $\bm{P}$.
    However, the CS regularizer is more useful than the $\ell_{2,1}$ one for feature selection in FMs since it is also an upper bound of the $\ell_{2,1}$ norm for $\bm{P}\bm{P}^\top$ without diagonal elements, which seems appropriates for feature selection in FMs.
    
    Unfortunately, PSGD-based algorithms for the CS regularizer using~\cref{alg:prox_ti_randomize} cannot leverage the sparsity of dataset $\bm{X}$.
    Thus, PSGD-based algorithms should be used only when the number of instances is large and dataset $\bm{X}$ is dense like those for the TI regularized objective function as described in~\cref{sec:ti}.

\subsection{PBCD Algorithm for CS Regularizer}
    \label{subsec:pbcd_cs}
    We here propose an efficient PBCD algorithm for CS-sparse FMs that optimizes each row vector in $\bm{P}$ at each iteration.
    Like the PCD algorithm for TI-sparse FMs, strictly speaking, $\ell_{2,1}^2$ is not separable w.r.t $\bm{p}_1, \ldots, \bm{p}_d$ and thus it should not be used but it has several advantages.

    We consider $\Omega_{\mathrm{CS}}$ regularized objective function $L_{\mathrm{FM}}(\bm{w}, \bm{P}; \lambda_{w}, \lambda_{p})+\tilde{\lambda}_{p} \Omega_{\mathrm{CS}}(\bm{P})$ because $\norm{\bm{P}}_2^2$ can be taken into $L_{\mathrm{FM}}$ as in~\cref{subsec:pgd_ti}, focusing on the $j$-th row vector in $\bm{P}$ as the optimized parameter.
    Then, the $\Omega_{\mathrm{CS}}$ regularizer can be regarded as a $\ell_{2}$ regularizer such that its regularization strength is $\sum_{i \in [d]\setminus \{j\}}\norm{\bm{p}_i}_2$:
    \begin{align}
        \Omega_{\mathrm{CS}}(\bm{P}) = \left(\sum_{i \in [d]\setminus\{j\}}\norm{\bm{p}_i}_2\right)\norm{\bm{p}_j}_2 + \mathrm{const}.
    \end{align}
    Therefore, as the PCD algorithm for TI-sparse FMs is almost the same as that for $\ell_1$-sparse FMs, the PBCD algorithm for CS-sparse FMs is almost the same as that for $\ell_{2,1}$-sparse FMs~\citep{xu2016synergies}.
    Our remaining task is to design an algorithm for computing regularization strength $\sum_{i \in [d]\setminus\{j\}}\norm{\bm{p}_i}_2$ in $O(k)$ time.
    Fortunately, given $c \coloneqq \sum_{i=1}^d \norm{\bm{p}_i}_2$, we can compute $c_j \coloneqq \sum_{i \in [d]\setminus\{j\}}\norm{\bm{p}_i}_2$ in $O(k)$ time as $c_j = c - \norm{\bm{p}_j}_2$.
    The gradient of $\sum_{n=1}^N \ell(f(\bm{x}_n), y_n)/N$ w.r.t $\bm{p}_j$ is Lipschitz continuous with constant $\mu \sum_{n \in \supp(\bm{x}_{:, j})}\norm{\nabla_{\bm{p}_j}f_{\mathrm{FM}}(\bm{x}_n)}_2^2/N$ since FMs are linear w.r.t $\bm{p}_j$.
    Thus, for determining step size $\eta$, we do not have to use the line search technique~\citep{tseng2009coordinate}, which might require a high computational cost.
    For more detail, please see~\cref{sec:impl}.
    
\section{Extensions to Related Models}
    \label{sec:extension}
    In this section, we extend $\Omega_{\mathrm{TI}}, \tilde{\ell}_{1,2}^2, \Omega_{\mathrm{CS}}$ and $\ell_{2,1}^2$ to other related models.
    
\subsection{Higher-order FMs}
    \citet{blondel2016higher} proposed higher-order FMs (HOFMs), which use not only second-order feature interactions but also higher-order feature interactions.
    $M$-order HOFMs predict the target of $\bm{x}$ as
    \begin{align}
        f^M_{\mathrm{HOFM}}(\bm{x}; \bm{w}, \bm{P}^{(2)}, \ldots, \bm{P}^{(M)}) \coloneqq \inner{\bm{x}}{\bm{w}} + \sum_{m=2}^M \sum_{s=1}^k\anova^m\left(\bm{x}, \bm{p}^{(m)}_{:, s}\right),
        \label{eq:hofm}
    \end{align}
    where $\bm{P}^{(2)}, \ldots, \bm{P}^{(M)} \in \real^{d \times k}$ are learnable parameters for $2, \ldots, M$-order feature interactions, respectively, and $\anova^m: \real^d \times \real^d \to \real$ is the $m$-order ANOVA kernel:
    \begin{align}
        \anova^m(\bm{x}, \bm{p}) \coloneqq \sum_{j_m > j_{m-1} > \cdots > j_1}x_{j_1}p_{j_1}\cdots x_{j_m}p_{j_m}.
        \label{eq:anova}
    \end{align}
    $M$-order HOFMs clearly use from second to $M$-order feature interactions.
    Although the evaluation of HOFMs~\eqref{eq:hofm} seems to take $O\left(d^mk\right)$ time at first glance, it can be completed in $O\left(dkM^2\right)$ (strictly speaking, $O(\nnz{\bm{x}}kM^2)$) time since $m$-order ANOVA kernels can be evaluated in $O(dm)$ (strictly speaking, $O(\nnz{\bm{x}}m)$) time by using dynamic programming~\citep{blondel2016higher,shawe2004kernel}.
    \citet{blondel2016higher} also proposed efficient CD and SGD-based algorithms.

    \subsubsection{Extension of $\Omega_{\mathrm{TI}}$ and $\tilde{\ell}_{1,2}^2$ for HOFMs}
    The output of $M$-order HOFMs~\eqref{eq:hofm} can be rearranged as
    \begin{align}
        f^M_{\mathrm{HOFM}}(\bm{x}) \coloneqq \inner{\bm{x}}{\bm{w}} + \sum_{m=2}^M \sum_{j_m>\cdots >j_1} w_{j_1, \ldots, j_m} x_{j_1}\cdots x_{j_m},\quad \text{where } w_{j_1, \ldots, j_m} = \sum_{s=1}^k p^{(m)}_{j_1, s}\cdots p^{(m)}_{j_m, s}.
        \label{eq:hofm_pr}
    \end{align}
    Thus, we extend $\Omega_{*}$ and $\Omega_{\mathrm{TI}}$ for higher-order feature interactions as
    \begin{align} 
        \Omega_{*}^{m}(\bm{P}) \coloneqq \sum_{j_m > \cdots > j_1}\abs{\sum_{s=1}^kp_{j_1, s}\cdots p_{j_m, s}} \le \sum_{j_m > \cdots > j_1}\sum_{s=1}^k\abs{p_{j_1, s}\cdots p_{j_m, s}} \eqqcolon \Omega_{\mathrm{TI}}^{m}(\bm{P}).
        \label{eq:Omega_higher_order_ti}
    \end{align}
    Obviously, $\Omega_{\mathrm{TI}} = \Omega_{\mathrm{TI}}^2$ holds.
    We hence propose using the following regularization for higher-order feature interactions in $M$-order HOFMs:
    \begin{align}
        \Omega_{\mathrm{TI}}^2\left(\bm{P}^{(2)}\right) + \cdots + \Omega_{\mathrm{TI}}^M\left(\bm{P}^{(M)}\right).
    \end{align}
    $\Omega_{\mathrm{TI}}^m$ can be represented by the ANOVA kernel:
    \begin{align}
        \Omega_{\mathrm{TI}}^m(\bm{P}) =\sum_{s=1}^k \anova^m\left(\absop\left(\bm{p}_{:, s}\right), \bm{1}\right).
        \label{eq:sum_of_Omega_higher_order_ti}
    \end{align}
    By using multi-linearity of the ANOVA kernel, we can rewrite $\Omega_{\mathrm{TI}}^m(\bm{P})$ as
    \begin{align}
        \Omega_{\mathrm{TI}}^m(\bm{P}) = \abs{p_{j, s}}\frac{\partial}{\partial \abs{p_{j, s}}}\anova^m\left(\absop\left(\bm{p}_{:, s}\right), \bm{1}\right) + \mathrm{const}.
    \end{align}
    Hence, $\Omega_{\mathrm{TI}}^m(\bm{P})$ is also regarded as $\ell_1$ regularization for one parameter $p_{j, s}$, like $\Omega_{\mathrm{TI}}$, and the PCD algorithm for the $\Omega_{\mathrm{TI}}^m$ regularized objective function is almost the same as that for TI-sparse FMs: at each iteration, the algorithm first updates the parameter as in the canonical HOFMs and next applies $\prox_{\eta\tilde{\lambda}_p \cdot c\abs{\cdot}}$ to updated $p^{(m)}_{j, s}$, where $c=\partial\anova^m\left(\absop\left(\bm{p}^{(m)}_{:, s}\right), \bm{1}\right)/ \partial \abs{p^{(m)}_{j, s}}$.
    Given $\anova^m\left(\absop(\bm{p}^{(m)}_{:, s}), \bm{1}\right)$ and additional caches, we can compute $\partial \anova^m\left(\absop\left(\bm{p}^{(m)}_{:, s}\right), \bm{1}\right) / \partial \abs{p^{(m)}_{j, s}}$ in $O\left(m\right)$ time~\citep{atarashi2020link}.
    Therefore, we can update $\bm{p}_{j, s}^{(m)}$ in $O\left(m\nnz{\bm{x}_{:, j}}k\right)$ time, which is the same as that for the PCD algorithm for canonical HOFMs.
    
   Next, we consider the extension of $\tilde{\ell}_{1,2}^2$ for higher-order feature interactions.
   We first present a generalization of~\cref{thm:bound_by_square_of_norm}.
   \begin{restatable}{thm}{boundbypowernorm}
        \label{thm:bound_by_power_of_norm}
        For any $m$-homogeneous quasi-norm $\Omega^m$, there exists $C > 0$ such that $\Omega_{*}^m(\bm{P}) \le C \Omega^m(\bm{P})$ for all $\bm{P} \in \real^{d \times k}$.
    \end{restatable}
    Thus, we simply propose the use of $\tilde{\ell}_{1, m}^m$ as an extension of $\tilde{\ell}_{1,2}^2$ for $\bm{P}^{(m)}$, $m \in \{2, \ldots, M\}$.
    Because $\tilde{\ell}_{1, m}^m$ is column-wise separable, PGD/PSGD-based algorithms for HOFMs with $\tilde{\ell}_{1, m}^m$ require to evaluate the following proximal operator:
    \begin{align}
        \prox_{\lambda \norm{\cdot}_{1}^m}(\bm{p}_{:, s}) = \argmin_{\bm{q} \in \real^d} \frac{1}{2} \norm{\bm{p}-\bm{q}} + \lambda \norm{\bm{q}}_1^m
        \label{eq:prox_higher_order_ti}
    \end{align}
    For this proximal operator, we show a generalization of~\cref{thm:solution_prox_ti}.
    \begin{restatable}{thm}{proxhigherorderti}
        \label{thm:solution_prox_higher_order_ti}
        Assume that $\bm{p}\in \real^d$ is sorted in descending order by absolute value: $\abs{p_1} \ge \abs{p_2} \ge \cdots \ge \abs{p_d}$.
        Then, the solution to the proximal problem~\eqref{eq:prox_higher_order_ti} $\bm{q}^* \in \real^d$ is
        \begin{align}
            q^*_j &= \begin{cases}
            \sign(p_j)\left[|p_j| - \lambda m S_{\theta}^{m-1} \right]& j \le \theta,\\
            0 & \text{otherwise},
            \end{cases}
            \label{eq:solution_prox_higher_order_ti}
        \end{align}
        \begin{align}
            S_{j} \in \left[0, \sum_{i=1}^{j}|p_i|\right] \ \text{s.t.}\ \lambda m j S_{j}^{m-1} + S_{j} - \sum_{i=1}^j |p_i|=0
            \label{eq:condition_S}
        \end{align}
        and $\theta = \max\{j : |p_j| - \lambda m S_j^{m-1} \ge 0\}$.
    \end{restatable}
    Like~\eqref{eq:prox_ti},~\eqref{eq:prox_higher_order_ti} can be evaluated in $O(d \log d)$ time or $O(d)$ time in expectation if $S_j$ can be computed in $O(1)$.
    Unfortunately, $S_j$ in~\eqref{eq:condition_S} cannot be computed analytically when $m > 5$.
    However, HOFMs with $M=3$ is typically recommended~\citep{blondel2016higher} and then it can be computed analytically.
    Even if $m > 5$, one can approximately compute $S_j$ by using a numerical method, e.g., Newton's method.
    
    \subsubsection{Extension of $\Omega_{\mathrm{CS}}$ and $\ell_{2,1}^2$ for HOFMs}
    We next extend $\Omega_{\mathrm{CS}}$ for higher-order feature interactions as
    \begin{align}
        \Omega_{\mathrm{CS}}^{m}(\bm{P}) \coloneqq \sum_{j_m > \cdots > j_1}\norm{\bm{p}_{j_1}}_2\norm{\bm{p}_{j_{2}}}_2\cdots \norm{\bm{p}_{j_m}}_2 = \anova^m\left((\norm{\bm{p}_1}_2, \ldots, \norm{\bm{p}_{d}}_2)^\top, \bm{1}\right),
        \label{eq:Omega_higher_order_cs}
    \end{align}
    and we propose using the following regularization for (order-wise) feature selection in $M$-order HOFMs:
    \begin{align}
        \Omega_{\mathrm{CS}}^2\left(\bm{P}^{(2)}\right) + \cdots + \Omega_{\mathrm{CS}}^M\left(\bm{P}^{(M)}\right),
        \label{eq:sum_of_Omega_higher_order_cs}
    \end{align}
    Clearly, $\Omega_{\mathrm{CS}} = \Omega_{\mathrm{CS}}^2$ holds.
    By using the multi-linearity, we can write it as
    \begin{align}
        \Omega_{\mathrm{CS}}^{m}(\bm{P})  = \norm{\bm{p}_{j}}_2\frac{\partial}{\partial \norm{\bm{p}_j}_{2}} \anova^m\left((\norm{\bm{p}_1}_2, \ldots, \norm{\bm{p}_{d}}_2)^\top, \bm{1}\right) + \mathrm{const}.
    \end{align}
    Therefore, $\Omega_{\mathrm{CS}}^m(\bm{P}^{(m)})$ is also regarded as $\ell_{2,1}$ regularization for one row vector $\bm{p}_j^{(m)}$ like $\Omega_{\mathrm{CS}}$.
    Thus, the PBCD algorithm for HOFMs with~\eqref{eq:sum_of_Omega_higher_order_cs} can be extended similarly as the PCD algorithm for HOFMs with~\eqref{eq:sum_of_Omega_higher_order_ti}.
    
    Like the extension of $\tilde{\ell}_{1, 2}^2$, we simply propose using $\ell_{2, 1}^m$ as an extension of $\ell_{2, 1}^2$.
    Then, PSGD/PGD-based algorithm requires to evaluate the following proximal operator for $\bm{P}^{(m)}$:
    \begin{align}
        \prox_{\lambda \norm{\cdot}_{2,1}^m}(\bm{P}) = \argmin_{\bm{Q} \in \real^{d \times k}} \frac{1}{2} \norm{\bm{Q}-\bm{P}}_2^2 + \lambda \norm{\bm{Q}}_{2,1}^m.
        \label{eq:prox_higher_order_cs}
    \end{align}
    The following is a generalization of~\cref{thm:solution_prox_cs} and states that~\eqref{eq:prox_higher_order_cs} can be evaluated analytically in $O(dk)$ when $m < 6$.
    \begin{restatable}{thm}{proxhigherordercs}
        \label{thm:solution_prox_higher_order_cs}
        The solution to the proximal problem~\eqref{eq:prox_higher_order_cs} $\bm{Q}^* \in \real^{d \times k}$ is
        \begin{align}
            \bm{q}^*_j = \begin{cases}
                \frac{c^*_j}{\norm{\bm{p}_j}_2}\bm{p}_j & \norm{\bm{p}_j}_2 \neq 0,\\
                \bm{0} & \norm{\bm{p}_j}_2 = 0,
            \end{cases}
            \label{eq:solution_prox_higher_order_cs}
        \end{align}
        where $\bm{c}^*=\prox_{\lambda \norm{\cdot}_1^{\textcolor{blue}{m}}}\left(\left(\norm{\bm{p}_1}_2, \ldots, \norm{\bm{p}_d}_2\right)^\top\right)$.
    \end{restatable}
    Therefore,~\eqref{eq:prox_higher_order_cs} can be evaluated by (i) computing the $\ell_2$ norm of each row vector ($O(dk)$), (ii) evaluating $\bm{c}^*=\prox_{\lambda \norm{\cdot}_1^{\textcolor{blue}{m}}}\left(\left(\norm{\bm{p}_1}_2, \ldots, \norm{\bm{p}_d}_2\right)^\top\right)$ ($O(d)$ in expectation), and (iii) computing $\bm{q}_{j}^*$ as~\eqref{eq:solution_prox_higher_order_cs} for all $j \in [d]$ ($O(dk)$).
    When $m > 5$,~\eqref{eq:prox_higher_order_cs} cannot be evaluated analytically since the evaluation of~\eqref{eq:prox_higher_order_ti} is required.

\subsection{All-subsets Model}
    For using HOFMs, a machine learning user must determine the maximum order of interactions, $M$.
    A machine learning user might want to consider all feature interactions.
    Although $d$-order HOFMs use all feature interactions, they require $O(dk(d^2))=O(d^3k)$ computational cost and it might be prohibitive for a high-dimensional case.
    To overcome this problem, \citet{blondel2016higher} also proposed the all-subsets model, which uses all feature interactions efficiently.
    The output of the all-subsets model is defined by
    \begin{align}
        f_{\mathrm{all}}(\bm{x}; \bm{P}) \coloneqq \sum_{s=1}^k K_{\mathrm{all}}(\bm{x}, \bm{p}_{:, s}) = \sum_{m=0}^d \sum_{s=1}^k \anova^m(\bm{p}_{:, s}, \bm{x}),
        \label{eq:all_subsets}
    \end{align}
    where $\bm{P} \in \real^{d \times k}$ is the learnable parameter, and $K_{\mathrm{all}}: \real^d \times \real^d \to \real$ is the all-subsets kernel~\citep{blondel2016higher}:
    \begin{align}
        K_{\mathrm{all}}(\bm{x}, \bm{p}) \coloneqq \prod_{j=1}^d (1+x_jp_j) = \sum_{S \in 2^{[d]}}\prod_{j \in S}x_jp_j. 
        \label{eq:all_subsets_kernel}
    \end{align}
    Clearly, the all-subsets kernel~\eqref{eq:all_subsets_kernel} can be evaluated in $O(d)$ (strictly speaking, $O(\nnz{\bm{x}})$) time, so the all-subsets model~\eqref{eq:all_subsets} can be evaluated in $O(dk)$ (strictly speaking, $O(\nnz{\bm{x}})k$) time.
    \citet{blondel2016higher} also proposed efficient CD and SGD-based algorithms for the all-subsets model.
    
\subsubsection{Extension of $\Omega_{\mathrm{TI}}$ and $\tilde{\ell}_{1,2}^2$ for the all-subsets model}
    The output of the all-subsets model~\eqref{eq:all_subsets} can be rearranged as
    \begin{align}
        f_{\mathrm{all}}(\bm{x}; \bm{P}) = \sum_{S \in 2^{[d]}}w_{S}\prod_{j \in S} x_j, \quad \text{where } w_{S} = \sum_{s=1}^k \prod_{j \in S} p_{j, s}.
    \end{align}
    Thus, we propose $\Omega_{\mathrm{TI}}^{\mathrm{all}}$, which is an extension of $\Omega_{\mathrm{TI}}$ to the all-subsets model:
    \begin{align}
        \Omega_{\mathrm{TI}}^{\mathrm{all}}(\bm{P}) \coloneqq \sum_{S \in 2^{[d]}}\sum_{s=1}^k \prod_{j \in S} \abs{p_{j, s}} = \sum_{s=1}^k \sum_{S \in 2^{[d]}}\prod_{j \in S} \abs{p_{j, s}}\cdot 1 = \sum_{s=1}^k K_{\mathrm{all}}(\absop(\bm{p}_{:, s}), \bm{1}) = \sum_{m=1}^d \Omega_{\mathrm{TI}}^{m}(\bm{P}).
        \label{eq:Omega_all_ti}
    \end{align}
    Since the output of the all-subsets model is multi-linear and $\Omega_{\mathrm{TI}}^{\mathrm{all}}(\bm{P})$ can be regarded as a $\ell_{1}$ regularization for one parameter, the all-subsets model with this regularization can be optimized efficiently by using the PCD algorithm. For $p_{j, s}$, $\Omega_{\mathrm{TI}}^{\mathrm{all}}(\bm{P})$ is written as $\{\partial K_{\mathrm{all}}(\absop(\bm{p}_{:, s}), \bm{1})/\partial \abs{p_{j, s}}\}\cdot \abs{p_{j,s}} + \mathrm{const}$, and $\partial K_{\mathrm{all}}(\absop(\bm{p}_{:, s}), \bm{1})/\partial \abs{p_{j, s}}$ can be computed in $O(1)$ time if $K_{\mathrm{all}}(\absop(\bm{p}_{:, s}), \bm{1})$ is given~\citep{blondel2016higher}.
    Thus, we can extend the PCD algorithm for the canonical all-subsets model~\citep{blondel2016higher} to the PCD algorithm for the all-subsets model with $\Omega_{\mathrm{TI}}^{\mathrm{all}}$ regularization similarly as for FMs and HOFMs.
    
    We next extend $\tilde{\ell}_{1,2}^2$ to the all-subsets model.
    Based on~\cref{thm:bound_by_power_of_norm} and~\eqref{eq:Omega_all_ti}, we extend it as $\sum_{m=1}^{d} \norm{\bm{P}^{\top}}_{1, m}^m$.
    We leave the development of an efficient algorithm for evaluating the corresponding proximal problem for future work.
    Because the all-subsets model is multi-linear w.r.t $\bm{p}_{1}, \ldots, \bm{p}_{d}$ (and of course $p_{1, 1}, \ldots, p_{d, k}$), it is also optimized by using the CD algorithm efficiently~\citep{blondel2016higher}.
    
\subsubsection{Extension of $\Omega_{\mathrm{CS}}$ and $\ell_{2,1}^2$ for the all-subsets model}
    Next, we extend $\Omega_{\mathrm{CS}}$ to all-subsets model as
    \begin{align}
        \Omega_{\mathrm{CS}}^{\mathrm{all}}(\bm{P}) \coloneqq \sum_{m=1}^d \Omega_{\mathrm{CS}}^m(\bm{P}) = K_{\mathrm{all}}\left((\norm{\bm{p}_1}_2, \ldots, \norm{\bm{p}_{d}}_2)^\top, \bm{1}\right).
        \label{eq:Omega_all_cs}
    \end{align}
    The all-subsets model with $\Omega_{\mathrm{CS}}^{\mathrm{all}}$ is efficiently optimized by using the PBCD algorithm in a similar manner because the all-subsets model is multi-linear w.r.t $\bm{p}_j$ and $\Omega_{\mathrm{CS}}^{\mathrm{all}}$ is also multi-linear w.r.t $\norm{\bm{p}_j}_2$ for all $j \in [d]$.
    
    We also extend $\ell_{2,1}^2$ to the all-subsets model as $\sum_{m=1}^{d}\norm{\bm{P}}_{2, 1}^m$ but we leave the development of an efficient algorithm for evaluating the corresponding proximal problem for future work.

\section{Related Work}
    \label{sec:related}
    Because FMs are equivalent to the QR with low-rank factorized $\bm{W}$ and the QR is essentially a linear model, one can na\"ively use any feature selection method~\citep{tibshirani1996regression,liu2017dual,mallat1993matching} for feature interaction selection in the QR.
    Especially, the QR with the $\ell_1$ norm and the trace (nuclear) norm $\norm{\cdot}_{\mathrm{tr}}$ regularization is one of the natural choice for learning a low-rank $\bm{W}$ with feature interaction selection.
    Formally, the corresponding objective function is
    \begin{align}
        L_{\mathrm{QR}}(\bm{w}, \bm{W}; \lambda_{w}) +  \lambda_{\mathrm{tr}}\norm{\bm{W}}_{\mathrm{tr}} + \tilde{\lambda} \norm{\bm{W}}_1,
        \label{eq:objective_slqr}
    \end{align}
    where $\lambda_w, \lambda_{\mathrm{tr}},$ and $\tilde{\lambda} > 0$ are regularization-strength hyperparameters.
    We call the QR learned by minimizing~\eqref{eq:objective_slqr} the sparse and low-rank QR (SLQR).
    \citet{richard2012estimation} firstly proposed proximal algorithms for convex objective functions with $\norm{\cdot}_1$ and $\norm{\cdot}_{\mathrm{tr}}$ such as~\eqref{eq:objective_slqr} for estimating a simultaneously sparse and low-rank matrix.
    The incremental PGD algorithm proposed by~\citet{richard2012estimation} updates $\bm{W}$ as
    \begin{align}
        \bm{W} \leftarrow \prox_{\eta\tilde{\lambda}}\left(\prox_{\eta \lambda_{\mathrm{tr}}\norm{\cdot}_{\mathrm{tr}}}\left(\bm{W}-\eta\nabla L_{\mathrm{QR}}\right)\right).
        \label{eq:prox_slqr}
    \end{align}
    Because the objective function~\eqref{eq:objective_slqr} is convex, its all local minima are global minima unlike the FM-based existing methods.
    It is a great advantage of the QR-based methods compared to the FM-based methods.
    However, the QR (and of course the SLQR) requires $O(d^2)$ memory and $O(\nnz{\bm{x}}^2)$ time for evaluation, so it is hard to use QR-based methods for a high-dimensional case.
    Moreover, the evaluation of~\eqref{eq:prox_slqr} takes $O(d^3)$ computational cost~\citep{parikh2014proximal}.
    Therefore, it is harder to use the SLQR for a high-dimensional case.
    
    \citet{agrawal2019kernel,yang2019interaction,morvan2018whinter}, and~\citet{suzumura2017selective} proposed feature interaction selection methods in the QR and/or QR-like models.
    They can be more efficient than above-mentioned na\"ive methods~\citep{tibshirani1996regression,liu2017dual,mallat1993matching}.
    However, the methods of~\citet{agrawal2019kernel} and~\citet{yang2019interaction} require super linear computational cost w.r.t $d$ or $N$, and those of~\citet{morvan2018whinter} and~\citet{suzumura2017selective} can be used only when $\bm{x} \in [0, 1]^d$.
    Moreover, as described in~\cref{sec:background}, the QR cannot estimate the weights for interactions that are not observed from the training data.

    \citet{cheng2014gradient} proposed a greedy (forward) feature interaction selection algorithm in FMs for context-aware recommendation.
    They call FMs with their algorithm gradient boosting FMs (GBFMs). 
    In general, greedy selection algorithms produce a sub-optimal solution and are often inferior to shrinkage methods (e.g., methods based on sparse regularization)~\citep{hastie2009elements}.
    Moreover, at each greedy interaction selection step, their algorithm sequentially selects each feature that constructs the interaction, namely, their greedy step is approximately greedy.
    
    \citet{chen2019bayesian} proposed a Bayesian feature interaction selection method in FMs for personalized recommendation.
    Assume that there are $U$ users and $I$ items, and $\bm{x}_i, \ldots, \bm{x}_{I} \in \real^{d}$ are feature vectors of items.
    Then, they proposed (Bayesian) personalized FMs, which predict the preference of each user for each item.
    The output of personalized FMs for $u$-th user $f_u: \real^d \to \real$ is defined as
    \begin{align}
        f_{u}(\bm{x}; \bm{\pi}_{u}, \bm{\mu}, \bm{P}, \phi) =  \inner{\bar{\bm{w}}_u}{\bm{x}} + \sum_{j_2 > j_1} \bar{w}_{u, j_1, j_2} x_{j_1}x_{j_2},
        \label{eq:personailzed_fm}
    \end{align}
    where
    \begin{align*}
        \bar{w}_{u, j} &= \pi_{u, j} \mu_{j},\ \bar{w}_{u, j_1, j_2} = \pi_{u, j_1, j_2}\mu_{j_1, j_2}, \\
        \pi_{u, j_1, j_2} &= \pi_{u, j_1}\pi_{u, j_2}[1 + (1 - \pi_{u,j_1}\pi_{u, j_2})(\pi_{u,j_1} + \pi_{u, j_2})],
    \end{align*}
    where $\bm{\mu} \in \real^d$, $\bm{\pi}_u \in [0, 1]^{d}$ and $\bm{P}_{u} \in \real^{d \times k}$ are learnable parameters and $\mu_{j_1, j_2} \in \real^d$ is computed defined by using $\bm{\mu}$, $\bm{p}_{j_1}$, and $\bm{p}_{j_2}$ (for more detail, please see~\citep{chen2019bayesian}).
    Intuitively, $\pi_{u, j}$ represents the selection probability of $j$-th feature.
    Similarly, $\pi_{u, j_1, j_2}$ represents that of the interaction between $j_1$-th and $j_2$-th features.
    Unfortunately, their method cannot actually select feature interactions without selecting features because $\pi_{u, j_1, j_2} = 0$ if and only if $\pi_{u, j_1}=0$ or $\pi_{u, j_2}=0$ (on $\pi_{u,j} \in [0, 1]$), namely, their method actually is for feature selection.
    
    Several researchers proposed FMs and deep-neural-network-extension of FMs that adapt feature interaction weights depending on an input feature vector $\bm{x}$~\citep{xiao2017attentional,song2019autoint,hong2019interaction,xue2020autohash}.
    While such methods outperformed FMs on some recommender system tasks, they also require $O\left(\nnz{\bm{x}}^2\right)$ time for evaluation.
    Moreover, their feature interaction weights cannot be completely zero.
    
\section{Experiments}
    \label{sec:experiments}
    In this section, we experimentally demonstrate the advantages of the proposed methods for feature interaction selection (and additionally feature selection) compared with existing methods. 
    Firstly, we show the results for synthetic datasets. 
    These results indicate that (i) TI-sparse FMs are useful for feature interaction selection in FMs while existing methods are not and (ii) CS-sparse FMs are more useful than existing methods for feature selection in FMs.
    Secondly, we show the results for some real-world datasets on an interpretability constraint setting.
    Finally, we compares the convergence speeds of some optimization algorithms for proposed methods on both synthetic and real-world datasets.

    We implemented the existing and proposed methods in Nim~\footnote{\url{https://nim-lang.org/}} and ran the experiments on an Arch Linux desktop with an Intel Core i7-4790 (3.60 GHz) CPU and 16 GB RAM.
    The \textbf{FM} implementation was as fast as libFM~\citep{rendle2012factorization}, which is the implementation of FMs in C++.
    Our implementation is available at \url{https://github.com/neonnnnn/nimfm}.

\subsection{Comparison of Proximal Operators}
    \label{subsec:comparison_proximal_operator}
    \begin{figure}[t]
        \centering
        \subfloat{
        \includegraphics[width=80mm]{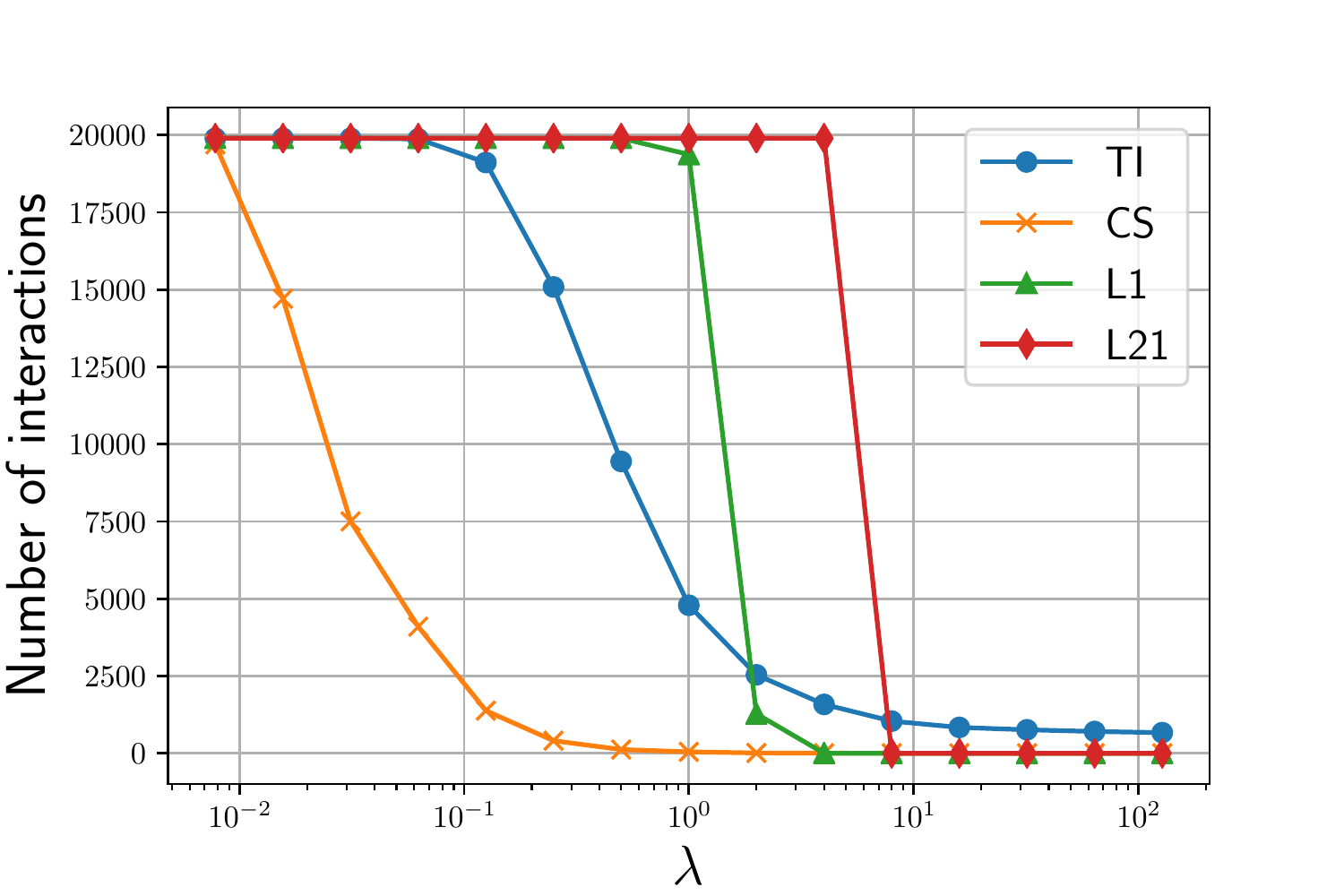} 
        \includegraphics[width=80mm]{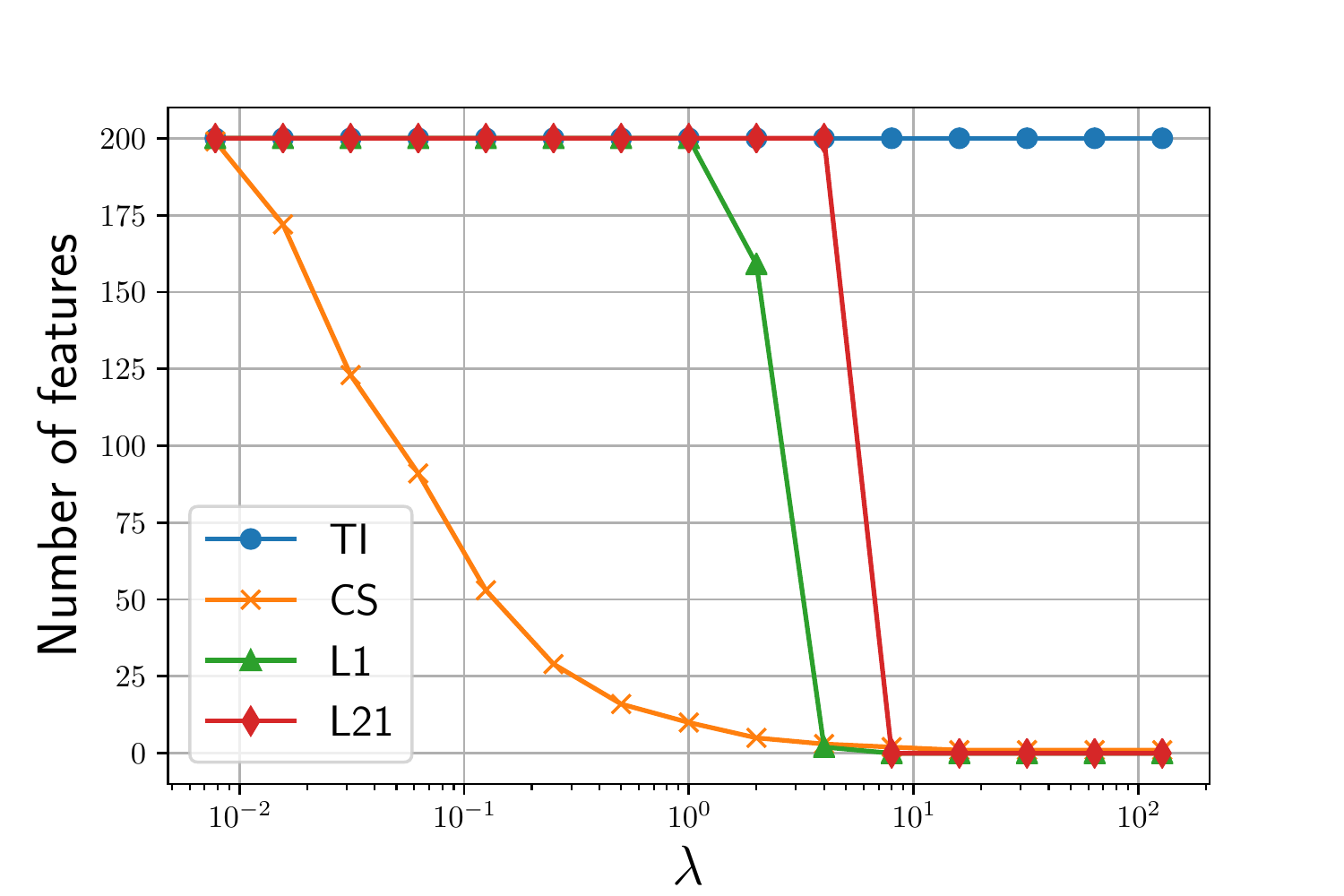}}
        \caption{Comparison of proximal operators associated with \textbf{TI}, \textbf{CS}, \textbf{L21}, and \textbf{L1} regularizer. We evaluated the proximal operators at a randomly sampled $\bm{P}$ with various $\lambda$. Left graph shows the number of used feature interactions and right graph shows the number of used features in $\bm{Q}^*(\bm{Q}^*)^\top$, where $\bm{Q}^*$ is the output of the proximal operator: $\bm{Q}^* = \prox_{\lambda \Omega(\cdot)}(\bm{P})$ and $\Omega$=$\ell_1$ (\textbf{L1}~\citep{pan2016sparse}), $\ell_{2,1}$ (\textbf{L21}~\citep{xu2016synergies,zhao2017meta}), $\tilde{\ell}_{1,2}^2$ (\textbf{TI}, proposed in~\cref{sec:ti}), or $\ell_{2,1}^2$~(\textbf{CS}, proposed in~\cref{sec:cs}).}
        \label{fig:toy_with_proposed}
    \end{figure}
    Firstly, we compared the outputs of proximal operators of the existing methods and the proposed methods in the same way as in~\cref{subsec:verification}.
    We evaluated proximal operators of not only \textbf{L1}~\citep{pan2016sparse} and \textbf{L21}~\citep{xu2016synergies,zhao2017meta} but also $\tilde{\ell}_{1,2}^2$ (\textbf{TI}) and $\ell_{2,1}^2$ (\textbf{CS}).
    Their corresponding proximal operators are~\eqref{eq:prox_l1},~\eqref{eq:prox_l21},~\eqref{eq:prox_ti}, and~\eqref{eq:prox_cs}, respectively.

    Results are shown in~\cref{fig:toy_with_proposed}.
    Unlike the existing methods (\textbf{L1} and \textbf{L21}), the proposed methods (\textbf{TI} and \textbf{CS}) could produce sparse but moderately sparse feature interaction matrices.
    Moreover, the number of used features in \textbf{TI} was $200$ for all $\lambda \in \{2^{-7}, \ldots, 2^{7}\}$.
    It means that \textbf{TI} successfully selects feature interactions in FMs.
    The number of used features in \textbf{CS} decreased gradually as $\lambda$ increased.
    It indicates that \textbf{CS} can be more useful than \textbf{TI} and existing methods for feature selection.
    
\subsection{Synthetic Datasets}
    \label{subsec:synthetic}
    We next evaluated the performance of the proposed methods and existing methods on feature interaction selection and feature selection problems using synthetic datasets.
    We ran the experiments on an Ubuntu 18.04 server with an AMD EPYC 7402P 24-Core Processor (2.80 GHz) and 128 GB RAM.
    
\subsubsection{Settings}
    \paragraph{Datasets.} To evaluate the proposed TI regularizer and CS regularizer in the feature interaction selection and feature selection scenarios, we created datasets such that true models used partial second-order feature interactions. We defined the true model as the QR without the linear term with the block diagonal matrix $\bm{W}$. We defined each block diagonal matrix as a $(d_{\mathrm{true}}/b, d_{\mathrm{true}}/b)$ all-ones matrix, where $b$ is the number of blocks (we set $b$ such that $d_{\mathrm{true}}$ was dividable by $b$). Intuitively, there were $b$ distinct groups of features, and $w_{j_1, j_2}=1$ if $j_1$ and $j_2$ were in the same group of features and equaled zero otherwise. Precisely, if $\lceil j_1 / (d_{\mathrm{true}}/b) \rceil = \lceil j_2 / (d_{\mathrm{true}}/b) \rceil$, $j_1$ and $j_2$ were in the same group of features. For the distribution of feature vector $\bm{x}$, we used a Gaussian distribution, $\mathcal{N}(\bm{\mu}, \bm{\Sigma})$. We set $\bm{\mu} = \bm{0}$ and $\Sigma_{j, j} =1$ for all $j \in [d_{\mathrm{true}}]$ and set $\Sigma_{j_1, j_2} = 0.2$ if $j_1\neq j_2$ were in the same feature group and zero otherwise. Moreover, we concatenated $d_{\mathrm{noise}}$-dimensional noise features to feature vector $\bm{x}$ and used the concatenated vector as the observed feature vector (namely, the dimension of the observed feature vectors $d=d_{\mathrm{true}}+d_{\mathrm{noise}}$). We set the distribution of each noise feature to $\mathcal{N}(0, 1)$ (the noise features were independent of each other). Furthermore, we added noise to the observation of target $f_{\mathrm{true}}(\bm{x})$. We used $\mathcal{N}(0, 0.1^2)$ for the target noises. We considered two settings.
    \begin{itemize}
        \item \textbf{Feature interaction selection setting}: $d_{\mathrm{true}}=80$, $b=8$, and $d_{\mathrm{noise}}=20$. In this setting, there were eight groups of features, so the methods that perform only feature selection in FMs like CS-sparse FMs and $\ell_{2,1}$-sparse FMs were not useful. Again, our main goal is to develop sparse FMs that are useful in this setting.
        \item \textbf{Feature selection setting}: $d_{\mathrm{true}}=20$, $b=1$, and $d_{\mathrm{noise}}=80$. In this setting, there was only one group of features, so the methods that perform only feature selection in FMs were also useful.
    \end{itemize}
     
    \paragraph{Evaluation Metrics.} We mainly used three metrics.
    They were computed using the parameters $\bm{W}$ of the true models.
    \begin{itemize}
        \item \textbf{Estimation error}: $\lVert\bm{W}-\hat{\bm{P}}\hat{\bm{P}}^\top\rVert_{2, >}/ \norm{\bm{W}}_{2, >}$, where $\norm{\cdot}_{2, >}$ is the $\ell_2$ norm for only the strictly upper triangular elements, $\bm{W}$ is the true feature interaction matrix, and $\hat{\bm{P}}$ is the learned parameter in FMs and sparse FMs. Lower is better.
        \item \textbf{F1-score}: the F1-score of the support prediction problem. To be more precise, we regarded $(j_1, j_2)$ as a positive instance if $w_{j_1, j_2} \neq 0$ and as a negative instance otherwise, and we regarded $(j_1, j_2)$ as a positive predicted instance if $\inner{\hat{\bm{p}}_{j_1}}{\hat{\bm{p}}_{j_2}} \neq 0$ and as  a negative predicted instance otherwise for all $j_2 > j_1$. Higher is better.
        \item \textbf{Percentage of successful support recovery (PSSR)}~\citep{liu2017dual}: the percentages of the results such that $\{(j_1, j_2) : w_{j_1, j_2} \neq 0, j_2 > j_1\} = \{(j_1, j_2) : \inner{\hat{\bm{p}}_{j_1}}{\hat{\bm{p}}_{j_2}} \neq 0, j_2 > j_1\}$ among the different datasets. Higher is better.
    \end{itemize}
     
    \paragraph{Methods Compared.} We compared the following eight methods.
    \begin{itemize}
        \item \textbf{TI}: $\tilde{\ell}_{1,2}^{2}$-sparse FMs (TI-sparse FMs) optimized using PCD algorithm.
        \item \textbf{CS}: $\ell^2_{2,1}$-sparse FMs (CS-sparse FMs) optimized using PBCD algorithm.
        \item \textbf{L21}: $\ell_{2,1}$-sparse FMs~\citep{xu2016synergies,zhao2017meta} optimized using PBCD algorithm.
        \item \textbf{L1}: $\ell_1$-sparse FMs~\citep{pan2016sparse} optimized using PCD algorithm.
        \item \textbf{$\Omega_*$-nmAPGD}: $\Omega_*$-sparse FMs optimized using non-monotone accelerated inexact PGD algorithm~\citep{li2015accelerated}.
        \item \textbf{$\Omega_*$-SubGD}: $\Omega_*$-sparse FMs optimized using SubGD algorithm.
        \item \textbf{SLQR}: the SLQR optimized using the incremental PGD algorithm~\citep{richard2012estimation} with acceleration~\citep{nesterov1983method,beck2009fast} and restart~\citep{giselsson2014monotonicity}.
        The PGD algorithm with acceleration is known as a fast iterative shrinkage thresholding algorithm (FISTA)~\citep{beck2009fast}.
        \item \textbf{FM}: canonical FMs~\citep{rendle2010factorization} optimized using CD algorithm.
    \end{itemize}
    For all methods, we omitted the linear term $\inner{\bm{w}}{\bm{x}}$ since the true models do not use it. Since the target values are in $\real$, we used the squared loss for $\ell(\cdot, \cdot)$.

\subsubsection{Results with Tuned Hyperparameters}
    \label{subsubsec:synthetic_tuned}
    \begin{figure*}[t]
        \centering
        \subfloat[Feature interaction selection setting: $d_{\mathrm{true}}=80$, $b=8$, and $d_{\mathrm{noise}}=20$.]{
        \includegraphics[width=52mm]{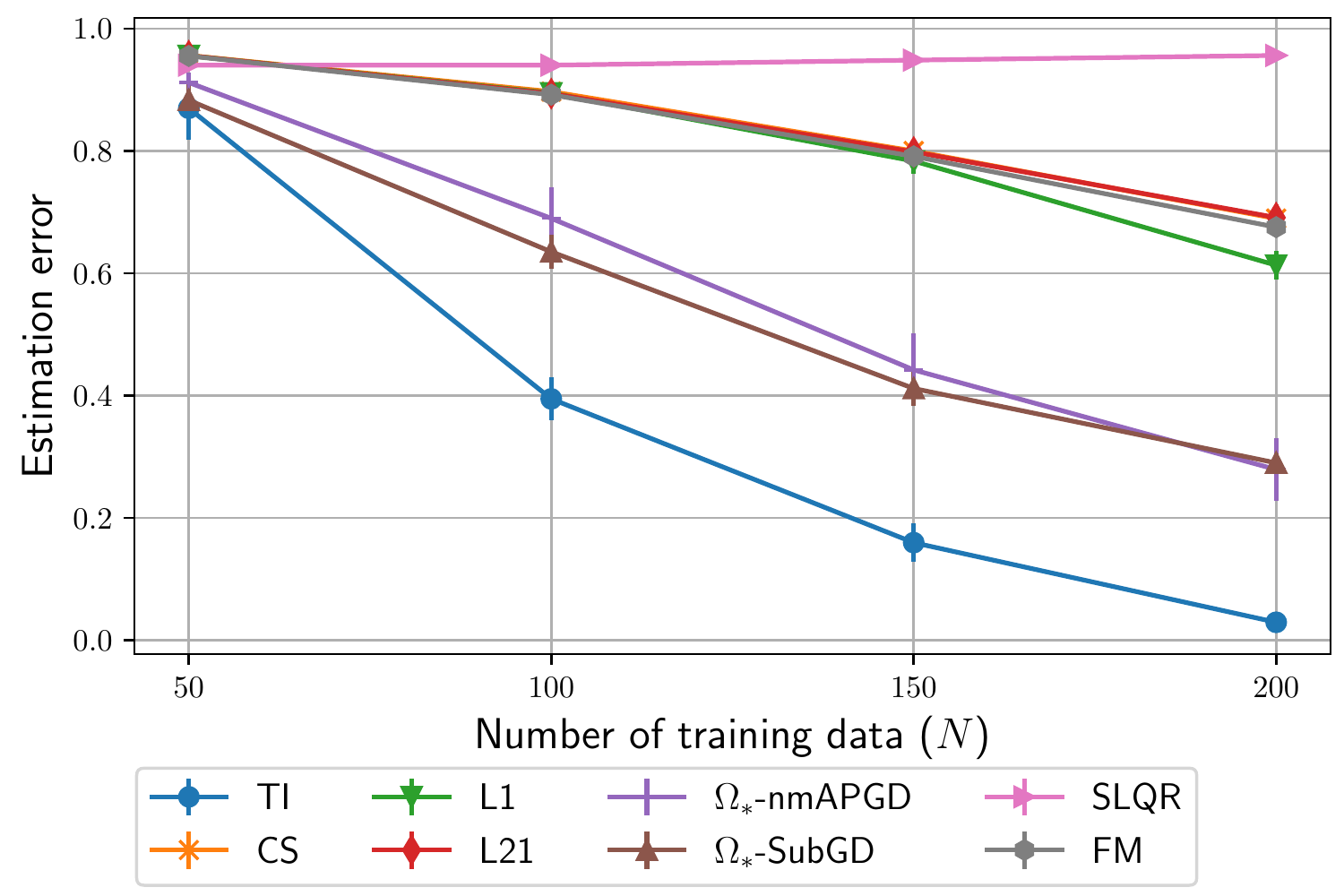} 
        \includegraphics[width=52mm]{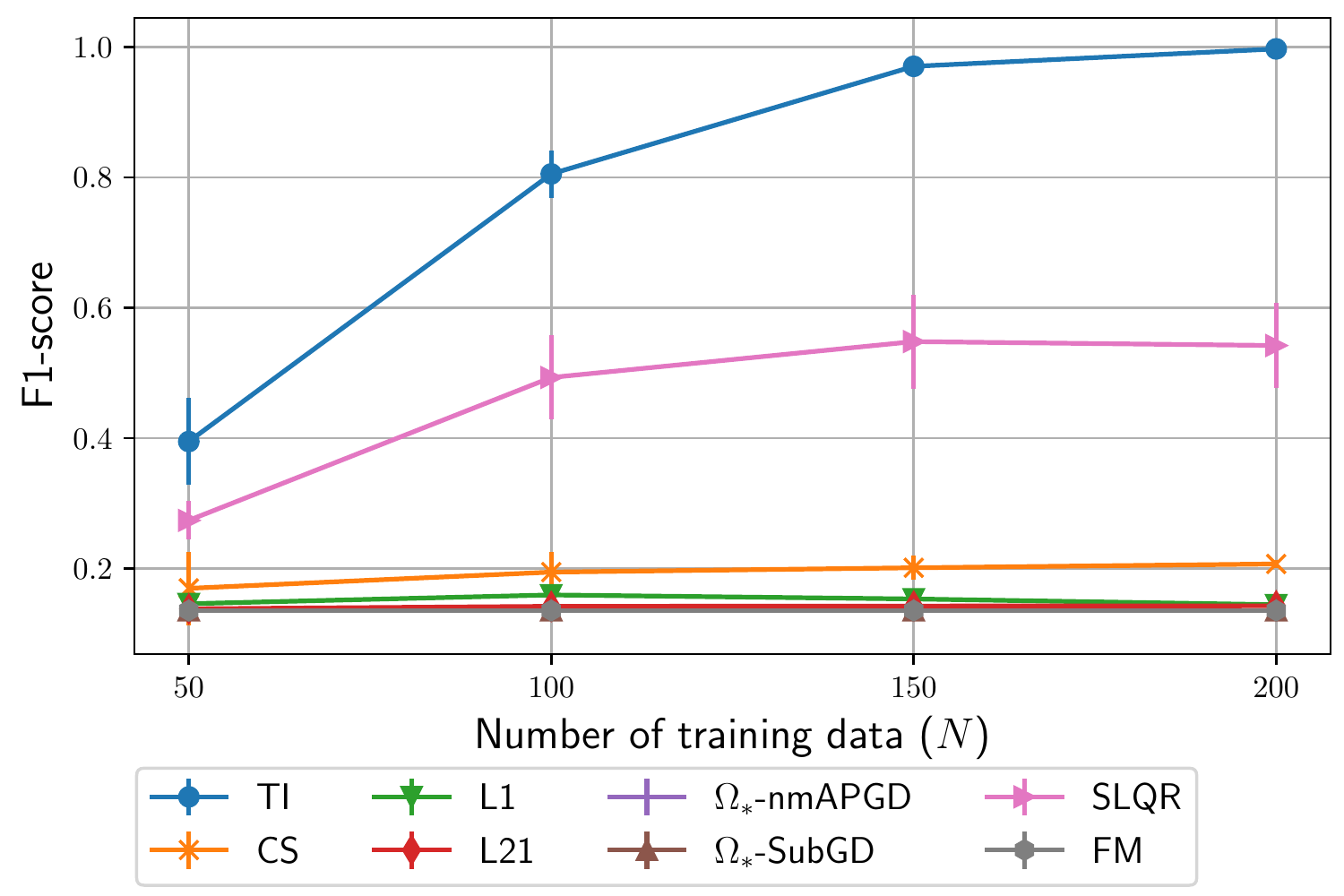}
        \includegraphics[width=52mm]{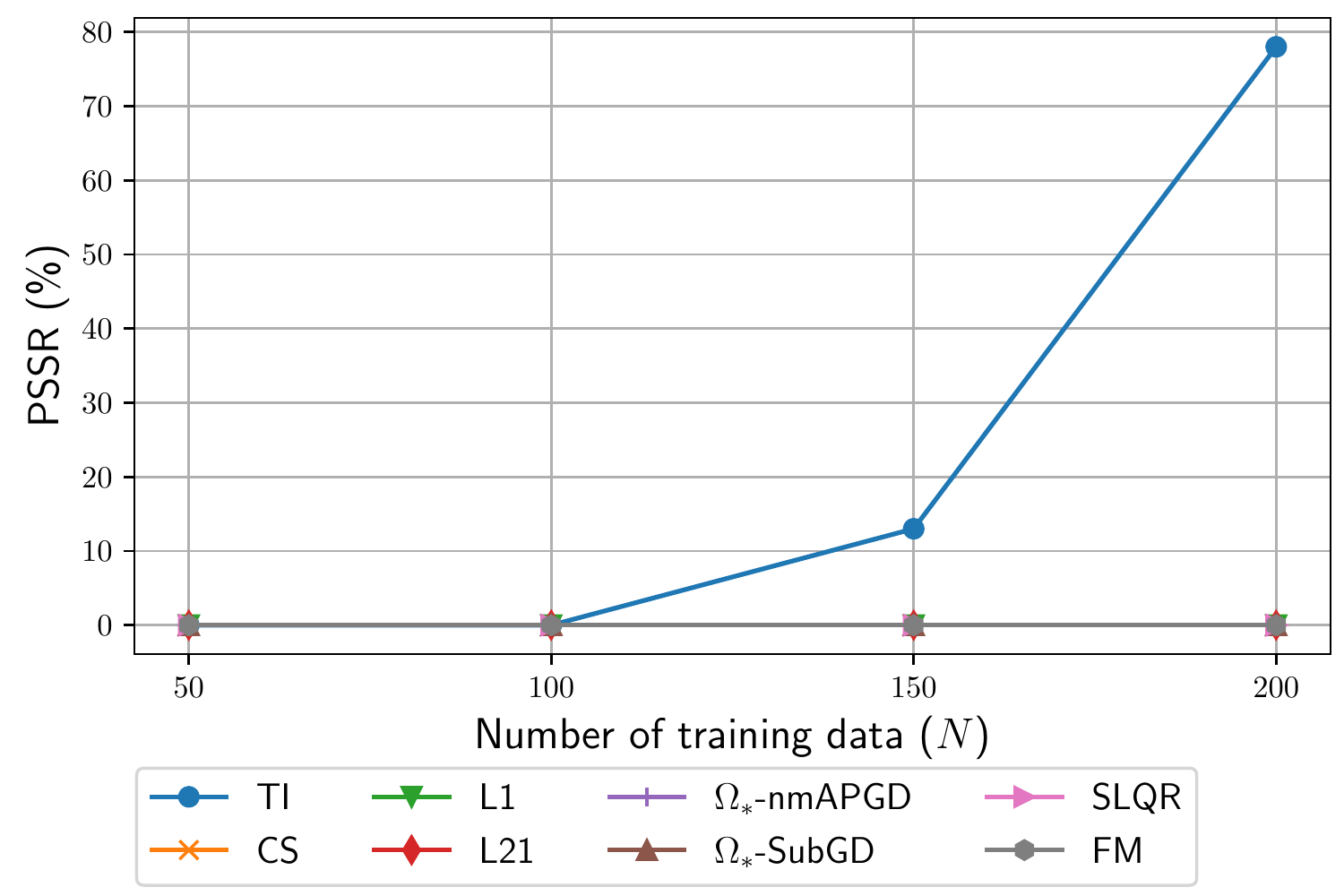}
        \label{fig:synthetic_tuned_feature_interaction_selection}
        }\\
        \subfloat[Feature selection setting: $d_{\mathrm{true}}=20$, $b=1$, and $d_{\mathrm{noise}}=80$.]{
        \includegraphics[width=52mm]{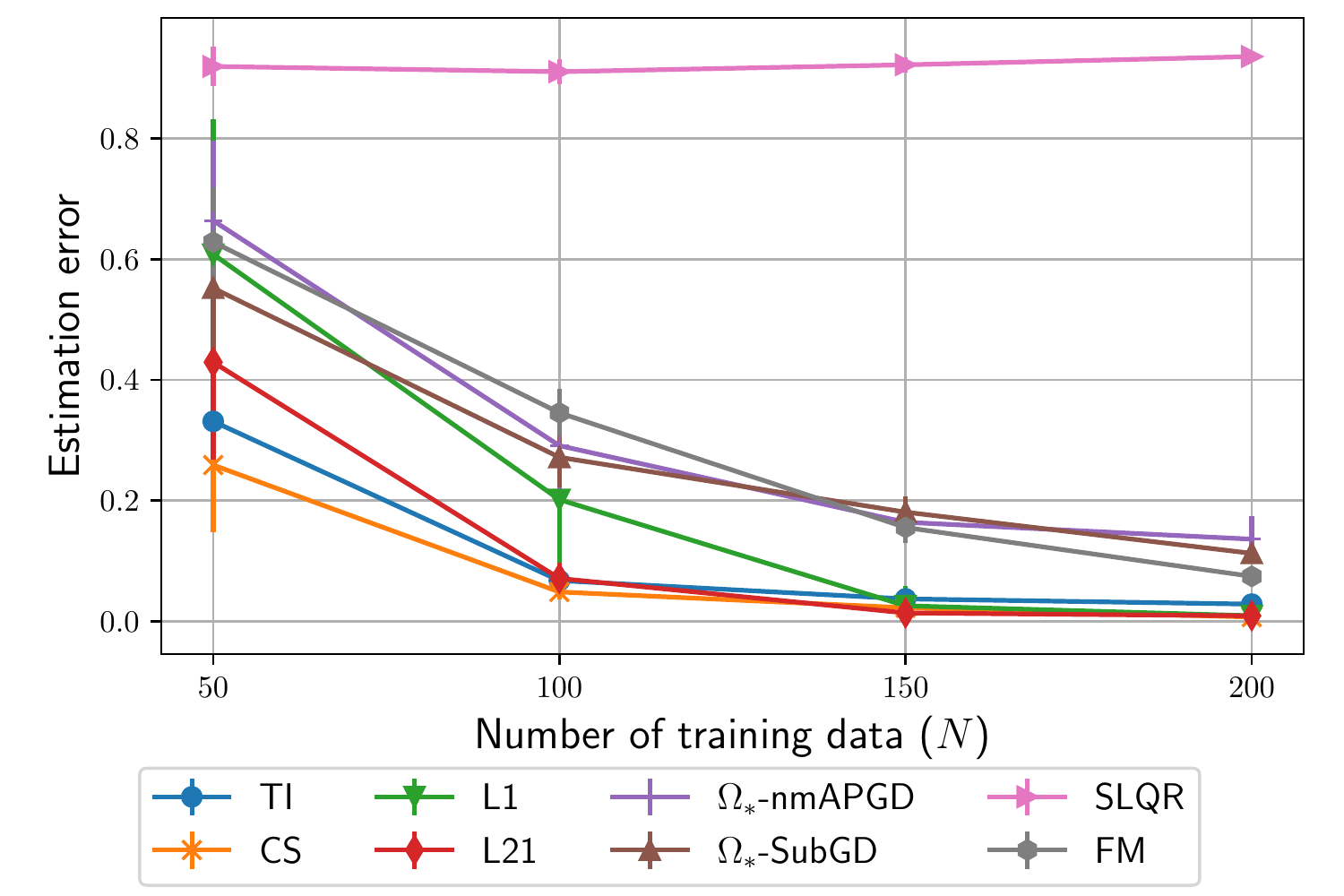} 
        \includegraphics[width=52mm]{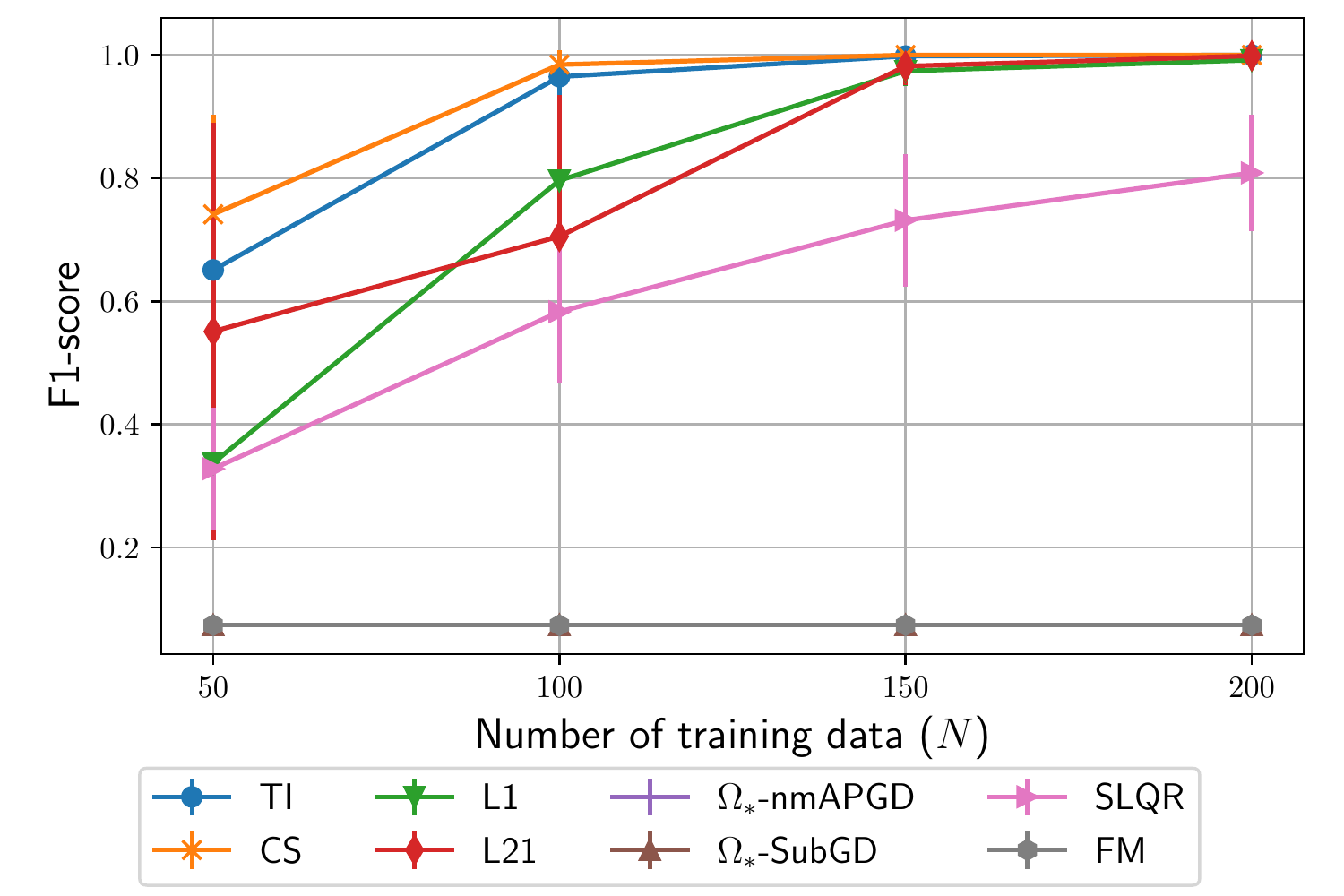}
        \includegraphics[width=52mm]{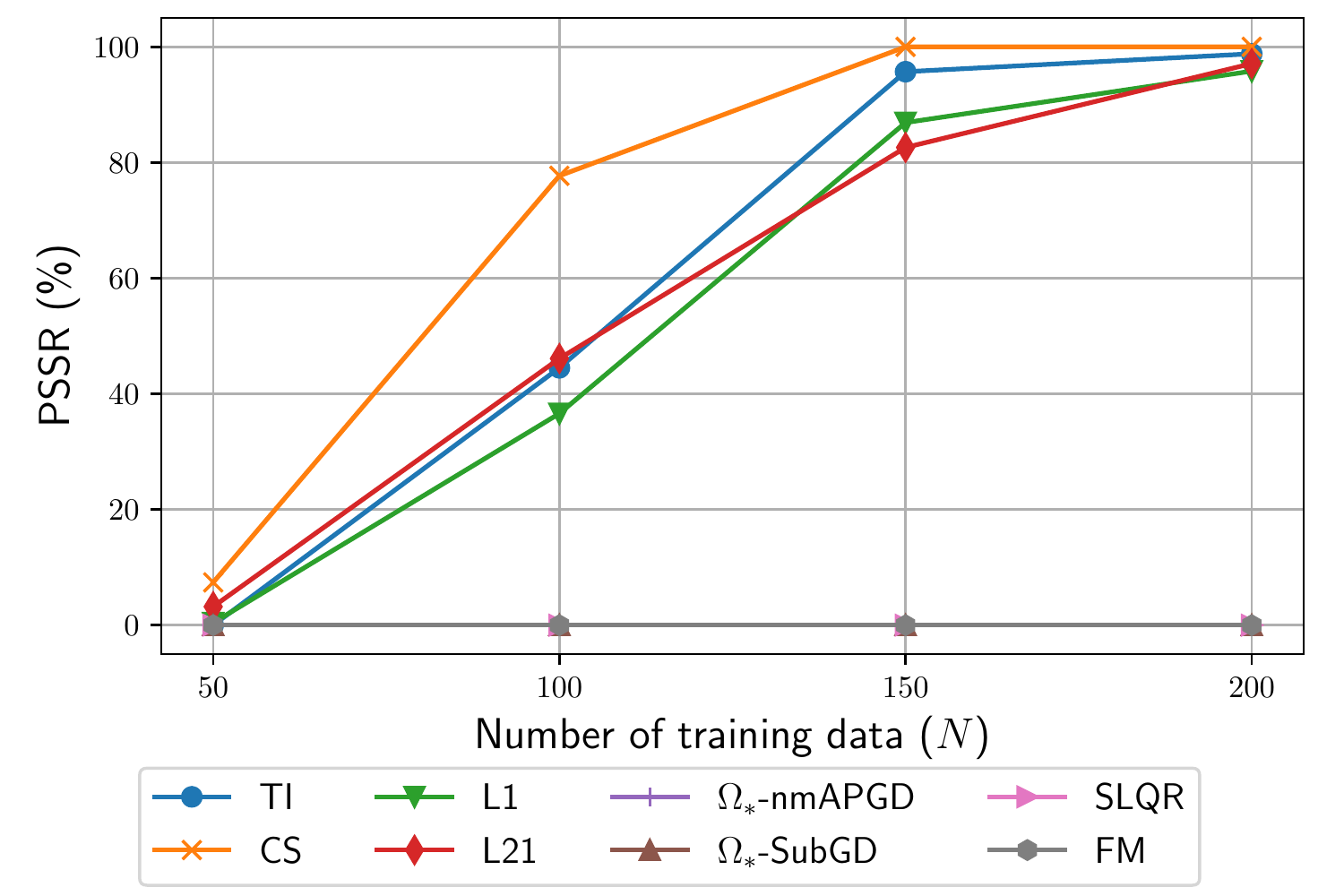} 
        \label{fig:synthetic_tuned_feature_selection}
        }
        \caption{Experimental results with tuned parameters for \textbf{TI}, \textbf{CS}, \textbf{L21}, \textbf{L1}, \textbf{$\Omega_*$-nmAPGD}, \textbf{$\Omega_*$-SubGD}, \textbf{SLQR}, and \textbf{FM} methods on synthetic datasets using different numbers of training examples: (a) feature interaction selection setting datasets; (b) feature selection setting datasets. Left graphs show estimation error (lower is better), center graphs show F1-score (higher is better), and right graphs show PSSR (higher is better). In (a), the plots of estimation errors of \textbf{CS} and \textbf{L21}, those of F1-scores of \textbf{L1}, \textbf{L21}, \textbf{$\Omega_*$-nmAPGD}, and \textbf{$\Omega_*$-SubGD}, and those of PSSRs except for \textbf{TI} overlap, respectively. In (b), similarly, the plots of F1-scores of \textbf{$\Omega_*$-nmAPGD} and \textbf{$\Omega_*$-SubGD} and those of PSSRs of \textbf{SLQR}, \textbf{$\Omega_*$-nmAPGD}, and \textbf{$\Omega_*$-SubGD} overlap, respectively.}
        \label{fig:synthetic_tuned}
    \end{figure*}
    We compared the above-mentioned three metrics among the eight methods.
    Since these metrics use the parameter of the true model, $\bm{W}$, and the results clearly depended on the hyperparameter settings, we followed~\citet{liu2017dual} for evaluation: created $150$ datasets (not the number of instances in a dataset but the number of datasets), divided them into $50$ validation datasets and $100$ test datasets, tuned the hyperparameters on the validation datasets (which also used $\bm{W}$), and finally learned models on the test datasets with tuned hyperparameter and evaluated them.
    We tuned hyperparameters $\lambda_p$ and $\tilde{\lambda}_p$.
    For the sparse FM methods (i.e., \textbf{TI}, \textbf{CS}, \textbf{L21}, \textbf{L1}, \textbf{$\Omega_*$-nmAPGD}, and \textbf{$\Omega_*$-SubGD}), we set them to $10^{-2}, 10^{-1}, \ldots, 10^2$.
    Since the \textbf{FM} method had only $\lambda_p$, we tuned it more carefully than the sparse FM methods: $10^{-3+0\cdot 7/24}, 10^{-3 + 1\cdot 7/24}, \ldots$, $10^{-3 + 24\cdot 7/24}=10^4$.
    We set rank hyperparameter $k$ to $30$ for \textbf{FM} and the sparse FM methods and used $\mathcal{N}(0, 0.01^2)$ to initialize $\bm{P}$.
    We used line search techniques for computing step sizes in \textbf{$\Omega_*$-SubGD} and \textbf{$\Omega_*$-nmAPGD}.
    We used the SubGD method for solving the proximal operator~\eqref{eq:prox_omegastar} in \textbf{$\Omega_*$-nmAPGD}.
    In this SubGD method, we used a diminishing step size: $\eta = 0.1 / \sqrt{t}$ at the $t$-th iteration, and set the number of iteration to $20$.
    In \textbf{SLQR}, we set $\lambda_{\mathrm{tr}}$ and $\tilde{\lambda}$ to $10^{-2}, 10^{-1}, \ldots, 10^{2}$.
    We ran the experiment $10$ times with different initial random seeds for FMs and sparse FMs since their learning results depend on the initial value of $\bm{P}$ (we thus learned and evaluated above-mentioned methods $100 \times 10 = 1,000$ times).
    We also ran it with different numbers of instances in each dataset: $N=50, 100, 150$, and $200$.
    For fair comparison, we set the time budgets for optimization to $N / 50$ second (CPU time) for all methods.
    However, we stopped optimization if the absolute difference between the current parameter and previous parameter was less than $10^{-3}$ for \textbf{FM}, \textbf{L1}, \textbf{L21}, \textbf{TI}, and \textbf{CS}, and $10^{-7}$ for \textbf{SLQR}, \textbf{$\Omega_*$-SubGD}, and \textbf{$\Omega_*$-nmAPGD}.
        
    As shown in~\cref{fig:synthetic_tuned_feature_interaction_selection}, \textbf{TI} performed the best on the feature interaction selection setting datasets for all metrics.
    Note that the plots of estimation errors of \textbf{CS} and \textbf{L21}, those of F1-scores of \textbf{L1}, \textbf{L21}, \textbf{$\Omega_*$-nmAPGD}, and \textbf{$\Omega_*$-SubGD}, and those of PSSRs except for \textbf{TI} overlap, respectively.
    Only \textbf{TI} successfully selected feature interactions in this setting: the F1-score and PSSR of \textbf{TI} increased with $N$, and \textbf{TI} achieved about 80\% of PSSR when $N=200$.
    The F1-scores and PSSRs of \textbf{CS}, \textbf{L21}, \textbf{L1}, \textbf{$\Omega_*$-nmAPGD}, \textbf{$\Omega_*$-SubGD}, and \textbf{FM} did not increase with $N$.
    Although \textbf{SLQR} achieved the better F1-scores than the other methods except for \textbf{TI}, its estimation errors were worst and its PSSRs were zero for all $N$.
    We observed that \textbf{SLQR} could achieve as low estimation errors as \textbf{$\Omega_*$-nmAPGD} and \textbf{$\Omega_*$-SubGD} if we set time budgets more than $10 \times N / 50$ second.
    Namely, \textbf{SLQR} could select feature interactions and learn a good $\bm{W}$ but it was inefficient compared to \textbf{TI}.
    \textbf{$\Omega_*$-nmAPGD} and \textbf{$\Omega_*$-SubGD} achieved lower estimation errors than the existing methods but their F1-scores and PSSRs were as low as those of the existing methods.
    This is because the nmAPGD algorithm with an inexact proximal operator and the SubGD algorithm do not produce a sparse $\bm{P}$ and a sparse $\bm{W}$, and the PSSR and F1-score of a dense $\bm{W}$ are low by definition.
    In contrast, not only \textbf{TI} but also \textbf{CS}, \textbf{L21}, and \textbf{L1} selected features correctly for the feature selection setting datasets (\cref{fig:synthetic_tuned_feature_selection}).
    Note that the plots of F1-scores of \textbf{$\Omega_*$-nmAPGD} and \textbf{$\Omega_*$-SubGD} and those of PSSRs of \textbf{SLQR}, \textbf{$\Omega_*$-nmAPGD}, and \textbf{$\Omega_*$-SubGD} overlap, respectively.
    \textbf{CS} performed the best on all metrics for the feature selection setting datasets and it seems reasonable because \textbf{CS} used the upper bound of $\Omega_*$ and \textbf{CS} produces a row-wise sparse $\bm{P}$ (i.e., \textbf{CS} is more preferable for feature selection than \textbf{TI}).
    
\subsubsection{Sensitivity to Regularization-strength Hyperparameter}
\label{subsubsec:sensitivity}
    \begin{figure*}[t]
        \centering
        \subfloat{
        \includegraphics[width=52mm]{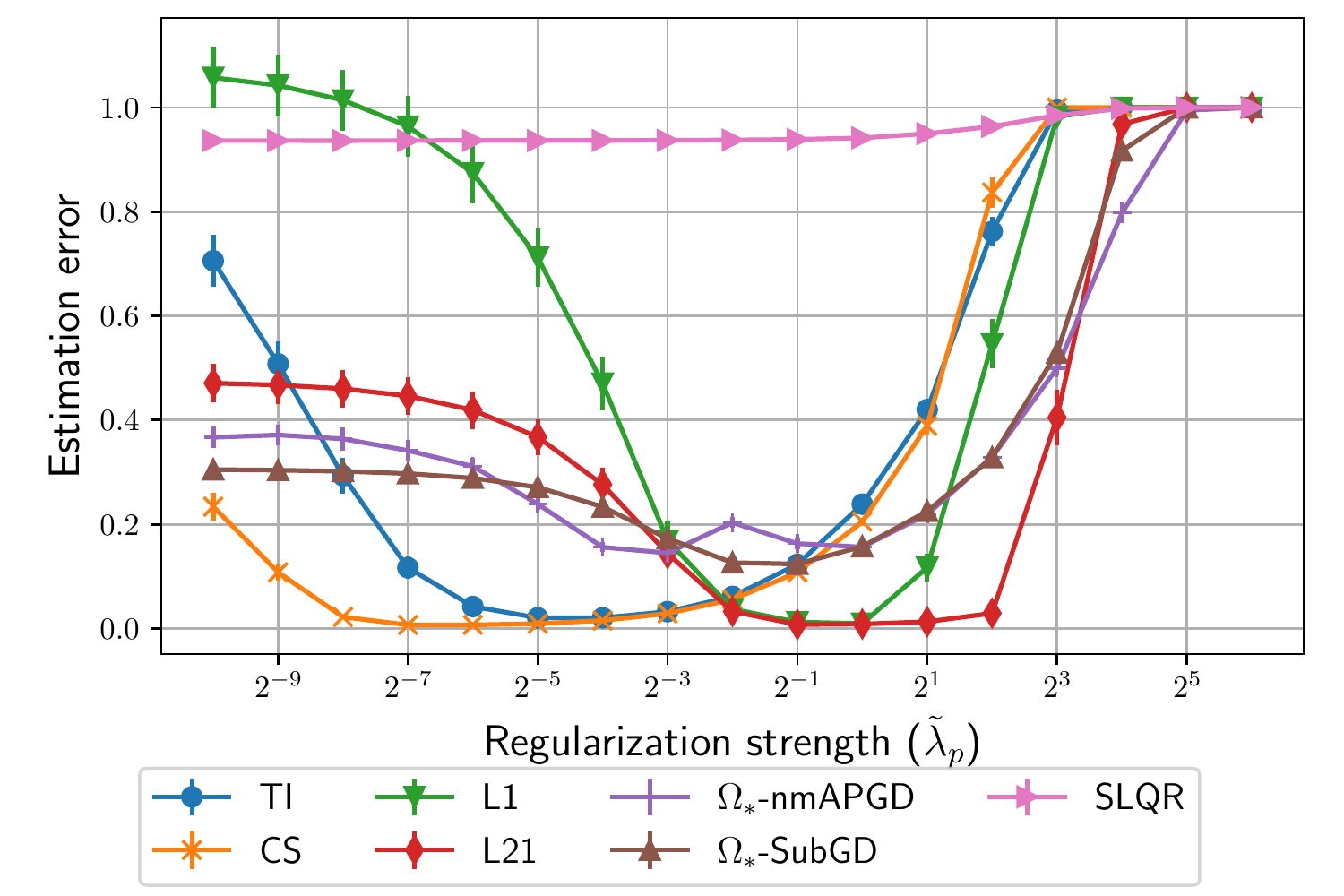} 
        \includegraphics[width=52mm]{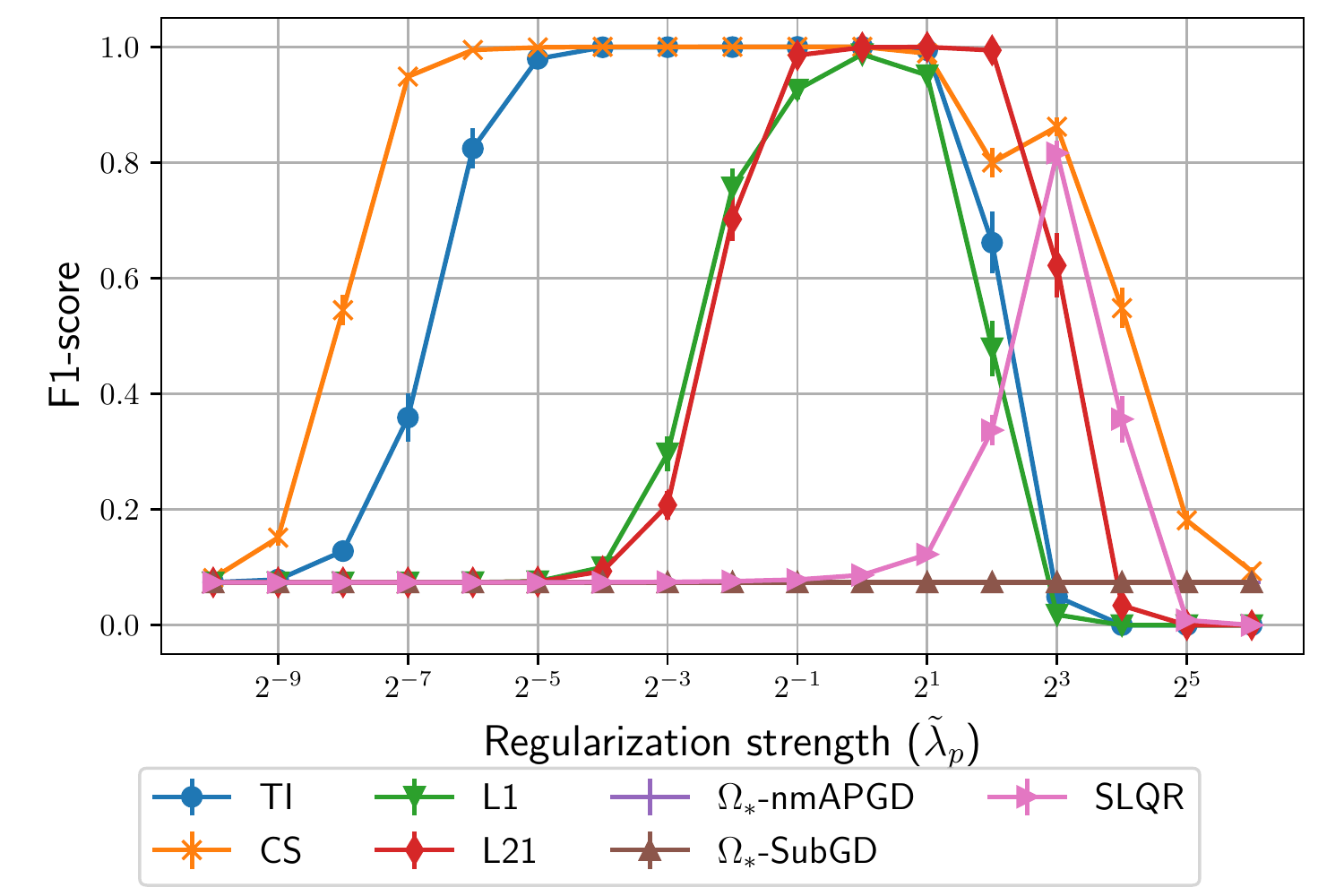} 
        \includegraphics[width=52mm]{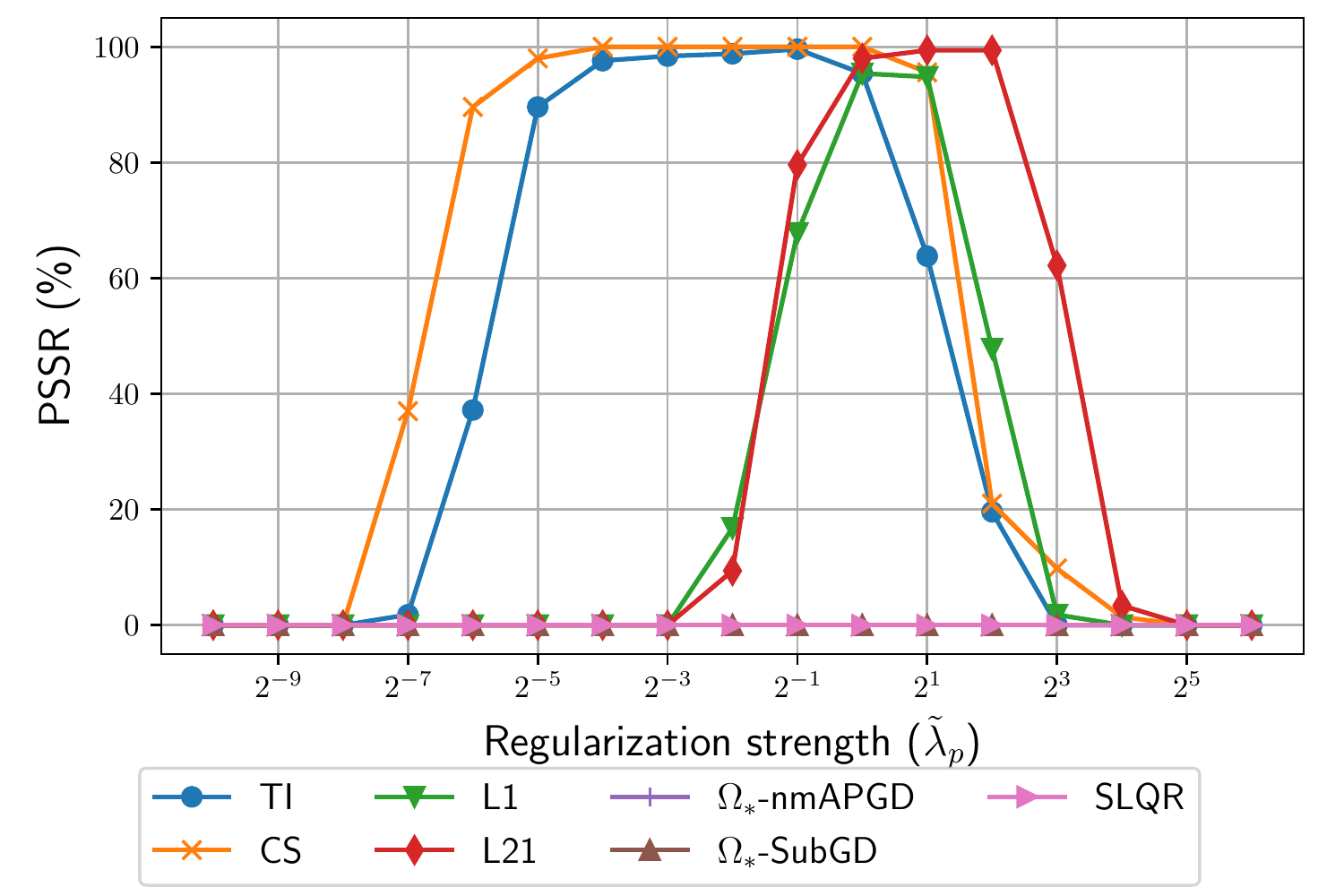}}
        \caption{Sensitivity to regularization-strength hyperparameter $\tilde{\lambda}_p$ for sparse regularization for \textbf{TI}, \textbf{CS}, \textbf{L21}, and \textbf{L1} methods on feature selection setting synthetic datasets with $d_{\mathrm{true}}=20$, $b=1$, $d_{\mathrm{noise}}=80$, and $N=200$. Left graph shows estimation error (lower is better), center graph shows F1-score (higher is better), and right graph shows PSSR (higher is better). The plots of F1-scores of \textbf{$\Omega_*$-nmAPGD} and \textbf{$\Omega_*$-SubGD} and those of PSSRs of \textbf{$\Omega_*$-nmAPGD}, \textbf{$\Omega_*$-SubGD}, and \textbf{SLQR} overlap, respectively.}
        \label{fig:robustness}
    \end{figure*}
    In this experiment, we investigated sensitivity to the regularization-strength hyperparameter for sparse regularization in the existing and proposed sparse FM methods.
    We evaluated and compared the estimation errors, F1-scores, and PSSRs for various $\tilde{\lambda}_p$, which is the regularization strength for sparse regularization.
    In this experiment, we used $50$ feature selection datasets with $N=200$ since the results on the feature interaction selection datasets were bad except \textbf{TI} even with tuned $\tilde{\lambda}_p$, as discussed in~\cref{subsubsec:synthetic_tuned}.
    We set $\tilde{\lambda}_p$ and $\tilde{\lambda}$ to $2^{-10}, 2^{-9}, \ldots, 2^{6}$ and set $\lambda_p=\lambda_{\mathrm{tr}}=0.1$.
    The other settings (rank-hyper parameter, initialization, and stopping criterion) were the same as those described in~\cref{subsubsec:synthetic_tuned}.
    Again, we ran the experiment $10$ times with different initial random seeds for FM-based methods.
    
    As shown in~\cref{fig:robustness}, although the regions of an adequate $\tilde{\lambda}_p$ differed among methods, that of \textbf{CS} was wider than those of the other methods for all metrics.
    The regions of an adequate $\tilde{\lambda}_p$ of \textbf{TI} for the F1-score and PSSR were also wider than those of other methods.
    Thus, our proposed methods are less sensitive to $\tilde{\lambda}_p$ than the other methods.
    This is important in real-world applications, especially large-scale applications that require high computational costs for hyperparameter tuning.
 
\subsubsection{Efficiency and Scalability}
\label{subsubsec:efficiency}
    \begin{figure*}[t]
        \centering
        \subfloat[Feature interaction selection setting: $d_{\mathrm{true}}=80$, $b=8$ and $d_{\mathrm{noise}}=20$ with $\tilde{\lambda}_p=0.1$ (left) and $1.0$ (right).]{
        \includegraphics[width=80mm]{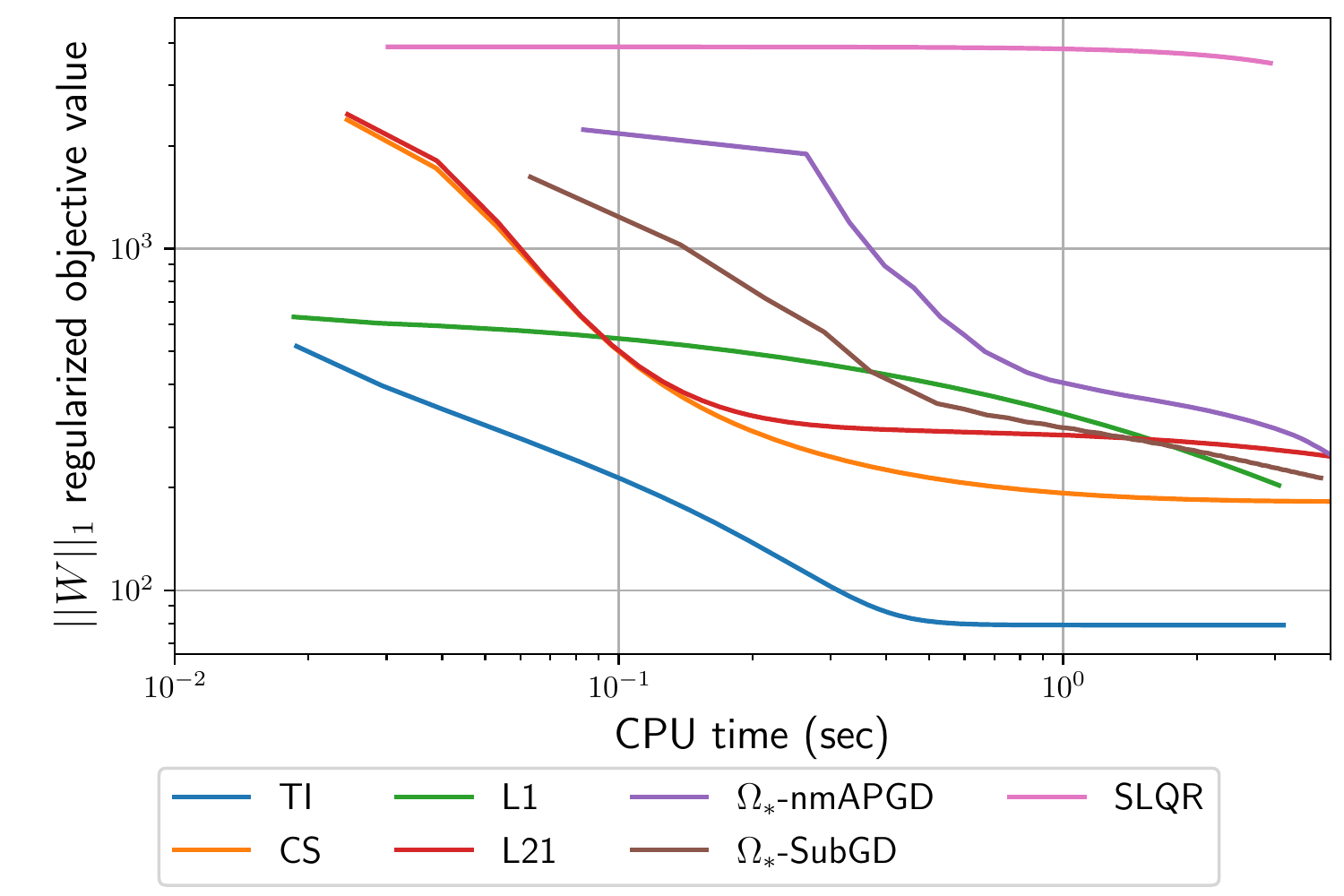}
        \includegraphics[width=80mm]{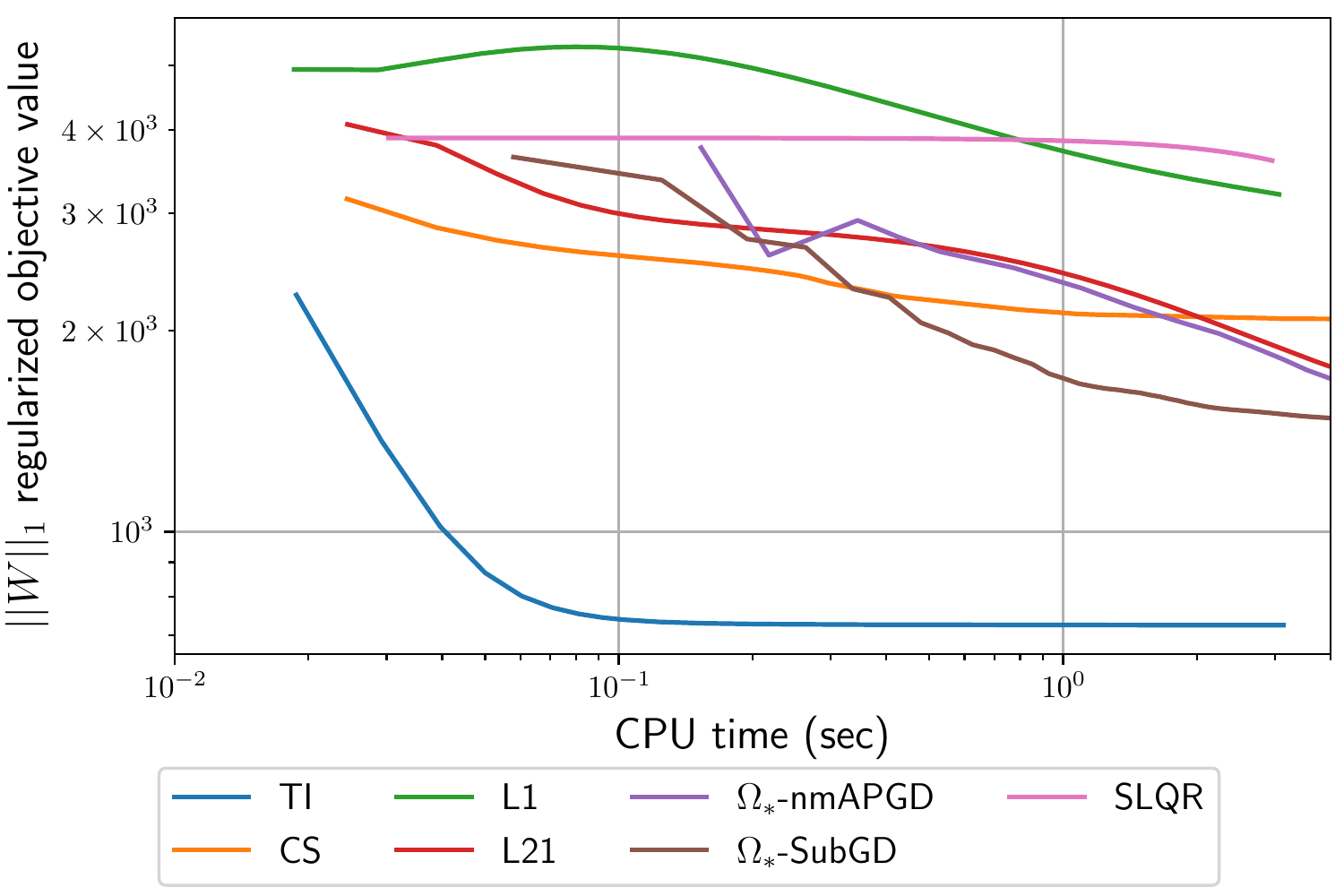}
        \label{fig:exact_feature_interaction_selection}
        }\\
        \subfloat[Feature selection setting: $d_{\mathrm{true}}=20$, $b=1$ and $d_{\mathrm{noise}}=80$ with $\tilde{\lambda}_p=0.1$ (left) and $1.0$ (right).]{
        \includegraphics[width=80mm]{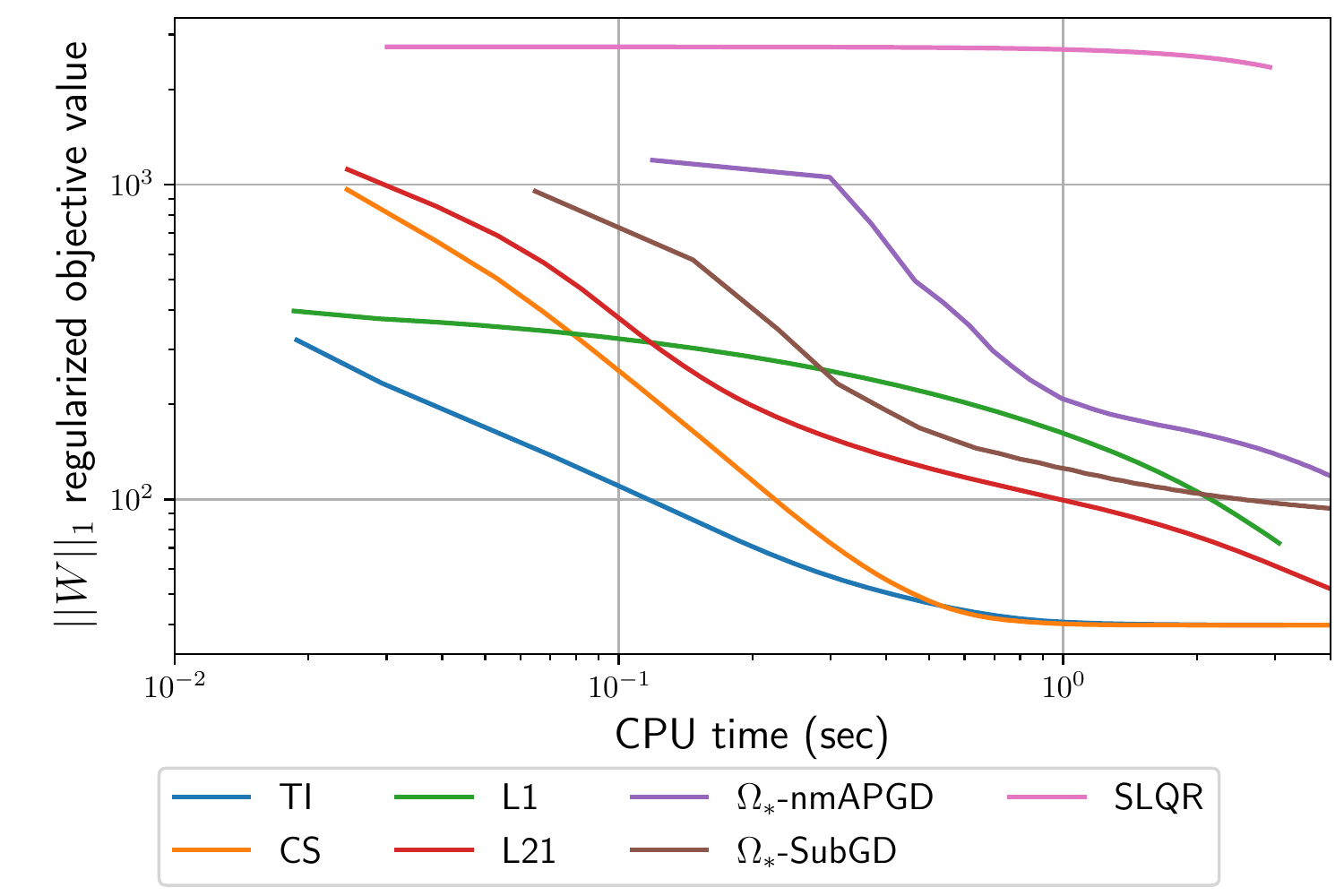}
        \includegraphics[width=80mm]{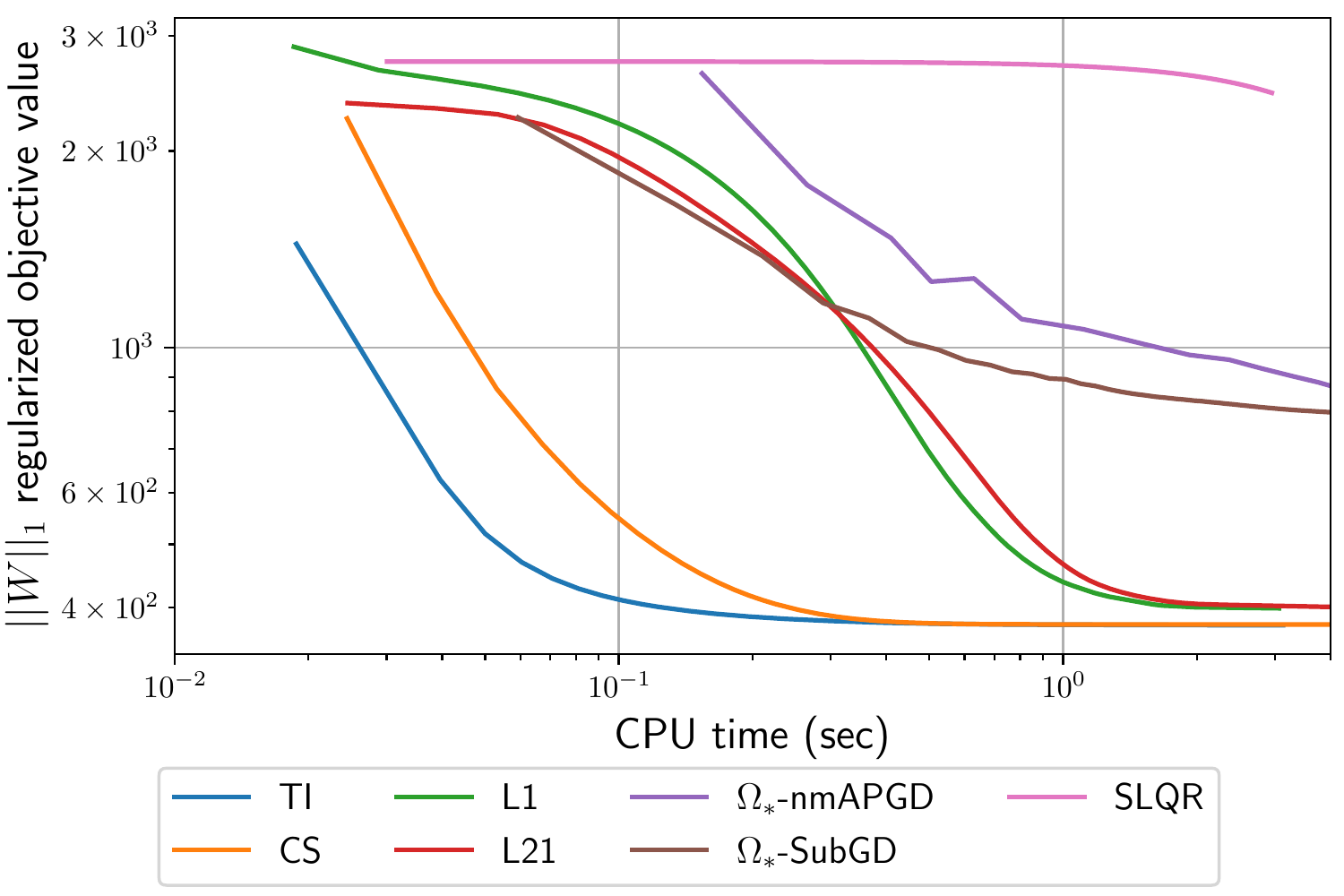}
        \label{fig:exact_feature_selection}
        }
        \caption{Trajectories of $\norm{\bm{W}}_1$ regularized objective value for \textbf{TI}, \textbf{CS}, \textbf{L21}, \textbf{L1}, \textbf{$\Omega_*$-nmAPGD}, \textbf{$\Omega_*$-SubGD}, and \textbf{SLQR} methods on (a) feature interaction selection setting synthetic datasets and (b) feature selection setting synthetic datasets with $N=200$.}
        \label{fig:exact}
    \end{figure*}

    \begin{figure*}
        \centering
        \includegraphics[width=80mm]{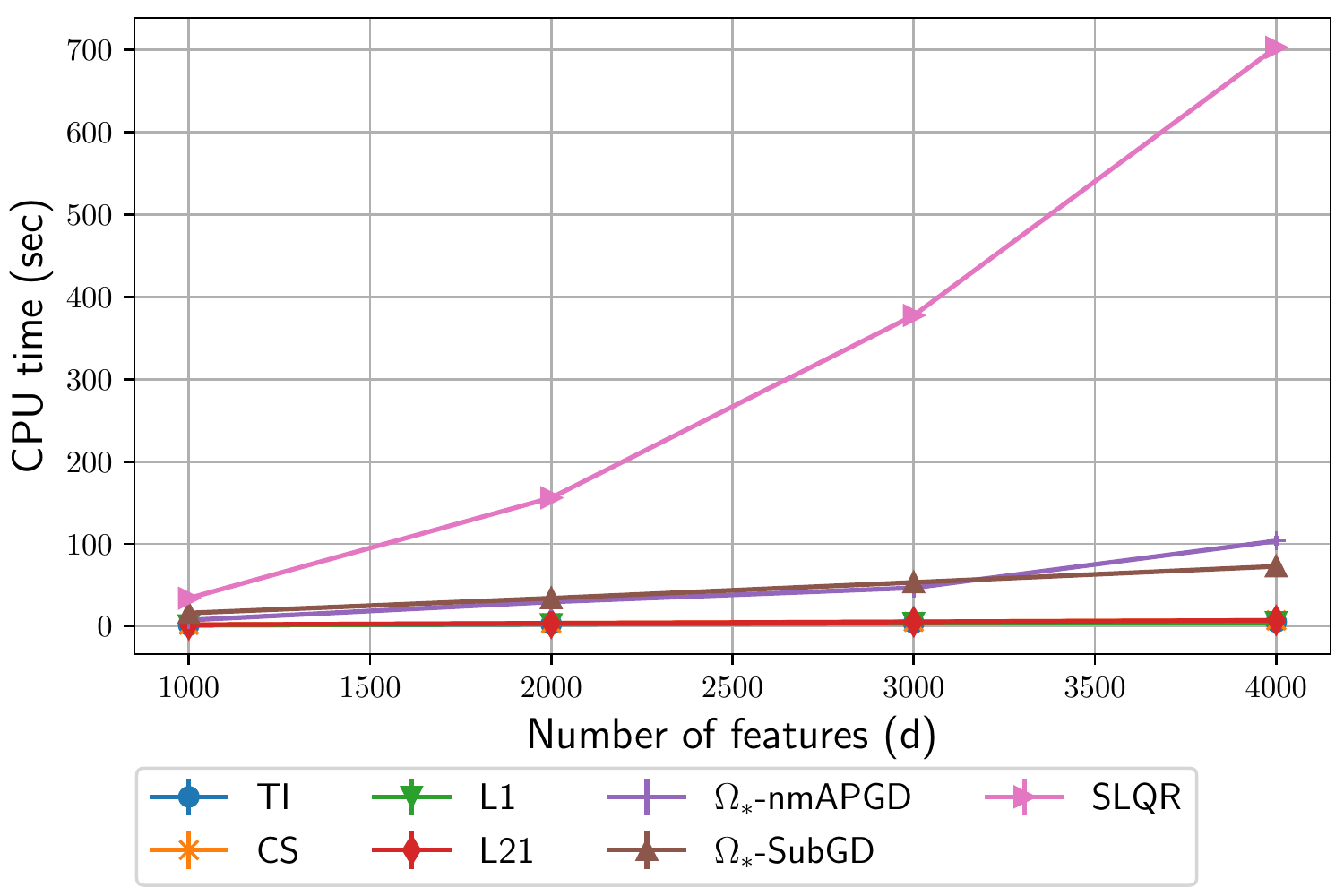}
        \includegraphics[width=80mm]{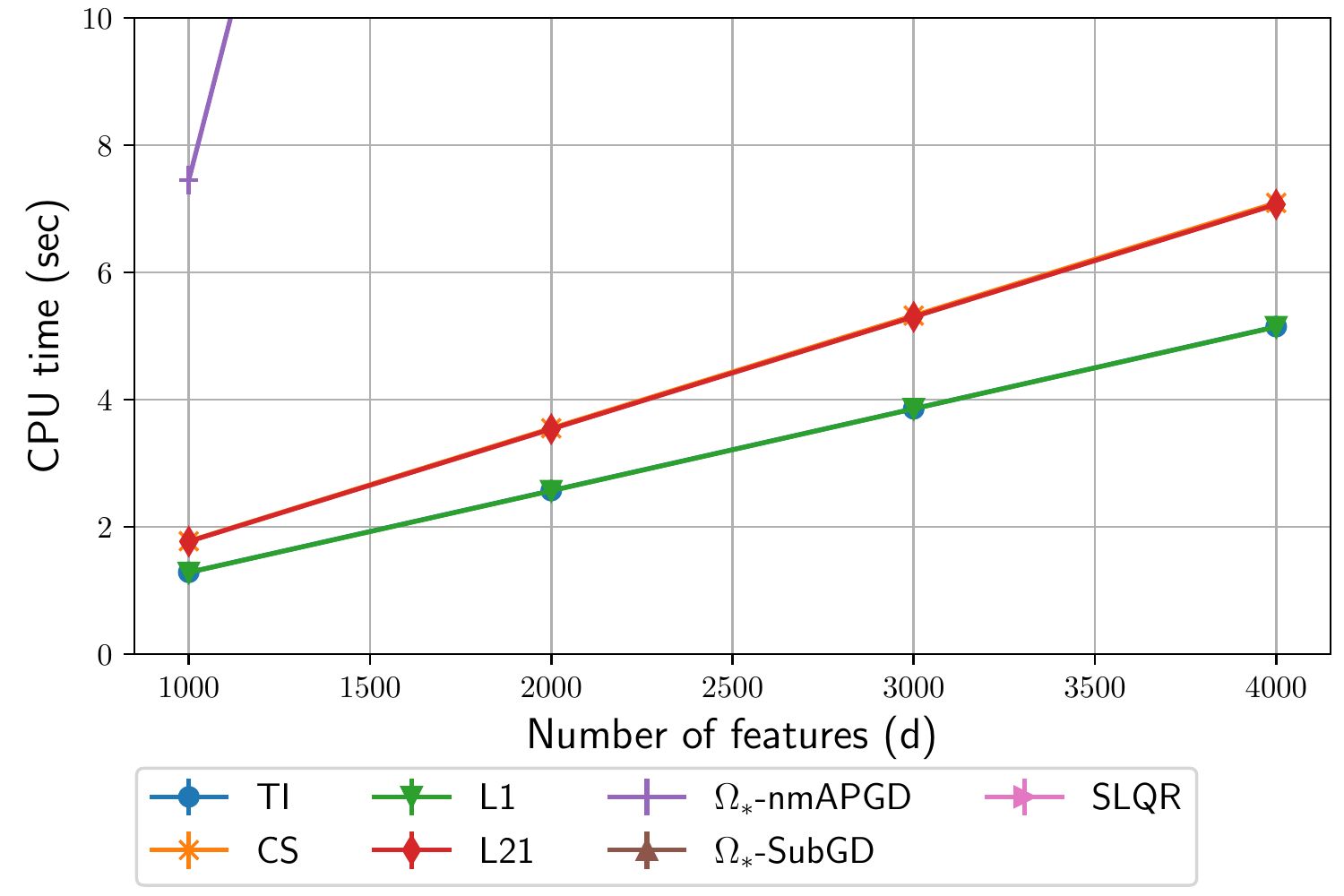}
        \label{fig:runtime_varying_d}
        \caption{Running times for one epoch in \textbf{TI}, \textbf{CS}, \textbf{L21}, \textbf{L1}, \textbf{$\Omega_*$-nmAPGD}, \textbf{$\Omega_*$-SubGD}, and \textbf{SLQR} methods on synthetic datasets.
        The right figure is a zoomed-in-version of the left figure.
        The plots of \textbf{TI} and \textbf{L1} overlap and those of \textbf{CS} and \textbf{L21} overlap.}
         \label{fig:running_time_varying_d}
    \end{figure*}
    
    In the third experiment, we evaluated the efficiency of the proposed and existing methods.
    We first compared convergences of the $\norm{\bm{W}}_1$ (it is equivalent to $2\Omega_*(\bm{P}) + \norm{\bm{P}}_2^2$ in FMs as shown in~\cref{subsec:exactl1}) regularized objective function value among \textbf{TI}, \textbf{CS}, \textbf{L21}, \textbf{L1}, \textbf{$\Omega_*$-nmAPGD}, \textbf{$\Omega_*$-SubGD}, and \textbf{SLQR} methods.
    We tracked the value in the optimization processes using $50$ feature interaction selection setting datasets and feature selection setting datasets with $N=200$.
    Since the appropriate $\tilde{\lambda}_p$ differed among methods, as mentioned in~\cref{subsubsec:sensitivity}, we ran the experiment with $\tilde{\lambda}_p=0.1$ and $1.0$.
    For both settings, we set $\lambda_p=\lambda_{\mathrm{tr}}=0.1$ and the other settings were the same as those described in~\cref{subsubsec:synthetic_tuned}.
    Note that we show the results of this experiment on some real-world datasets in~\Cref{subsec:efficiency_real_world}.
     
    As shown in~\cref{fig:exact_feature_interaction_selection}, the proposed \textbf{TI} method achieved the lowest $\norm{\bm{W}}_1$ regularized objective value on the feature interaction selection datasets.
    The difference was remarkable for $\tilde{\lambda}_p = 1.0$.
    As shown in~\cref{fig:exact_feature_selection}), the proposed \textbf{TI} and \textbf{CS} methods achieved lower $\norm{\bm{W}}_1$ regularized objective values on the feature selection datasets than the other methods for all parameter settings.
    Moreover, the objective values converged faster with \textbf{TI} and \textbf{CS} than with the other methods all parameter settings.
    Thus, our proposed sparse FMs are more attractive alternatives to $\Omega_*$-sparse FMs than the existing sparse FMs.

    We next compared the scalability of the existing and proposed methods w.r.t the number of features $d$.
    We created synthetic datasets with varying $d$ and compared the running time for one epoch among \textbf{TI}, \textbf{CS}, \textbf{L21}, \textbf{L1}, \textbf{$\Omega_*$-nmAPGD}, \textbf{$\Omega_*$-SubGD}, and \textbf{SLQR} methods.
    We changed $d_{\mathrm{true}}$ as $1000$, $2000$, $3000$, and $4000$.
    We set $b= d_{\mathrm{true}}/100$, $d_{\mathrm{noise}}=0$, and $N=2500$.
    We created ten datasets with different random seeds for all $d$ and report the average running times.
    The other settings were the same as those described in~\cref{subsubsec:synthetic_tuned}.
    
    As shown in~\cref{fig:running_time_varying_d}, the running time of \textbf{TI}, \textbf{CS}, \textbf{L1} and \textbf{L21} linearly increased w.r.t $d$.
    On the other hand, that of \textbf{$\Omega_*$-nmAPGD} and \textbf{$\Omega_*$-SubGD} increased quadratically, and that of \textbf{SLQR} increased cubically w.r.t $d$.
    When $d=4000$, \textbf{SLQR} ran more than 100 times slower than \textbf{TI}, \textbf{CS}, \textbf{L1}, and \textbf{L21}.
    Thus, the proposed \textbf{TI} and \textbf{CS} are better than \textbf{$\Omega_*$-nmAPGD}, \textbf{$\Omega_*$-SubGD}, and \textbf{SLQR} for a high-dimensional case.
    
\subsection{Real-world Datasets}
\label{subsec:real}
    Next, we used real-world datasets to demonstrate the usefulness of the proposed methods.

    \paragraph{Settings and Datasets.} We compared existing and proposed sparse FMs on an interpretability-constraint setting: the number of interactions in sparse FMs were constrained to be (about) $1,000$.
    We used one regression dataset, the MovieLens 100K (ML100K)~\citep{harper2016movielens} dataset, and three binary classification datasets, a9a, RCV1~\citep{chang2011libsvm}, and Flight Delay (FD)~\citep{ke2017lightgbm}.
    \cref{tab:datasets} summarizes the details of these datasets.
     \begin{table}[t]
        \centering
        \caption{Datasets used to demonstrate usefulness of proposed methods.}
        \begin{tabular}{c|c|cccc}
            Dataset & Task & $d$ & $N_{\mathrm{train}}$ & $N_{\mathrm{valid}}$ &$N_{\mathrm{test}}$ \\\hline
            ML100K~\citep{harper2016movielens} & Regression & 2,703 & 64,000 & 16,000 & 20,000\\
            a9a~\citep{chang2011libsvm} & Classification & 123 & 26,048 & 6,513 & 16,281\\
            RCV1~\citep{chang2011libsvm} & Classification & 47,236 & 16,193 & 4,049 & 677,399\\
            FD~\citep{ke2017lightgbm} & Classification & 696 & 16,000 & 4,000 & 80,000\\
        \end{tabular}
        \label{tab:datasets}
    \end{table}
    With the ML100K dataset, which is used for movie recommendation, we considered a regression problem: predicting the score given to a movie by a user.
    Possible scores were $1, 2, \ldots, 5$.
    We created feature vectors by following the method of~\citet{blondel2016higher}.
    We divided the $100,000$ user-item scores in the dataset into sets of $64,000$, $16,000$, and $20,000$ for training, validation, and testing, respectively.
    For the a9a and RCV1 datasets, we used feature vectors and targets that are available from the LIBSVM~\citep{chang2011libsvm} datasets repository.~\footnote{https://www.csie.ntu.edu.tw/~cjlin/libsvmtools/datasets/}
    Both a9a and RCV1 have already been divided into training and testing datasets; we used $20\%$ of the training dataset as the validation dataset and the remaining $80\%$ as the training dataset in this experiment.
    For the FD dataset, we considered the classification task to predict whether a flight will be delayed by more than 15 minutes.
    We used the scripts provided in \url{https://github.com/szilard/benchm-ml}: ran \texttt{2-gendata.txt} and randomly sampled train, valid, and test datasets from \texttt{train-1m.csv} and \texttt{test.csv}.
    Then, each instance had eight attributes and we encoded them to one-hot features except \texttt{DepTime} and \texttt{Distance}.
    \texttt{DepTime} represents	an actual departure time and consists of four integers, HHMM.
    We split such \texttt{DepTime} values to HH and MM, and encoded them to one-hot features (MM values were encoded to six-dimensional one-hot features based on their tens digit).
    For \texttt{Distance} values, we regarded them as numerical values and used their logarithm values.
     
    \paragraph{Evaluation Metrics.} As metrics, we used the root mean squared error (RMSE) for the ML100K dataset and the area under the receiver operating characteristic curve (ROC-AUC) for the a9a, RCV1, and FD datasets. Lower is better for the RMSE, and higher is better for the ROC-AUC.
    
    \paragraph{Hyperparameter Settings.} As in the experiments using synthetic datasets, we compared the \textbf{TI}, \textbf{CS}, \textbf{L21}, \textbf{L1}, and \textbf{FM} methods.
    We used the squared error as loss function $\ell$.
    We set rank-hyper parameter $k$ to $30$.
    We used the linear term $\inner{\bm{w}}{\bm{x}}$ and introduced a bias (intercept) term.
    We initialized each element in $\bm{P}$ by using $\mathcal{N}(0, 0.01^2)$ as in the experiments using synthetic datasets.
    We initialized $\bm{w}$ to $\bm{0}$ and the bias term to $0$.
    We ran the experiment five times with different random seeds and calculated the average values of the evaluation metrics.
    For the \textbf{FM} method, we chose $\lambda_w$ and $\lambda_p$ from $0.5\times10^{-7}, 0.5\times10^{-6}, \ldots, 0.5\times10^{-2}$. 
    For the \textbf{TI}, \textbf{CS}, \textbf{L21}, and \textbf{L1} methods, we set $\lambda_w=\lambda_w^{\mathrm{FM}}$ and $\lambda_p=\lambda_p^{\mathrm{FM}}/10$, where $\lambda_w^{\mathrm{FM}}$ and $\lambda_p^{\mathrm{FM}}$ are the tuned $\lambda_w$ and $\lambda_p$ in the \textbf{FM} method.
    Since sparse FMs have additional regularizers, we set $\lambda_p$ to $\lambda_p^{\mathrm{FM}}/10$, not $\lambda_p^{\mathrm{FM}}$.
    As described above, because we constrained the number of used interactions in sparse FMs to be about $1,000$, the method for tuning $\tilde{\lambda}_p$ was complicated.
    We searched for the appropriate $\tilde{\lambda}_p$ by binary search since the number of used interactions (i.e., number of non-zero elements among the strictly upper triangular elements in $\bm{P}\bm{P}^\top$) tended to be monotonically non-decreasing in $\tilde{\lambda}_p$.
    For each sparse FM, the initial range (i.e., upper bound and lower bound) of the binary search was chosen from $10^{-7}, 10^{-6}, \ldots, 10^{-2}$.
    After the initial range was chosen, we searched for the appropriate $\tilde{\lambda}_p$ by binary search.
    Since it was hard to achieve the number of used interactions to be exactly $1,000$, we accepted the models with the number of used interactions to be in $[990, 1,035]$.
    The reason why we set the acceptable range to be $[990, 1,035]$ is that \textbf{CS} and \textbf{L21} achieve only feature selection, and $990$ and $1,035$ are the nearest binomial coefficients to $1,000$: $990=\binom{45}{2}$ and $1,035=\binom{46}{2}$.
    Moreover, if the gap between an upper and a lower bound in binary search was lower than $10^{-12}$, we gave up tuning $\tilde{\lambda}_p$ for such models and set their scores to be N/A.
    Although \textbf{FM} did not select feature interactions, we showed results of it for comparison.
    We also show the results of \textbf{TI}, \textbf{CS}, \textbf{L1} and \textbf{L21} with tuned $\tilde{\lambda}_p \in \{10^{-7}, 10^{-6}, \ldots, 10^{-2}\}$ (i.e., the best results before doing binary search) for comparison although the numbers of used interactions in them were not close to $1,000$.

    \begin{table}[t]
        \centering
        \caption{Comparison of test RMSE for ML100K dataset and test ROC-AUC for a9a, RCV1, and FD datasets. (a) Results on interpretability constraint setting. (b) Results on no interpretability constraint setting ($\tilde{\lambda}_p \in \{10^{-7}, 10^{-6}, \ldots, 10^{-2}\}$). We report not only RMSE and ROC-AUC (upper value in each cell) but also the number of used interactions (lower value in each cell). Lower is better for RMSE, and higher is better for ROC-AUC.}
        \subfloat[Interpretability constraint setting.]{
            \begin{tabular}{c|c|c|c|c}
                 Method & ML100K (RMSE) & a9a (ROC-AUC) & RCV1 (ROC-AUC) & FD (ROC-AUC) \\ \hline
                 \multirow{2}{*}{\textbf{TI}} & \textbf{0.93018} & \textbf{0.90301} & \multirow{2}{*}{N/A} & 0.71373 \\
                 & 991.4 & 1,001.6 & & 1,004.4\\ \hline
                 \multirow{2}{*}{\textbf{CS}} & 0.93127 & 0.90259 & \textbf{0.99197} & \textbf{0.71437} \\
                 & 1,026.0 & 1,035.0 & 1,035.0 & 1,035.0\\ \hline
                 \multirow{2}{*}{\textbf{L1}} & \multirow{2}{*}{N/A} & 0.90197 & \multirow{2}{*}{N/A} & \multirow{2}{*}{N/A} \\
                 & & 1,026.0 &  & \\ \hline
                 \multirow{2}{*}{\textbf{L21}} & 0.93302 & 0.90259 & \multirow{2}{*}{N/A} & \multirow{2}{*}{N/A} \\
                 & 1,030.2 & 990.0 &  &
            \end{tabular}
            \label{tab:real_constraint}
        }
        
        \subfloat[No interpretability constraint setting ($\tilde{\lambda}_p$ was tuned from $\{10^{-7}, 10^{-6}, \ldots, 10^{-2}\}$).]{
        \begin{tabular}{c|c|c|c|c}
             Method & ML100K (RMSE) & a9a (ROC-AUC) & RCV1 (ROC-AUC) & FD (ROC-AUC) \\ \hline
             \multirow{2}{*}{\textbf{TI}} & 0.91848 & \textbf{0.90288} & \textbf{0.99258} & \textbf{0.71378} \\
             & 186,600.4 & 1,408.2 & 663,793.8 & 2,471.4\\ \hline
             \multirow{2}{*}{\textbf{CS}} & \textbf{0.91610} & 0.90250 & 0.99243 & 0.71057 \\
             & 7264.0 & 903.0 & 170,820.0 & 0.0\\ \hline
             \multirow{2}{*}{\textbf{L1}} & 0.92612 & 0.90193 & 0.99195 & 0.71057 \\
             & 1,616,087.2 & 946.0 & 0.0 & 0.0 \\ \hline
             \multirow{2}{*}{\textbf{L21}} & 0.91975 & 0.90269 & 0.99195 & 0.71057 \\
             & 1,775,679.2 & 1,081.0 & 0.0 & 0.0 \\ \hline
             \multirow{2}{*}{\textbf{FM}} & 0.91734 & 0.90280 & 0.99195 & 0.71334 \\
             & 3,483,480.0 & 7,503.0 & 0.0 & 206,403.0
             \label{tab:real_no_constraint}
        \end{tabular}}
        \label{tab:real_results}
    \end{table}
    
    \paragraph{Results.} As shown in~\cref{tab:real_constraint}, although the differences were not remarkable, our proposed methods achieved the best performance for each dataset.
    For ML100K and a9a datasets, \textbf{TI} achieved the best performance.
    This results match the experimental results in~\cref{subsubsec:synthetic_tuned}.
    Among \textbf{TI}, \textbf{CS}, \textbf{L21}, and \textbf{L1}, only \textbf{TI} can perform feature interaction selection.
    For the RCV1 dataset, only \textbf{CS} succeeded in learning models using approximately $1,000$ feature interactions.
    For the FD dataset, \textbf{TI} and \textbf{CS} succeeded in learning such models but \textbf{L21} and \textbf{L1} did not.
    This result matches the experimental results in~\cref{subsubsec:synthetic_tuned} and~\cref{subsubsec:sensitivity}.
    \textbf{TI} and \textbf{CS}, especially \textbf{CS}, are less sensitive to regularization-strength hyperparameter and can perform better feature selection than \textbf{L21} and \textbf{L1}.
    Moreover, as shown in~\cref{tab:real_no_constraint}, \textbf{TI} achieved the best performance for a9a, RCV1, and FD datasets and \textbf{CS} achieved the best performance for ML100K dataset on no interpretability constraint setting.
    They used not all but partial feature interactions except \textbf{CS} for FD dataset.
    On the other hand, \textbf{FM} used all feature interactions for ML100K, a9a, and FD datasets and used no feature interactions for RCV1 dataset (in our training-validation-test splitting, the training ML100K dataset contained only 2,640 features and the training FD dataset contained only 643 features).
    \textbf{L1} and \textbf{L21} also used no feature interactions for RCV1 and FD datasets.
    From these results, we conclude that our proposed methods could select better (important) features and feature interactions in terms of the prediction performance.

\section{Conclusion}
    \label{sec:conclusion}
    In this paper, we have presented new sparse regularizers for feature interaction selection and feature selection in FMs, the TI regulazier and the CS regularizer, respectively, as well as efficient proximal optimization methods for these proposed methods.
    Our basic idea is the use of $\ell_1$ regularizer for feature interaction weight matrix computed from the parameter matrix of FMs.
    This regularization seems appropriate for feature interaction selection in FMs because it is reported as one of the most promising sparse regularizers and selecting feature interactions necessarily means making $\bm{W}$ sparse.
    Unfortunately, the associated objective function is hard to optimize w.r.t the parameter matrix of FMs.
    To overcome this difficulty, we have proposed the use of squares of sparsity-inducing (quasi-)norms as an upper bound of the $\ell_1$ norm for $\bm{W}$, and we have presented such regularizers concretely, the TI regularizer and the CS regularizer.
    The TI enables feature interaction selection without feature selection, and the CS can select better (more important) features.
    Fortunately, the associated objective functions are now easy to optimize, because TI-sparse FMs and CS-sparse FMs can be optimized at the same computational cost as canonical FMs. 
    We have demonstrated the effectiveness of the proposed methods on synthetic and real-world datasets.

    As future work, we would like to (i) develop more efficient PSGD-based algorithms for TI/CS-sparse FMs and (ii) investigate the theoretical properties of the proposed methods.

\section*{Acknowledgements}
    This work was partially supported by JSPS KAKENHI Grant Number JP20J13620 and by the Global Station for Big Data and Cybersecurity, a project of the Global Institution for Collaborative Research and Education at Hokkaido University.

\newpage
\bibliographystyle{plainnat}

\newpage
\appendix
\section{Proofs}
\paragraph{Additional Notation.}
    For given two matrices $\bm{P}, \bm{Q} \in \real^{n \times n}$, $\bm{P} =_{>} \bm{Q}$ means $p_{i, j}=q_{i, j}$ for all $0 < i < j \le n$, that is, $\bm{P}$ and $\bm{Q}$ have same values in their strictly upper triangular elements.
    We use $\bm{e}_{j}^{d} \in \{0, 1\}^d$ for the $d$-dimensional standard basis vector whose $j$-th element is one and the others are zero.
    For given a scalar $\alpha \in \real$ and a subset $C$ of $\real^d$, we define $\alpha C \coloneqq \{\alpha \bm{c} : \bm{c} \in C\}$.
    For given a $d$-dimensional vector $\bm{x} \in \real^d $ and a subset $C$ of $\real^d$, we use $\bm{x} + C$ as $\{\bm{x}+\bm{c}: \bm{c} \in C\}$.

\subsection{Subdifferentials of Powers of Norms}
    \label{subsec:subdifferential}
    For~\cref{thm:solution_prox_higher_order_ti} and~\cref{thm:solution_prox_higher_order_cs} (and of course~\cref{thm:solution_prox_cs}), we first derive the subdifferentials of powers of norms.
    \begin{lemm}
        \label{lemm:subdifferential_pow_norm}
        Let $\bm{x} \in \real^d$ and $\norm{\cdot}$ be a norm on $\real^d$.
        Then, for all $m \ge 1$, the subdifferential of $\norm{\cdot}^m$ at $\bm{x}$, $\partial \norm{\bm{x}}^m$ is
        \begin{align}
            \partial \norm{\bm{x}}^m = m \norm{\bm{x}}^{m-1} \partial \norm{\bm{x}} = m \norm{\bm{x}}^{m-1} \{\bm{z} \in \real^d: \inner{\bm{z}}{\bm{x}} = \norm{\bm{x}}, \norm{\bm{z}}_* \le 1\},
            \label{eq:subdifferential_pow_norm}
        \end{align}
        where $\norm{\cdot}_{*}: \real^d \to \real$ is the dual norm of $\norm{\cdot}: \norm{\bm{x}}_* = \sup_{\bm{z} \in \real^d, \norm{\bm{z}} \le 1} \inner{\bm{z}}{\bm{x}}$.
    \end{lemm}
    
    \begin{proof}
        (i) We first show that $\partial \norm{\bm{x}}^m \supseteq  m \norm{\bm{x}}^{m-1} \partial \norm{\bm{x}}$.
        For all $\bm{x}, \bm{y} \in \real^d$ and $\bm{z} \in \partial \norm{\bm{x}}$, from the definition of $\partial\norm{\bm{x}}$ and the H\"older's inequality ($\inner{\bm{z}}{\bm{y}} \le \norm{\bm{z}}_*\norm{\bm{y}}$), we have
        \begin{align}
            &\norm{\bm{x}}^m + \inner{m\norm{\bm{x}}^{m-1}\bm{z}}{\bm{y}-\bm{x}} = \norm{\bm{x}}^m + m\norm{\bm{x}}^{m-1}(\inner{\bm{z}}{\bm{y}} - \inner{\bm{z}}{\bm{x}}) \\
            &= \norm{\bm{x}}^m + m\norm{\bm{x}}^{m-1}(\inner{\bm{z}}{\bm{y}} - \norm{\bm{x}}) \le \norm{\bm{x}}^m + m\norm{\bm{x}}^{m-1}\norm{\bm{z}}_*\norm{\bm{y}} - m\norm{\bm{x}}^m\\
            &\le (1-m)\norm{\bm{x}}^m + m \norm{\bm{x}}^{m-1}\norm{\bm{y}}.
        \end{align}
        Moreover, from the convexity of $x^m$ on $x \ge 0$ ($m \ge 1$) and the non-negativity of the norm, we have
        \begin{align}
            (\norm{\bm{y}})^m \ge (\norm{\bm{x}})^m + m (\norm{\bm{x}})^{m-1}(\norm{\bm{y}} - \norm{\bm{x}}) = (1-m)\norm{\bm{x}}^m + m \norm{\bm{x}}^{m-1}\norm{\bm{y}}.
        \end{align}
        Combining them, we obtain
        \begin{align}
            \norm{\bm{y}}^m \ge \norm{\bm{x}}^m + \inner{m\norm{\bm{x}}^{m-1}\bm{z}}{\bm{y}-\bm{x}}.
        \end{align}
        The above inequality means $(m \norm{\bm{x}}^{m-1}\bm{z}) \in \partial \norm{\bm{x}}^m$.
        Therefore $\partial \norm{\bm{x}}^m \supseteq  m \norm{\bm{x}}^{m-1} \partial \norm{\bm{x}}$.
        
        (ii) We next prove that $\partial \norm{\bm{x}}^m \subseteq  m \norm{\bm{x}}^{m-1} \partial \norm{\bm{x}}$ by contradiction.
        \begin{enumerate}
            \item $\inner{\bm{z}}{\bm{x}} = m \norm{\bm{x}}^{m-1}\norm{\bm{x}}=m\norm{\bm{x}}^m$ for all $\bm{z} \in \partial \norm{\bm{x}}^m$.
            It clearly holds when $\bm{x}=\bm{0}$, so we consider the case where $\bm{x}\neq \bm{0}$.
            Assume that $\inner{\bm{z}}{\bm{x}} = m \norm{\bm{x}}^{m} + \varepsilon$, where $\varepsilon > 0$, and we show that then there exists $\bm{y} \in \real^d$ such that $\norm{\bm{y}}^m < \norm{\bm{x}}^m + \inner{\bm{z}}{\bm{y}-\bm{x}}$ (this contradicts to the definition of $\partial \norm{\bm{x}}^m$: if $\bm{z} \in \partial \norm{\bm{x}}^m$, then $\norm{\bm{y}}^m \ge \norm{\bm{x}}^m + \inner{\bm{z}}{\bm{y}-\bm{x}}$ for all $\bm{y} \in \real^d$).
            Let $\bm{y} = c \bm{x}$, $c > 0$, and define $h:\real \to \real$ as
            \begin{align}
                h(c) \coloneqq \norm{\bm{x}}^m + \inner{\bm{z}}{\bm{y}-\bm{x}} - \norm{\bm{y}}^m = (1-c^m) \norm{\bm{x}}^m - (1-c)(m \norm{\bm{x}}^m + \varepsilon).
            \end{align}
            $h(1)=0$, $h'(c) = - mc^{m-1}\norm{\bm{x}}^m + (m\norm{\bm{x}}^m+\varepsilon)$, and $h'(c) > 0$ on $c^{m-1} \in (0, 1 + \varepsilon / (m \norm{\bm{x}}^m))$, so $h(c) > 0$ on $c \in (1, \{1+ \varepsilon / (m \norm{\bm{x}}^m)\}^{1/(m-1)})$.
            This means that there exists $\bm{y} \in \real^d$ such that $\norm{\bm{y}}^m < \norm{\bm{x}}^m + \inner{\bm{z}}{\bm{y}-\bm{x}}$, and this contradicts to the definition of the subdifferential.
            The contradiction can be derived under the assumption $\inner{\bm{z}}{\bm{x}} = m \norm{\bm{x}}^{m} - \varepsilon$ in a similar manner.
            Thus, $\inner{\bm{z}}{\bm{x}} = m\norm{\bm{x}}^m$ for all $\bm{z} \in \partial \norm{\bm{x}}^m$.
            
            \item $\norm{\bm{z}}_* \le m \norm{\bm{x}}^{m-1}$ for all $\bm{z} \in \partial \norm{\bm{x}}^m$.
            Assume that $\norm{\bm{z}}_* = m \norm{\bm{x}}^{m-1} + \varepsilon$, $\varepsilon > 0$.
            From the H\"older's inequality, $\inner{\bm{z}}{\bm{y}} \le \norm{\bm{z}}_* \norm{\bm{y}}$ for all $\bm{y} \in \real^d$.
            We can take $\bm{y} \in \real^d$ such that $\inner{\bm{z}}{\bm{y}} = \norm{\bm{z}}_* \norm{\bm{y}}$.
            Then,
            \begin{align}
                \norm{\bm{x}}^m + \inner{\bm{z}}{\bm{y}-\bm{x}} &= \norm{\bm{x}}^m + \norm{\bm{z}}_*\norm{\bm{y}} - m\norm{\bm{x}}^m \\
                & = (1-m) \norm{\bm{x}}^m + (m\norm{\bm{x}}^{m-1} + \varepsilon)\norm{\bm{y}}.
            \end{align}
            We first consider the case where $\bm{x}\neq \bm{0}$.
            If we choose $\bm{y}$ such that $\norm{\bm{y}}=\norm{\bm{x}}$, then
            \begin{align}
                (1-m)\norm{\bm{x}}^m + (m\norm{\bm{x}}^{m-1} + \varepsilon)\norm{\bm{y}} = (1+\varepsilon) \norm{\bm{y}} > \norm{\bm{y}}.
            \end{align}
            This contradicts to the definition of the subdifferential.
            We next consider the case where $\bm{x} = \bm{0}$.
            If we choose $\bm{y}$ such that $\varepsilon > \norm{\bm{y}}^{m-1}$, then
            \begin{align}
                (1-m)\norm{\bm{x}}^m + (m\norm{\bm{x}}^{m-1} + \varepsilon)\norm{\bm{y}} = \varepsilon \norm{\bm{y}} > \norm{\bm{y}}^{m-1}\norm{\bm{y}} = \norm{\bm{y}}^m,
            \end{align}
            and this also contradicts to the definition of the subdifferential.
            Thus $\norm{\bm{z}}_* \le m \norm{\bm{x}}^{m-1}$ for all $\bm{z} \in \partial \norm{\bm{x}}^m$.
        \end{enumerate}
        
        These results imply $\partial \norm{\bm{x}}^m \subseteq  m \norm{\bm{x}}^{m-1} \partial \norm{\bm{x}}$.
        
        From (i) and (ii), we have $\partial \norm{\bm{x}}^m = m \norm{\bm{x}}^{m-1} \partial \norm{\bm{x}}$.
    \end{proof}

    From~\cref{lemm:subdifferential_pow_norm}, we have the following corollaries.
    \begin{coro}
        \label{coro:subdifferential_pow_l1}
        For all $m \ge 1$, the subdifferential of $\norm{\cdot}_1^m$ at $\bm{p} \in \real^d$, $\partial \norm{\bm{p}}_1^m$ is
        \begin{align}
            \partial \norm{\bm{p}}_1^m = m \norm{\bm{p}}_1^{m-1} \partial \norm{\bm{p}}_1 = m \norm{\bm{p}}_1^{m-1} \prod_{i=1}^d J(p_i),\quad \text{where } J(p) = \begin{cases}
            \{1\} & p > 0,\\
            \{-1\} & p < 0,\\
            [-1, 1] & p = 0.
            \end{cases}
            \label{eq:subdifferential_pow_l1}
        \end{align}
    \end{coro}
    
    \begin{coro}
        \label{coro:subdifferential_pow_l21}
        For all $m \ge 1$, the subdifferential of $\norm{\cdot}_{2,1}^m$ at $\bm{P} \in \real^{d \times k}$, $\partial \norm{\bm{P}}_{2,1}^m$ is
        \begin{align}
            \partial \norm{\bm{P}}_{2,1}^m = m \norm{\bm{P}}_{2,1}^{m-1} \partial \norm{\bm{P}}_{2,1} = m \norm{\bm{P}}_{2,1}^{m-1} \{\bm{Z} \in \real^{d\times k}: \bm{z}_i \in \partial \norm{\bm{p}_i}_2, i \in [d]\}.
            \label{eq:subdifferential_pow_l21}
        \end{align}
    \end{coro}

\subsection{Proximal Operator for $\tilde{\ell}_{1,m}^m$}
    In this section, we prove~\cref{thm:solution_prox_higher_order_ti}.
    We first present some properties of $\bm{q}^* = \prox_{\lambda \norm{\cdot}_1^m}(\bm{p})$.
    \begin{prop}
        \label{prop:properties_prox_ti}
        Let $\bm{q}^* = \prox_{\lambda \norm{\cdot}_1^m}(\bm{p}) \in \real^d$.
        Then, the followings hold for all $i, j \in [d]$:
        \begin{enumerate}[(i)]
            \item $p_i = 0$ $\rightarrow$ $q^*_i = 0$,
            \item $p_i > 0$ $\rightarrow$ $q^*_i \ge 0$,
            \item $p_i < 0$ $\rightarrow$ $q^*_i \le 0$,
            \item $|p_i| \ge |p_j|$ $\rightarrow$ $|q^*_i| \ge |q^*_j|$,
        \end{enumerate}
        \end{prop}
    \begin{proof}
        (i) is trivial.
        
        (ii) Assume $p_i > 0$.
        We prove (ii) by showing that $\bm{q}$ is not an optimal value if $q_i < 0$.
        We consider the objective function in the proximal operation: $g_{\lambda \norm{\cdot}_1^m}(\bm{q}; \bm{p}) = \norm{\bm{p}-\bm{q}}_2^2/2 + \lambda\norm{\bm{q}}_1^m$.
        By construction $\bm{q}^* = \argmin_{\bm{q}} g_{\lambda \norm{\cdot}_1^m}(\bm{q}; \bm{p})$.
        Let $\bm{q}$ be the $d$-dimensional vector with $q_i < 0$ and $\bm{q}'$ be the $d$-dimensional vector such that $q'_i = |q_i|$ and $q'_j = q'_j$ for all $j \in [d] \setminus \{i\}$ (i.e, the sign of $i$-th element is reversed compared to $\bm{q}$).
        Then, $g_{\lambda \norm{\cdot}_1^m}(\bm{q}';\bm{p}) < g_{\lambda \norm{\cdot}_1^m}(\bm{q};\bm{p})$ since $\norm{\bm{q}'}_1^m = \norm{\bm{q}}_1^m$ but $\norm{\bm{p}-\bm{q}'}_2^2 < \norm{\bm{p}-\bm{q}}_2^2$.
        That is, when $p_i > 0$, $\bm{q}$ with $q_i < 0$ is not an optimal solution.
        It implies (ii).
        
        (iii) can be derived as in (ii).
        
        (iv) Assume $|p_i| > |p_j|$ and we prove (iv) as in the proof of (ii).
        Let $\bm{q}$ be the $d$-dimensional vector such that $|q_i| < |q_j|$ and $\sign(p_i)\cdot \sign(q_i) \ge 0$, $\sign(p_j)\cdot \sign(q_j) \ge 0$ (i.e., $q_i$ and $q_j$ satisfy (i)-(iii)).
        Moreover, let $\bm{q}'$ be the $d$-dimensional vector such that $q'_i = \sign(p_i)|q_j|$, $q'_j = \sign(p_j)|q_i|$, and $q'_{j'} = q_{j'}$ for all $j' \in [d] \setminus \{i, j\}$ (namely, $\bm{q}'$ is defined by exchanging the absolute value of $i$-th and $j$-th element in $\bm{q}$).
        Then, $g_{\lambda \norm{\cdot}_1^m}(\bm{q}';\bm{p}) < g_{\lambda \norm{\cdot}_1^m}(\bm{q};\bm{p})$ like (ii): the values of second terms are same ($\norm{\bm{q}'}_1^m = \norm{\bm{q}}_1^m$) but
        \begin{align}
            &\norm{\bm{p}-\bm{q}'}_2^2 - \norm{\bm{p}-\bm{q}}_2^2 = -2|p_i||q'_i| - 2|p_j||q'_j| + 2|p_i||q_i| + 2|p_j||q_j|\\
            &=  -2|p_i||q_j| - 2|p_j||q_i| + 2|p_i||q_i| + 2|p_j||q_j| = -2 (|p_i| - |p_j|)(|q_j| - |q_i|) < 0.
        \end{align}
        That is, when $|p_i| > |p_j| $, $\bm{q}$ with $|q_i| < |q_j|$ is not an optimal solution, and it implies (iv).
    
    \end{proof}
    
    Finally, we prove~\cref{thm:solution_prox_higher_order_ti}.
    \proxhigherorderti*
    \begin{proof}
        We first prove that there exists $\theta \in 
        [d]$ such that Equation~\eqref{eq:solution_prox_higher_order_ti} holds.
        $\bm{q}^* = \prox_{\lambda \norm{\cdot}_1^m}(\bm{p})$ means the subdifferential of $g_{\lambda \norm{\cdot}_1^m}$ at $\bm{q}^*$ includes $\bm{0}$.
        From $\partial g_{\lambda \norm{\cdot}_1^m}(\bm{q}^*) = \bm{q}^* - \bm{p} + \lambda \partial \norm{\bm{q}^*}_1^m$ and \cref{coro:subdifferential_pow_l1}, we have
        \begin{align}
            q^*_j = \begin{cases}
            p_j - \lambda m \norm{\bm{q}^*}_1^{m-1}& q_j^* > 0,\\
            p_j + \lambda m \norm{\bm{q}^*}_1^{m-1} & q_j^* < 0.
            \end{cases}
            \label{eq:optimality_condition1}
        \end{align}
        Moreover, from~\cref{prop:properties_prox_ti}, $\bm{q}^*$ is also sorted by absolute value and there exists $\theta \in [d]$ such that $|q_j| > 0$ for all $j \le \theta$ and $|q_j| = 0$ for all $j > \theta$.
        Therefore, Equation~\eqref{eq:optimality_condition1} can be rewritten as
        \begin{align}
            q^*_j &= \begin{cases}
            \sign(p_j)\left[|p_j| - \lambda m \left(\sum_{i=1}^{\theta}|q^*_i|\right)^{m-1} \right]& j \le \theta,\\
            0 & \text{otherwise},
            \end{cases}
            \label{eq:optimality_condition2}
        \end{align}
        Here, about the summation of $|q^*_1|, \ldots, |q^*_{\theta}|$, we have
        \begin{align}
            & \sum_{i=1}^{\theta} |q^*_i| = \sum_{i=1}^{\theta} \sign(q^*_i)q^*_i = \sum_{i=1}^{\theta} \sign(p_i)q^*_i = \sum_{i=1}^{\theta} |p_i| - \lambda m \theta \left(\sum_{i=1}^{\theta}|q_i^*|\right)^{m-1}.
        \end{align}
        Thus,~\eqref{eq:optimality_condition2} is rewritten as
        \begin{align}
            q^*_j &= \begin{cases}
            \sign(p_j)\left[|p_j| - \lambda m S_{\theta}^{m-1}\right] & j \le \theta,\\
            0 & \text{otherwise},
            \end{cases}
            \tag{\ref{eq:prox_higher_order_ti}}
        \end{align}
        where
        \begin{align}
            S_{j} \in \left[0, \sum_{i=1}^{j}|p_i|\right] \ \text{s.t.}\ \lambda m j S_{j}^{m-1} + S_{j} - \sum_{i=1}^j |p_i|=0.
            \tag{\ref{eq:condition_S}}
        \end{align}
        Such $S_{j}$ always exists uniquely for all $j \in [d]$ because $\lambda m j \cdot 0^{m-1} + 0 - \sum_{i=1}^j |p_i| \le 0$, $\lambda m j ( \sum_{i=1}^j|p_i|)^{m-1} +  \sum_{i=1}^j |p_i| -  \sum_{i=1}^j |p_i| \ge 0$, and $\lambda m j S^{m-1} + S$ is monotonically increasing w.r.t $S$ when $S\ge 0$.
        
        It is worth noting that $S_j$ is monotonically non-decreasing for all $j$ such that $|p_{j}| - \lambda m S_{j}^{m-1} \ge 0$.
        Assume that $|p_{j'}| - \lambda m S_{j'}^{m-1} \ge 0$.
        Then, for all $j \in [j']$, we have
        \begin{align}
            \lambda m j S_{j'}^{m-1} + S_{j'} &= \sum_{i=1}^{j} |p_i| + \left[\sum_{i=j+1}^{j'} |p_i| + \lambda m (j - j') S_{j'}^{m-1} \right] \\
            &= \sum_{i=1}^{j} |p_i| + \left[\sum_{i=j+1}^{j'} (|p_i| - \lambda m S_{j'}^{m-1})\right] \ge \sum_{i=1}^{j} |p_i| = \lambda m j S_{j}^{m-1} + S_{j},
            \label{eq:property_S}
        \end{align}
        where the inequality follows from the assumption $|p_j| \ge |p_{j'}|$ for all $j \in [j']$.
        Because $\lambda m j S^{m-1} + S$ is monotonically increasing w.r.t $S$ (when $S \ge 0$),~\eqref{eq:property_S} implies $S_{j'} \ge S_{j}$.
        Obviously, it also implies $|p_j| - \lambda m S_j \ge 0$ since $|p_j| \ge |p_{j'}|$.
        Moreover, if $|p_{j+1}| - \lambda m S_{j+1}^{m-1} = 0$, $S_{j+1} = S_j$.
        Assume that $|p_{j+1}| - \lambda m S_{j+1}^{m-1} = 0$.
        Then, $\lambda m (j+1)S_{j+1}^{m-1} + S_{j+1} - \sum_{i=1}^{j+1}|p_i| = \lambda m j S_{j+1}^{m-1} + S_{j+1} - \sum_{i=1}^{j}|p_i| = 0$ and thus $S_{j+1} = S_j$.
        
        Finally, we prove $\theta = \max\{j : |p_j| - \lambda m S_j^{m-1} \ge 0\}$ by contradiction.
        Since $|p_1| - \lambda m S_1^{m-1} = S_1$, the maximum value exists.
        Let $\theta' = \max\{j : |p_j| - \lambda m S_j^{m-1} \ge 0\}$.
        First suppose $\theta > \theta'$ (of course we assume that $\theta' < d$).
        Then, from the assumption $|p_\theta| - \lambda m S_{\theta}^{m-1} < 0$, we have $\sign(p_\theta) \sign(q_\theta) = -1$.
        This contradicts (i)-(iii) in~\cref{prop:properties_prox_ti}.
        Next suppose $\theta < \theta'$ (we assume $\theta' > 1$). 
        Then, the subdifferential of $g_{\lambda \norm{\cdot}_1^m}$ at $\bm{q}^*$ is
        \begin{align}
            \partial g_{\lambda \norm{\cdot}_1^m}(\bm{q}^*) & = \bm{q}^* - \bm{p} + \lambda \partial \norm{\bm{q}^*}_1^m = \bm{q}^* - \bm{p} + \lambda m \norm{\bm{q}^*}_1^{m-1} \prod_{i=1}^d J_i(q^*_i) \\
            & = \bm{q}^* - \bm{p} + \lambda m S_{\theta}^{m-1} \prod_{i=1}^d J(q^*_i).
        \end{align}
        If $S_{\theta} = S_{\theta'}$, we can clearly replace $\theta$ with $\theta'$ and thus we assume $S_{\theta} < S_{\theta'}$ ($S_j$ is monotonically non-decreasing for all $j \in [\theta']$).
        Then, $|p_\theta'| \ge \lambda m S_{\theta'}^{m-1} > \lambda m S_{\theta}^{m-1}$ and hence for all $\alpha \in [-1, 1]$
        \begin{align}
            q^*_{\theta'} - p_{\theta'} + \lambda m S_{\theta}^{m-1} \alpha = 0 - p_{\theta'} + \lambda m S_{\theta}^{m-1}\alpha \neq 0.
        \end{align}
        It implies $\bm{0} \not \in  \partial g_{\lambda \norm{\cdot}_1^m}(\bm{q}^*; \bm{p}) \iff \bm{q}^* \neq \argmin g_{\lambda \norm{\cdot}_1^m}(\bm{q}; \bm{p})$, and this is a contradiction.
    \end{proof}

\subsection{Proximal Operator for $\ell_{2, 1}^m$}
    \proxhigherordercs*
    \begin{proof}
        If $\bm{p}_j = \bm{0}$, clearly $\bm{q}_j^* = \bm{0}$.
        Then, we can eliminate $j$-th row vector from the proximal problem~\eqref{eq:prox_higher_order_cs}.
        Hence, with out loss of generality, we assume that $\bm{p}_j \neq \bm{0}$ for all $j \in [d]$.
        
        We first show that
        \begin{align}
            \exists C_j^* \ge 0 \ \text{s.t.}\ \bm{q}^*_j = C_j^* \bm{p}_j \ \forall j \in [d].
            \label{eq:prox_cs_property_propto}
        \end{align}
        If $\bm{q}^*_j = \bm{0}$, then $\bm{q}^*_j = C^*_j \bm{p}_j$ holds with $C^*_j=0$.
        Thus, we next consider the case where $\bm{q}^*_j \neq \bm{0}$.
        Let $g_{\lambda \norm{\cdot}_{2,1}^m}(\bm{Q}; \bm{P}) = \norm{\bm{P}-\bm{Q}}_2^2/2 + \lambda\norm{\bm{Q}}_{2,1}^m$.
        By construction, $\bm{Q}^* = \argmin_{\bm{Q}} g_{\lambda \norm{\cdot}_{2,1}^m}(\bm{Q}; \bm{P})$ and $\bm{0} \in \partial g_{\lambda \norm{\cdot}_{2,1}^m}(\bm{Q}^*; \bm{P})$.
        From~\cref{coro:subdifferential_pow_l21}, we have
        \begin{align}
            \bm{0} \in - \bm{p}_j + \bm{q}^*_j + \lambda m \norm{\bm{Q}^*}_{2,1}^{m-1}\partial \norm{\bm{q}^*_j}_2, \forall j \in [d].
        \end{align}
        By the assumption $\bm{q}^*_j \neq 0$, $\partial \norm{\bm{q}^*_j}_2 = \left\{\bm{q}^*_j / \norm{\bm{q}^*_j}_2\right\}$.
        Therefore, we obtain
        \begin{align}
            \bm{0} = - \bm{p}_j + \bm{q}^*_j + \lambda m \norm{\bm{Q}^*}_{2,1}^{m-1} \frac{1}{\norm{\bm{q}^*_j}_2}\bm{q}^*_j \rightarrow \left(1 + \frac{\lambda m \norm{\bm{Q}^*}_{2,1}^{m-1}}{\norm{\bm{q}^*_j}_2}\right)\bm{q}^*_j = \bm{p}_j.
        \end{align}
        It implies~\eqref{eq:prox_cs_property_propto}.
        
        Since~\eqref{eq:prox_cs_property_propto} holds, we consider the following optimization problem instead of~\eqref{eq:prox_higher_order_cs}:
        \begin{align}
            \bm{C}^* = \argmin_{\bm{C} \in \real^d_{\ge 0}} \frac{1}{2} \sum_{j=1}^d \left(\norm{\bm{p}_j}_2 - C_j\norm{\bm{p}_j}_2\right)^2 + \lambda \norm{\left(\norm{C_1\bm{p}_1}_2, \ldots, \norm{C_d\bm{p}_d}_2\right)^\top}_1^m.
        \end{align}
        Let $c^*_j = C^*_j  \norm{\bm{p}_j}_2$ for all $j \in [d]$.
        Then, we have $\bm{q}^*_j = (c^*_j/\norm{\bm{p}_j}_2)\cdot \bm{p}_j$, where
        \begin{align}
            \bm{c}^* &= \argmin_{\bm{c} \in \real^d_{\ge 0}} \frac{1}{2} \sum_{j=1}^d \left(\norm{\bm{p}_j}_2 - c_j\right)^2 + \lambda \norm{\bm{c}}_1^m
            \\
            &= \argmin_{\bm{c} \in \real^d} \frac{1}{2} \norm{(\norm{\bm{p}_1}_2, \ldots, \norm{\bm{p}_d}_2)^\top - \bm{c}}_2^2 + \lambda \norm{\bm{c}}_1^m\\
            &= \prox_{\lambda \norm{\cdot}_1^m}\left((\norm{\bm{p}_1}_2, \ldots, \norm{\bm{p}_d}_2)^\top\right),
        \end{align}
        which concludes the proof.
    \end{proof}
    Note that~\cref{thm:solution_prox_higher_order_cs} is a generalization of~\cref{thm:solution_prox_cs}.

\subsection{Regularization by Powers of Norms}
    \powofquasinorm*
    \begin{proof}
        ($\Rightarrow$) Let $\Omega:\real^{d \times k} \to \real_{\ge 0}$ be an $m$-homogeneous quasi-norm with $\Omega(\bm{P} + \bm{Q}) \le K (\Omega(\bm{P}) + \Omega(\bm{Q}))$ for all $\bm{P}, \bm{Q} \in \real^{d \times k}$.
        We show that $\sqrt[m]{\Omega}$ satisfies the axiom of quasi-norm.
        For all $\bm{P}, \bm{Q} \in \real^{d \times k}$, we have
        \begin{itemize}
            \item $\sqrt[m]{\Omega(\bm{P})} \ge 0$ and $\sqrt[m]{\Omega(\bm{P})} = 0 \iff \bm{P} = \bm{0}$ since $\Omega(\bm{P}) \ge 0$ and $\Omega(\bm{P}) = 0 \iff \bm{P} = \bm{0}$,
            \item $\sqrt[m]{\Omega(\alpha \bm{P})} = \sqrt[m]{\abs{\alpha}^m \Omega(\bm{P})} = \abs{\alpha}\sqrt[m]{\Omega(\bm{P})}$ for all $\alpha \in \real$, and
            \item $\sqrt[m]{\Omega(\bm{P} + \bm{Q})} \le \sqrt[m]{K(\Omega(\bm{P}) + \Omega(\bm{Q}))} \le \sqrt[m]{K}\left\{\sqrt[m]{\Omega(\bm{P})} + \sqrt[m]{\Omega(\bm{Q})}\right\}$.
        \end{itemize}
        Thus $\sqrt[m]{\Omega}$ is a quasi-norm.
        
        ($\Leftarrow$) Let $\norm{\cdot}'$ be a quasi-norm with $\norm{\bm{P} + \bm{Q}}' \le K (\norm{\bm{P}}' + \norm{\bm{P}}')$ for all $\bm{P}, \bm{Q} \in \real^{d \times k}$.
        We show that $(\norm{\cdot}')^m$ satisfies the axiom of $m$-homogeneous quasi-norm.
        For all $\bm{P}, \bm{Q} \in \real^{d \times k}$, the following holds:
        \begin{itemize}
            \item $(\norm{\bm{P}}')^m \ge 0$ and $(\norm{\bm{P}}')^m= 0 \iff \bm{P} = \bm{0}$ since $\norm{\bm{P}}' \ge 0$ and $\norm{\bm{P}}' = 0 \iff \bm{P} = \bm{0}$,
            \item $(\norm{\alpha \bm{P}}')^m=(|\alpha|\norm{\bm{P}}')^{m}=|\alpha|^m (\norm{\bm{P}}')^m$ for all $\alpha \in \real$, and
            \item since $\norm{\bm{P}}' \ge 0$ for all $\bm{P} \in \real^{d \times k}$ and $x^m$ is convex and monotone for all $x \ge 0$ and $m \ge 1$,
            \begin{align}
                (\norm{\bm{P}+\bm{Q}}')^{m} & \le \left\{K\left(\norm{\bm{P}}' + \norm{\bm{Q}}'\right)\right\}^m = (2K)^m\left\{\frac{1}{2}\norm{\bm{P}}' + \frac{1}{2}\norm{\bm{Q}}'\right\}^m\\
                & \le (2K)^m \left\{\frac{1}{2}\left(\norm{\bm{P}}'\right)^m + \frac{1}{2}\left(\norm{\bm{Q}}'\right)^m\right\} \\
                &= (2^{\frac{m-1}{m}}K)^m \left\{(\norm{\bm{P}}')^m + (\norm{\bm{Q}}')^m\right\},
            \end{align}
            where the last inequality follows from Jensen's inequality.
        \end{itemize}
    \end{proof}
    
    Before proving~\cref{thm:bound_by_power_of_norm}, we show that all quasi-norms on a finite-dimensional vector space (in this paper we consider $\real^{d \times k}$) are equivalent to each other.
    \begin{lemm}
        \label{lemm:equivalence_of_quasi_norm}
        For given two quasi-norms on $\real^{d \times k}$, $\Omega, \Omega': \real^{d \times k} \to \real$, there exists $c, C \in \real_{>0}$ and for all $\bm{P}$ such that
        \begin{align}
            c\Omega(\bm{P}) \le \Omega'(\bm{P}) \le C\Omega(\bm{P}).
            \label{eq:equivalence_of_quasi_norm}
        \end{align}
    \end{lemm}
    \begin{proof}
    We consider only $\bm{P} \neq \bm{0}$ since it trivially holds for $\bm{P}=\bm{0}$.
    We can prove it like the equivalence of norms on a finite-dimensional vector space.
    Let $\Omega_{\infty}(\bm{P}) \coloneqq \max_{j \in [d], s \in [k]} |p_{j,s}|$.
    $\Omega_{\infty}$ is clearly (quasi-)norm and it is sufficient to prove that any quasi-norm is equivalent to $\Omega_{\infty}$: if for given two quasi-norms $\Omega, \Omega': \real^{d \times k} \to \real$ there exists $c, C, c', C' > 0$ such that $c\Omega_{\infty}(\bm{P}) \le \Omega(\bm{P}) \le C\Omega_{\infty}(\bm{P})$ and $c'\Omega_{\infty}(\bm{P}) \le \Omega'(\bm{P}) \le C'\Omega_{\infty}(\bm{P})$ for all $\bm{P} \in \real^{d \times k}$, then~\eqref{eq:equivalence_of_quasi_norm} holds: $C/c' \Omega(\bm{P}) \le \Omega'(\bm{P}) \le C'/c \Omega(\bm{P})$.
    
    We first show that for any quasi-norm $\Omega$ there exists $C > 0$ such that $\Omega(\bm{P}) \le C\Omega_{\infty}(\bm{P})$ for all $\bm{P}$.
    Assume that $\Omega$ is quasi-norm with $\Omega(\bm{P} + \bm{Q})\le K (\Omega(\bm{P}) + \Omega(\bm{Q}))$ ($K \ge 1$).
    Let $\bm{E}^{d, k}_{j, s}$ be the $d \times k$ matrix such that its $(j, s)$ element is $1$ and others are $0$.
    Then, we have
    \begin{align}
        \Omega(\bm{P}) &= \Omega\left(\sum_{j \in [d], s \in [k]} p_{j, s}\bm{E}^{d,k}_{j, s}\right) \le  K^{dk}\sum_{j\in [d], s \in [k]} \abs{p_{j, s}} \Omega(\bm{E}^{d,k}_{j, s})  \\ 
        &\le \left(K^{dk}\sum_{j\in [d], s \in [k]} \Omega(\bm{E}^{d,k}_{j, s})\right) \max_{j \in [d], s \in [k]} \abs{p_{j, s}} = \left(K^{dk}\sum_{j\in [d], s \in [k]} \Omega(\bm{E}^{d,k}_{j, s})\right) \Omega_{\infty}(\bm{P}).
    \end{align}
    Since $K^{dk}\sum_{j\in [d], s \in [k]} \Omega(\bm{E}^{d,k}_{j, s})$ does not depend on $\bm{P}$, setting $C$ to be $K^{dk}\sum_{j\in [d], s \in [k]} \Omega(\bm{E}^{d, k}_{j, s})$ produces $\Omega(\bm{P}) \le C\Omega_{\infty}(\bm{P})$ for all $\bm{P}$.

    We next show that for any quasi-norm $\Omega$ there exists $c > 0$ such that $ c\Omega_{\infty}(\bm{P}) \le \Omega (\bm{P})$ for all $\bm{P}$.
    We prove it by contradiction: suppose that there does not exist $c > \real$ such that $c\Omega_{\infty}(\bm{P}) \le \Omega(\bm{P})$.
    This implies for all $c > 0$ there exists $\bm{P}$ such that $\Omega_{\infty}(\bm{P}) / \Omega (\bm{P}) > 1/c$, i.e., $\sup \Omega_{\infty}(\bm{P}) / \Omega (\bm{P}) = \infty$.
    Thus, for all $n \in \natu$, there exists $\bm{P}^{(n)}$ such that $\Omega_{\infty}(\bm{P}^{(n)})/\Omega(\bm{P}^{(n)}) > n$, and without loss of generality we can assume that $\Omega_{\infty}(\bm{P}^{(n)})=1$ and $\Omega(\bm{P}^{(n)}) < 1/n$ (if $\Omega_{\infty}(\bm{Q}^{(n)})/\Omega(\bm{Q}^{(n)}) > n$, then define $\bm{P}^{(n)}$ as $\bm{Q}^{n} / \Omega_{\infty}(\bm{Q}^{(n)})$).
    We show the contradiction by constructing a convergent sequence $\lim_{n\to \infty} \bm{P}^{(n)} = \bm{0}$.
    
    Since $\Omega_{\infty}(\bm{P}^{(n)}) = \sup_{j \in [d], s \in [k]} |p_{j,s}|$ for all $n \in \natu$, there exists $(j', s') \in [d] \times [k]$ such that $p_{j', s'}^{(n)}=1$ for infinitely many $n \in \natu$.
    We assume $(j', s')=(1,1)$ without loss of generality and $\{\bm{P}^{(n)}\}$ be the sequence such that $\Omega_{\infty}(\bm{P}^{(n)}) = 1$, $\Omega(\bm{P}^{(n)}) < 1/n$, and $p_{1,1}^{(n)}=1$.
    Then, $\{\bm{P}^{(n)}\}$ is a bounded sequence in $\real^{d \times k}$ and hence it has a convergent subsequence on the normed space $(\real^{d \times k}, \norm{\cdot}_2)$ (Bolzano–Weierstrass theorem).
    Let $\bm{P}$ be the limit of that subsequence and we re-define $\{\bm{P}^{(n)} \}$ as the corresponding convergent subsequence.
    Then, $\Omega_{\infty}(\bm{P} - \bm{P}^{(n)}) \to 0$ as $n \to \infty$ since $\Omega_{\infty}$ is a norm and we have
    \begin{align}
        \Omega(\bm{P}) &\le K (\Omega (\bm{P} - \bm{P}^{(n)}) + \Omega(\bm{P}^{(n)}))\\
        &\le K \left(C\Omega_{\infty}(\bm{P} - \bm{P}^{(n)}) + \frac{1}{n}\right).
    \end{align}
    Since $n$ is arbitrary, $\Omega(\bm{P}) = 0$ and hence $\bm{P}=\bm{0}$.
    However, $p^{(n)}_{1,1}=1$ for all $n \in \natu$.
    It implies $p_{1,1}=1$ and $\bm{P} \neq \bm{0}$.
    This is a contradiction and thus the assumption is wrong.
    Therefore, there exists $c > \real$ such that $c\Omega_{\infty}(\bm{P}) \le \Omega(\bm{P})$.
    \end{proof}
    
    Finally, we prove~\cref{thm:bound_by_power_of_norm}.
    \boundbypowernorm*
    \begin{proof}
        From~\cref{thm:pow_of_quasi_norm} and~\cref{lemm:equivalence_of_quasi_norm}, it is sufficient to prove there exists a quasi-norm $\norm{\cdot}$ such that $\Omega^m_*(\bm{P}) \le C \norm{\bm{P}}^m$ for all $\bm{P} \in \real^{d \times k}$, and here $\norm{\cdot}_1$ corresponds to such a norm: $\Omega_{*}^m(\bm{P}) \le \Omega_{\mathrm{TI}}^m(\bm{P}) \le \norm{\bm{P}}_1^m$.
    \end{proof}

\subsection{Justification for TI Regularizer}
    \label{subsec:proof_analysis}
    \equivalenceqr*
    \begin{proof}
        Without loss of generality, we can omit the linear term.
        We first consider the case $\lambda_p = 0$.
        We prove~\eqref{eq:equivalence_tisfm_qr} with $\lambda_p=0$ by showing that for any strictly upper triangular matrix $\bm{W} \in \real^{d \times d}$ there exists $\bm{P} \in \real^{d \times d(d-1)/2}$ such that
        \begin{align}
            \bm{P}\bm{P}^\top =_{>} \bm{W},\ \Omega_{\mathrm{TI}}(\bm{P}) = \Omega_{*}(\bm{W}).
        \end{align}
        It is sufficient for~\eqref{eq:equivalence_tisfm_qr} since $\bm{P}\bm{P}^\top =_{>} \bm{W}$ implies $L_{\mathrm{FM}}(\bm{w}, \bm{P}; \mathcal{D}, \lambda_{w}, 0) =  L_{\mathrm{QR}}(\bm{w}, \bm{W}; \mathcal{D}, \lambda_{w}, 0)$ and $\bm{W}^*_{\mathrm{QR}}$ is always strictly upper triangular matrix since lower triangle elements are not used in $f_{\mathrm{QR}}$.
        Fix $\bm{W}$ be a strictly upper triangular matrix and let $\bm{Q}$ be the $d \times d^2$ matrix with
        \begin{align}
            \bm{q}_j = \mathrm{vec}\left(\left(\underbrace{\sqrt{|w_{1, j}|}\bm{e}^d_j,\ldots, \sqrt{|w_{j-1, j}|}\bm{e}^d_j}_{j-1}, \sign(\bm{w}_j)\circ\sqrt{\mathrm{abs}(\bm{w}_j)}, \underbrace{\bm{0}, \ldots, \bm{0}}_{d-j}\right)\right) (\in \real^{d^2}),
        \end{align}
        where $\sqrt{\cdot}$ for a vector is the element-wise square root.
        Then, for all $0 < j_1 < j_2 \le d$,
        \begin{align}
            \inner{\bm{q}_{j_1}}{\bm{q}_{j_2}} &= 
            \sum_{i=1}^{j_1-1} \inner{\sqrt{\abs{w_{i,j_1}}}\bm{e}^d_{j_1}}{\sqrt{\abs{w_{i,j_2}}}\bm{e}^d_{j_2}} 
            + \inner{\sign(\bm{w}_{j_1})\circ\sqrt{\mathrm{abs}(\bm{w}_{j_1})}}{\sqrt{\abs{w_{j_1,j_2}}}\bm{e}^d_{j_2}} 
            \nonumber \\
            &\quad  + \sum_{i=j_1+1}^{j_2-1}\inner{\bm{0}}{\sqrt{\abs{w_{i,j_2}}}\bm{e}^d_{j_2}} + \inner{\bm{0}}{\sign(\bm{w}_{j_2})\circ\sqrt{\mathrm{abs}(\bm{w}_{j_2})}} + \sum_{i=j_2+1}^d \inner{\bm{0}}{\bm{0}} \\
            &= (d^2-1)\cdot 0 + \sign(w_{j_1, j_2})\sqrt{\abs{w_{j_1, j_2}}}\sqrt{\abs{w_{j_1, j_2}}}= w_{j_1, j_2} \iff \bm{Q}\bm{Q}^\top =_{>} \bm{W},\\
            \sum_{i=1}^{d^2} \abs{q_{j_1, i}}\abs{q_{j_2, i}} &=\abs{w_{j_1, j_2}} \iff \Omega_{\mathrm{TI}}(\bm{Q}) = \Omega_{*}(\bm{W}).
        \end{align}
        This proves~\eqref{eq:equivalence_tisfm_qr} with $\lambda_p=0$ when $k \ge d^2$.
        
        Here, we show that $\bm{Q}$ has $d(d+1)/2$ all-zeros columns.
        Let $\bm{Q}^j = (\bm{q}_{:, d(j-1)+1}, \cdots, \bm{q}_{:, d(j-1)+d}) \in \real^{d\times d}$, i.e., $\bm{Q} = (\bm{Q}^1, \cdots, \bm{Q}^d)$
        Then, $1, \ldots, j$-th columns in $\bm{Q}^j$ are all-zeros vectors since the row vectors in $\bm{Q}^{j}$ are
        \begin{align}
            &\bm{q}^{j}_{j_1} = (q_{j_1, d(j-1)+1} \ldots, q_{j_1, d(j-1)+d})^\top=\bm{0} \text{ for all } j_1 < j,\\
            &\bm{q}^{j}_{j} = \bm{w}_j = (0, \ldots, 0, w_{j,  j+1}, \ldots, w_{j, d})^\top,\\
            &\bm{q}^j_{j_2} = \bm{e}^d_{j_2} = (\underbrace{0, \ldots, 0}_{j_2-1}, 1, 0, \ldots, 0)^\top\text{ for all } j_2 > j.
        \end{align}
        Thus, $\bm{Q}$ has $1+2+\cdots +d=d(d+1)/2$ all-zeros columns and let $\bm{P} \in \real^{d (d-1)/2}$ be the sub-matrix of $\bm{Q}$ such that its all-zeros columns are removed.
        Then $\bm{P}\bm{P}^\top =_{>} \bm{W}$ and $\Omega_{\mathrm{TI}}(\bm{P}) = \norm{\bm{W}}_1$.
        It proves~\eqref{eq:equivalence_tisfm_qr} with $\lambda_p = 0$.
        Furthermore, since $\norm{\bm{W}}_1 \le \Omega_{\mathrm{TI}}(\bm{P})$ for all $\bm{P}\bm{P}^\top =_{>} \bm{W}$, the inverse inequality clearly holds if $\lambda_p=0$:
        \begin{align}
            \min_{\bm{w}\in \real^d, \bm{P} \in \real^{d \times k}} L_{\mathrm{FM}}(\bm{w}, \bm{P}; \mathcal{D}, \lambda_{w}, 0) + \tilde{\lambda}_p \Omega_{\mathrm{TI}}(\bm{P}) \ge \min_{\bm{w}\in \real^d, \bm{W} \in \real^{d \times d}} L_{\mathrm{QR}}(\bm{w}, \bm{W}; \mathcal{D}, \lambda_{w}) + \tilde{\lambda}_p \norm{\bm{W}}_1.
        \end{align}
        It implies the equality holds in~\eqref{eq:equivalence_tisfm_qr} and $f_{\mathrm{FM}}(\bm{x}; \bm{w}^*_{\mathrm{TI}}, \bm{P}^*_{\mathrm{TI}}) = f_{\mathrm{QR}}(\bm{x}; \bm{w}^*_{\mathrm{QR}}, \bm{W}^*_{\mathrm{QR}})$ for all $\bm{x} \in \real^d$.
        
        Next, we prove~\eqref{eq:equivalence_tisfm_qr} with $\lambda_p \ge 0$.
        For $\bm{P}$ as defined above, we have
        \begin{align}
            \norm{\bm{P}}_2^2 = \sum_{j=1}^d \norm{\bm{p}_j}_2^2 = \sum_{j=1}^d\left\{\sum_{i=1}^{j-1}|w_{i, j}| + \sum_{i=1}^{d}|w_{j, i}| \right\} = 2 \norm{\bm{W}}_1.
            \label{eq:squared_l2_l1}
        \end{align}
        Combining it with the above-mentioned proof for $\lambda_p=0$ implies~\eqref{eq:equivalence_tisfm_qr} for all $\lambda_p \ge 0$.
    \end{proof}
    We also obtain a similar relationship between TI-sparse FMs and $\Omega_*$-sparse FMs.
    \begin{restatable}{thm}{equivalenceoptim}
        \label{thm:equivalence_tisfm_optimsfm}
        For any $\lambda_{w}, \lambda_{p}, \tilde{\lambda}_{p} \ge 0$ and $k_* \in \natu_{>0}$, there exists $k' \le d(d-1)/2$ such that for all $k \ge k'$,
        \begin{align}
            &\min_{\bm{w}\in \real^d, \bm{P} \in \real^{d \times k}} L_{\mathrm{FM}}(\bm{w}, \bm{P}; \lambda_{w}, \lambda_{p}) + \tilde{\lambda}_p\Omega_{\mathrm{TI}}(\bm{P}) \nonumber \\
            &\le \min_{\bm{w}\in \real^d, \bm{P} \in \real^{d \times k_*}} L_{\mathrm{FM}}(\bm{w}, \bm{P}; \lambda_{w}, (d-1)\lambda_{p}) + \tilde{\lambda}_p\Omega_{*}(\bm{P}) .
            \label{eq:equivalence_tisfm_optimsfm}
        \end{align}
    \end{restatable}
    \begin{proof}
        Let $\bm{P}_{\Omega_{*}}^*$ be the optimal solution of the RHS in~\eqref{eq:equivalence_tisfm_optimsfm}.
        Then, we easily obtain~\eqref{eq:equivalence_tisfm_optimsfm} with $\lambda_p = 0$ by substituting the strictly upper triangular elements of $\bm{P}_{\Omega_{*}}^*(\bm{P}_{\Omega_{*}}^*)^\top$ to those of $\bm{W}$ in the proof of~\cref{thm:equivalence_tisfm_qr}.
        Thus, for $\lambda_p \ge 0$, we show that $\norm{\bm{P}}_2^2 \le (d-1)\norm{\bm{P}^*_{\Omega_*}}_2^2$, where $\bm{P} \in \real^{d \times d(d-1)/2}$ is constructed as in the proof of~\cref{thm:equivalence_tisfm_qr}.
        It is sufficient for~\eqref{eq:equivalence_tisfm_optimsfm}.
        From~\eqref{eq:squared_l2_l1}, we have
        \begin{align}
            \norm{\bm{P}}_2^2 &= 2\sum_{j_2>j_1}\abs{w_{j_1, j_2}} = 2\sum_{j_2>j_1}\abs{\inner{\bm{p}^*_{\Omega_*, j_1}}{\bm{p}^*_{\Omega_*, j_2}}} \le  2 \sum_{j_2 > j_1} \norm{\bm{p}^*_{\Omega_*, j_1}}_2 \norm{\bm{p}^*_{\Omega_*, j_2}}_2\\
            &\le  \sum_{j_2 > j_1} \norm{\bm{p}^*_{\Omega_*, j_1}}^2_2  + \norm{\bm{p}^*_{\Omega_*, j_2}}^2_2 = (d-1)\norm{\bm{P}_{\Omega_*}^*}_2^2.
        \end{align}
    \end{proof}
\section{Implementation Details}
    \label{sec:impl}
    In this section, we briefly review and show the implementation details of the existing and proposed algorithms.
    
    \paragraph{$O(d \log d)$ Time Algorithm for Proximal Operator~\eqref{eq:prox_ti}~\citep{filipe2011online}.} \Cref{alg:prox_ti_sort} shows the procedure for solving~\eqref{eq:prox_ti} in $O(d \log d)$ time, where $p_{(i)}$ is the $i$-th largest value of $p_{(\theta)}$ in absolute value; i.e., $|p_{(1)}| \ge |p_{(2)}| \ge \cdots \ge |p_{(d)}|$ and $\forall j \in [d]$ $\exists$ $i \in [d]$ s.t. $p_j = p_{(i)}$.
    \begin{algorithm}[t]
        \caption{Computation of the proximal operator~\eqref{eq:prox_ti} with sorting $\bm{p}$ ($O(d \log d)$).}
        \label{alg:prox_ti_sort}
        \begin{algorithmic}[1]
            \Input{$\bm{p} \in \real^d$, $\lambda \ge 0$}
            \State{$\tilde{\bm{p}} \leftarrow (p_{(1)}, \ldots, p_{(d)})^\top$;}\Comment{$O(d \log d)$}
            \State{$S_0\leftarrow 0$;}
            \For{$j = 1, \ldots, d$}
                \State{$S_j \leftarrow S_{j-1} + \abs{\tilde{p}_j}$;}
            \EndFor
            \State{$S_j \leftarrow S_{j} / (1 + 2\lambda j)$ for all $j \in [d]$;}
            \State{$\theta \leftarrow \max\{j \in [d]: \abs{\tilde{p}_j} - 2 \lambda S_j \ge 0\}$;}
            \State{$q^*_j \leftarrow \sign(p_j) \max \{\abs{p_j} - 2 \lambda S_{\theta}, 0\}$ for all $j \in [d]$;}
            \Output{$\bm{q}^* (= \prox_{\lambda \norm{\cdot}_1^2}(\bm{p}))$}
        \end{algorithmic}
    \end{algorithm}
    
    \paragraph{$O(d)$ Time Algorithm for Proximal Operator~\eqref{eq:prox_ti}~\citep{filipe2011online}.} In fact, the proximal operator~\eqref{eq:prox_ti} can be computed in $O(d)$ time.
    Given $p_{(\theta)}$ (not the index $\theta$, but the value $p_{(\theta)}$), one can compute $S_{\theta}$ in $O(d)$ time: $S_{G_{(\theta)}} \coloneqq \sum_{j \in G_{(\theta)}}|p_j| / (1+2\lambda |G_{(\theta)}|) = S_{\theta}$, where $G_{(\theta)} \coloneqq \{j \in [d]: |p_j| \ge |p_{(\theta)}|\}$ (clearly, $\theta = |G_{(\theta)}|$).
    Thus, even if $\bm{p}$ is not sorted by absolute value, one can compute $\bm{q}^*$ in $O(d)$ time if only $p_{(\theta)}$ is found. Fortunately, one can find $p_{(\theta)}$ in $O(d)$ time in expectation by using the randomized-median-finding-like algorithm~\citep{duchi2008efficient}.
    \Cref{alg:prox_ti_randomize} shows the procedure for computing the proximal operator~\eqref{eq:prox_ti} in $O(d)$ time.
    It finds $p_{(\theta)}$ and computes $S_{\theta}$ by repeating (i) randomly sampling $p_i$ from $\{p_i: j \in C\}$, (ii) determining whether $|p_i| \ge |p_{(\theta)}|$ (although $p_{(\theta)}$ is unknown), and (iii) reducing $C$ in accordance with whether $|p_i| \ge |p_{(\theta)}|$, where $C \subseteq [d]$ is the set of candidates of $\theta$ initialized as $[d]$.
    When $i \in C$ is sampled, the algorithm partitions the candidate index set $C$ into $G \coloneqq \{j \in C: |p_j| \ge |p_i|\}$ and $L\coloneqq \{j \in C: |p_j| < |p_i|\}$ and discards one of them as follows.
    If $|p_i| \ge \lambda S_{G_i}$, where $G_i \coloneqq \{j \in [d]: |p_j| \ge |p_i|\}$, $|p_{(\theta)}|$ is smaller than $|p_i|$, so the algorithm next searches for $p_{(\theta)}$ from $L$ (i.e., it discards $G$).
    In this case, the algorithm updates $S$ and $\theta$ as $S + \sum_{j \in G}|p_j|$ and $\theta +|G|$, respectively (both $S$ and $\theta$ are initially $0$ and are simply desired values when $C=\emptyset$). Maintaining $S$ and $\theta$ reduces the computation cost of $S_{G_i}$: $S_{G_i} = (S + \sum_{j \in G}|p_j|)/[1+2\lambda(\theta + |G|)]$.
    Otherwise (that is, if $|p_i| < 2\lambda S_{G_i}$), the algorithm discards $L$ and sets $C=G_i \setminus \{j \in G: |p_j| = |p_i|\}$ since $|p_{(\theta)}|$ is larger than $|p_i|$.

    \begin{algorithm}[t]
        \caption{Computation of the proximal operator~\eqref{eq:prox_ti} without sorting ($O(d)$)}
        \label{alg:prox_ti_randomize}
        \begin{algorithmic}[1]
            \Input{$\bm{p} \in \real^d$, $\lambda \ge 0$}
            \State{$C \leftarrow [d]$;}\Comment{Candidate index set}
            \State{$S \leftarrow 0$;}
            \State{$\theta \leftarrow 0$;}
            \While{$C \neq \emptyset$}\Comment{Find $p_{(\theta)}$}
                \State{Pick $i \in C$ at random;}
                \State{$G \leftarrow \{j \in C: |p_j| \ge |p_i|\}$}
                \State{$L \leftarrow \{j \in C: |p_j| < |p_i|\}$}
                \State{$S_{G_i} \leftarrow (S + \sum_{j \in G}|p_j|)/(1+2\lambda(\theta + |G|))$}
                \If{$|p_i| - 2\lambda S_{G_i}\ge 0$}\Comment{$|p_i| \ge |p_{(\theta)}|$}
                    \State{$C \leftarrow L$, $S \leftarrow S + \sum_{j \in G}|p_j|, \theta \leftarrow \theta + |G|;$}
                    \Else\Comment{$|p_i| < |p_{(\theta)}|$}
                    \State{$C \leftarrow G\setminus \{j \in G: |p_j| = |p_i|\}$;}
                \EndIf
            \EndWhile
            \State{$S_{\theta} \leftarrow S / (1+2\lambda\theta)$;}
            \State{$q^*_j \leftarrow \sign(p_j) \max \{|p_j| - 2\lambda S_{\theta}, 0\}$ for all $j \in [d]$;}
            \Output{$\bm{q}^* (= \prox_{\lambda \norm{\cdot}_1^2}(\bm{p}))$}
        \end{algorithmic}
    \end{algorithm}
    
    \paragraph{CD Algorithm for Canonical FMs.}~\cref{alg:cd_fm} shows the CD algorithm for objective function~\eqref{eq:objective_fm}.
    The CD algorithm requires the predictions of all training instances $f_n = f_{\mathrm{FM}}(\bm{x}_n; \bm{w}, \bm{P})$ for all $n \in \supp(\bm{x}_{:, j})$ for updating $p_{j, s}$.
    The na\"ive computation of such $f_n$ at each iteration is too costly, so a method for \emph{updating} (\emph{synchronizing}) predictions is essential for an efficient implementation.
    Typically, an efficient algorithm computes and caches $f_n$ for all $n \in [N]$ before starting the optimization and synchronizes them every time a parameter is updated.
    Fortunately, FMs are multi-linear w.r.t $w_1, \ldots, w_d$ and $p_{1,1}, \ldots, p_{d, k}$, so each prediction can be easily synchronized in $O(1)$.
    For $w_j$, the predictions are written as $f_n = x_{n, j}w_j + \mathrm{const}$ for all $n \in [N]$, so $f_n$ are synchronized as $f_n \leftarrow f_n - x_{n, j} \eta \delta$ after updating $w_j$ as $w_j \leftarrow w_j - \eta \delta$.
    To be more precise, the algorithm synchronizes $f_n$ for only $n \in \supp(\bm{x}_{:, j})$ since $f_n$ does not change for all $n \not \in \supp(\bm{x}_{:, j})$.
    For $p_{j, s}$, the predictions are written as $f_n = f'_{n,s}p_{j, s} + \mathrm{const}$, where $f'_{n,s} = x_{n, j}(a_{n,s} - x_{n,j}p_{j,s})$ and $a_{n,s} := \inner{\bm{x}_{n}}{\bm{p}_s}$.
    Thus the algorithm can synchronize $f_n$ as in the case for $w_j$ by caching $a_{n, s}$ for all $n \in [N]$, $s \in [k]$.
    Clearly, the gradient of the loss w.r.t $p_{j, s}$ can be efficiently computed by using $f'_{n,s}$: $\partial \ell (f_n, y_n) / \partial p_{j,s} = \partial \ell(f_n, y_n)/\partial f_n \cdot f'_{n,s}$.
    At each iteration, the algorithm requires $a_{n, s}$ for all $n \in [N]$ and therefore the additional space complexity for caches is $O(N)$ (the algorithm does not require $a_{n,s}$ for all $s \in [k]$ at the same time).
    
    \begin{algorithm}[t]
        \caption{CD algorithm for canonical FMs}
        \label{alg:cd_fm}
        \begin{algorithmic}[1]
            \Input{$\{(\bm{x}_n, \bm{y}_n) \}_{n=1}^{N}$, $k \in \natu_{>0}$, $\lambda_{w}, \lambda_{p} \ge 0$}
            \State{Initialize $\bm{P} \in \real^{d\times k}$, $\bm{w} \in \real^d$;}
            \State{Compute predictions: $f_{n} \leftarrow f_{\mathrm{FM}}(\bm{x}_n;\bm{w}, \bm{P})$ for all $n \in [N]$;}
            \While{not convergence}
                \For{$j=1, \ldots, d$}\Comment{Update $\bm{w}$}
                    \State{$\delta \leftarrow \sum_{i \in \supp(\bm{x}_{:, j})}\ell'(f_n, y_n) x_{n, j}/N + 2\lambda_{w} w_j$};
                    \State{$\eta \leftarrow (\mu \norm{\bm{x}_{:, j}}_2^2/N + 2\lambda_{w})^{-1}$;}
                    \State{$w_j \leftarrow w_j - \eta \delta$;}
                    \State{$f_n \leftarrow f_n -  x_{n,j}\eta \delta$ for all $n \in \supp(\bm{x}_{:, j})$;}
                \EndFor
                \For{$s=1, \ldots, k$}\Comment{Update $\bm{P}$}
                    \State{$a_{n, s} \leftarrow \inner{\bm{x}_n}{\bm{p}_{:, s}}$ for all $n \in [N]$;}\Comment{Cache for updating $\bm{p}_{:, s}$}
                    \For{$j=1, \ldots, d$}
                        \State{$f_{n,s}' \leftarrow x_{n, j} (a_{n,s} - x_{n, j}p_{j, s})$ for all $n \in \supp(\bm{x}_{:, j})$;}
                        \State{$\delta \leftarrow \sum_{i \in \supp(\bm{x}_{:, j})}\ell'(f_n, y_n)\cdot f_{n, s}'/N + 2\lambda_{p} p_{j, s}$;}
                        \State{$\eta \leftarrow [\mu \sum_{n \in \supp(\bm{x}_{:, j})} (f'_{n,s})^2/N + 2\lambda_{p}]^{-1}$;}
                        \State{$p_{j, s} \leftarrow p_{j, s} - \eta \delta$;}
                        \State{$f_{n} \leftarrow f_n - f'_{n, s}\eta \delta$ for all $n \in \supp(\bm{x}_{:, j})$;}
                        \State{$a_{n, s} \leftarrow a_{n, s} - x_{n, j}\eta \delta$ for all $n \in \supp(\bm{x}_{:, j})$;}
                    \EndFor
                \EndFor
            \EndWhile
            \Output{Learned $\bm{P}$ and $\bm{w}$}
        \end{algorithmic}
    \end{algorithm}
    
    \paragraph{SGD for Canonical FMs.}
    \Cref{alg:sgd_fm} shows the SGD algorithm for objective function~\eqref{eq:objective_fm}.
    Since it is almost straightforward, we only describe the \emph{lazy update} technique~\citep{duchi2011adaptive}, which improves the efficiency of the algorithm by leveraging the sparsity of a sampled instance.
    For all $j \not \in \supp(\bm{x}^t)$, where $\bm{x}^t$ is the sampled instance at $t$-th iteration, the gradients of loss function w.r.t $w_j$ and $\bm{p}_j$ are $2\lambda_w w_j$ and $2 \lambda_p \bm{p}_j$, and update rules are $w_j \leftarrow (1-2\eta^t_{w})w_j$ and $\bm{p}_j \leftarrow (1-2\eta^t_{p}) \bm{p}_j$, respectively.
    The lazy update technique enables us to omit updating such parameters and makes the computational cost of each iteration $O(\mathrm{nnz}(\bm{x}_{n})+\mathrm{nnz}(\bm{x}_{n})k)$ while a na\"ive implementation takes $O(d+dk)$.
    Although we hereinafter consider for only $\bm{w}$ for simplicity, the same holds for $\bm{P}$.
    Assume that the algorithm is at $t$-th iteration and $x^t_{j} \neq 0$, and also assume that there exists $t_j < t$ such that $x^{t_j}_j \neq 0$ and $x^{t'}_j = 0$ for all $t_j < t' < t$ (namely, $j$-th feature is $0$ from $(t_j+1)$-th until $(t-1)$-th iteration.
    Then, the value of $w_j$ at $(t-1)$-th iteration, $w^{t-1}_j$, is written as
    \begin{align}
        w^{t-1}_j &= (1-2\eta^{t-1}_w\lambda_w)w^{t-2}_j = \left\{\prod_{t'=t_j+1}^{t-1} (1-2\eta_w^{t'} \lambda_w)\right\} w_j^{t_j} \\
        & =\left\{\frac{\prod_{t'=1}^{t-1} (1-2\eta_w^{t'} \lambda_w)}{\prod_{t'=1}^{t_j} (1-2\eta_w^{t'} \lambda_w)}\right\} w_j^{t_j} =: \frac{\alpha^{t-1}_{w}}{\alpha^{t_j}_{w}} w_j^{t_j}.
    \end{align}
    Therefore, given $\alpha^{t-1}_w$ and $\alpha^{t_j}_{w}$, the algorithm can update $w^t_j$ from $w^{t_j}_j$ in $O(1)$.
    Since $x^{t'}_j = 0$ for all $t' \in \{t_j+1, \ldots, t-1\}$, omitting to update $w_j$ from $(t_j+1)$-th until $(t-1)$-th iteration does not affect the result.
    $\alpha_{w,j}$ in~\cref{alg:sgd_fm} corresponds to $\alpha_w^{t_j}$.
    In our implementation, if $\alpha_w < 10^{-9}$, algorithm updates all $w_j$ and resets $\alpha_w$ and $\alpha_{w,j}$ for all $j \in [d]$ in order to avoid numerical errors.
    The additional space complexity for caches is $O(d)$ (scaling values $\alpha_{w, j}, \alpha_{p,j}$ for all $j \in [d]$).

    \begin{algorithm}[t]
        \caption{SGD algorithm for canonical FMs}
        \label{alg:sgd_fm}
        \begin{algorithmic}[1]
            \Input{$\{(\bm{x}_n, \bm{y}_n) \}_{n=1}^{N}$, $k \in \natu_{>0}$, $\lambda_{w}, \lambda_{p} \ge 0$}
            \State{Initialize $\bm{P} \in \real^{d\times k}$, $\bm{w} \in \real^d$;}
            \State{$t \leftarrow 1;$}
            \State{$\alpha_w, \alpha_p, \alpha_{w,j}, \alpha_{p, j} \leftarrow 1$ for all $j \in [d]$;}\Comment{For lazy update}
            \While{not convergence}
                \State{Sample $(\bm{x}^t, y^t) \in \{(\bm{x}_1,y_1), \ldots, (\bm{x}_N, y_N)\}$;}
                \State{$w_j \leftarrow \alpha_w / \alpha_{w, j} w_j, \ \bm{p}_j \leftarrow \alpha_p / \alpha_{p, j} \bm{p}_j$ for all $j \in \supp(\bm{x}^t)$;}\Comment{Lazy update for regularization term}
                \State{Compute the prediction: $f^t \leftarrow f_{\mathrm{FM}}(\bm{x}^t;\bm{w}, \bm{P})$;}
                \State{$f'_{j, s} \leftarrow x^t_{j}[\inner{\bm{x}^t}{\bm{p}_{:, s}} - x^t_{j}p_{j,s}]$ for all $j \in \supp(\bm{x}^t)$, $s \in [k]$;}
                \State{$L'_{w,j}\leftarrow \ell'(f^t, y^t) \cdot x^t_j + 2\lambda_w w_j$ for all $j \in \supp(\bm{x}^t)$;}
                \State{$L'_{p, j, s} \leftarrow \ell'(f^t, y^t) \cdot f'_{j, s} + 2\lambda_p p_{j, s}$ for all $j \in \supp(\bm{x}^t)$, $s \in [k]$;}
                \State{Compute step size parameter $\eta^{t}_w, \eta^{t}_p$;}
                \State{$w_j \leftarrow w_j - \eta^t_w L'_{w, j}$ for all $j \in \supp(\bm{x}^t)$;}
                \State{$p_{j, s} \leftarrow p_{j, s} - \eta^t_p L'_{p, j, s}$; for all $j \in \supp(\bm{x}^t)$, $s \in [k]$;}
                \State{$\alpha_w \leftarrow (1-2\eta\lambda_w)\alpha_w, \ \alpha_p \leftarrow (1-2\eta\lambda_p)\alpha_p$;}
                \State{$\alpha_{w, j} \leftarrow \alpha_w,\ \alpha_{p, j} \leftarrow \alpha_p$ for all $j \in \supp(\bm{x}^t)$;}
                \State{$t \leftarrow t+1$;}
            \EndWhile
             \State{$w_j \leftarrow \alpha_w / \alpha_{w, j} w_j,\ \bm{p}_j \leftarrow \alpha_p / \alpha_{p, j} \bm{p}_j$ for all $j \in [d]$;}\Comment{Finalize}
            \Output{Learned $\bm{P}$ and $\bm{w}$}
        \end{algorithmic}
    \end{algorithm}

    \paragraph{PBCD Algorithm for Sparse FMs.}
    \Cref{alg:bcd_sfm} shows the PBCD algorithm for sparse FMs with/without the line search~\citep{tseng2009coordinate}, where $\Omega: \real^{d \times k}$ is a sparsity-inducing regularizer.
    The operator $\prox_{\Omega}(\cdot, j)$ in line 18 is the proximal operator for only $j$-th row vector.
    The PBCD algorithm for such sparse FMs updates not $p_{j, s}$ but $\bm{p}_j$ at each iteration and it takes $O(\mathrm{nnz}(\bm{x}_{:, j})k)$ computational cost (if the proximal operator can be evaluated in $O(k)$).
    In this algorithm, we show the two variants for choosing step size $\eta$: using the line search method proposed by~\citet{tseng2009coordinate} and using the Lipschitz constant of the gradient.
    $\sigma, \rho \in (0, 1)$ are hyperparameters for the line search.
    In our experiments, we used the BCD algorithm without the line search.
    The additional space complexity for caches is $O(Nk)$ (caching $\bm{a}_n \in \real^k$ for all $n \in [N]$ requires $O(Nk)$ and caching $\bm{q}, \bm{\delta}, \bm{d} \in \real^{k}$ requires $O(k)$).
    Some sparse regularizers (e.g, CS regularizer) require some additional caches and how to compute/use caches depends on the regularizer.
    Moreover, in the case of using the line search, some regularizers might also require a non-obvious efficient incremental evaluation method (i.e., na\"ive computation of $\Omega(\cdot)$ might take a high computational cost that is prohibitive at each line search iteration).
    
    \begin{algorithm}[t!]
        \caption{PBCD algorithm for sparse FMs}
        \label{alg:bcd_sfm}
        \begin{algorithmic}[1]
            \Input{$\{(\bm{x}_n, \bm{y}_n) \}_{n=1}^{N}$, $k \in \natu_{>0}$, $\lambda_{w}, \lambda_{p}, \tilde{\lambda}_{p} \ge 0$, optional: $\sigma, \rho \in (0, 1)$ (for line search)}
            \State{Initialize $\bm{P} \in \real^{d\times k}$, $\bm{w} \in \real^d$;}
            \State{Compute predictions: $f_{n} \leftarrow f_{\mathrm{FM}}(\bm{x}_n;\bm{w}, \bm{P})$ for all $n \in [N]$;}
            \While{not convergence}
                \State{Optimize $\bm{w}$ and update caches as in canonical FMs;}
                \State{$\bm{a}_{n} \leftarrow \bm{P}^\top \bm{x}_{n} \in \real^{s}$ for all $n \in [N]$;}\Comment{Cache for update $\bm{P}$}
                \If{Perform line search}
                    \State{$L \leftarrow \sum_{n=1}^N \ell(f_n, y_n)/N + \lambda_{p} \norm{\bm{P}}_2^2 + \tilde{\lambda}_{p} \Omega(\bm{P})$;}\Comment{Objective value used in line search}
                \EndIf
                \For{$j=1, \ldots, d$}
                    \State{$\bm{f}'_n \leftarrow x_{n, j}(\bm{a}_n - x_{n, j}\bm{p}_{j})$ for all $n \in \supp(\bm{x}_{:, j})$;}
                    \State{$\nabla_{\bm{p}_{j}}\ell_n \leftarrow \ell'(f_n, y_n) \bm{f}'_n$ for all $n \in \supp(\bm{x}_{:, j})$};
                    \State{$\bm{d} \leftarrow (\sum_{n \in \supp(\bm{x}_{:, j})}\nabla_{\bm{p}_{j}}\ell_n)/N + 2\lambda_{p} \bm{p}_{j}$;}
                    \If{Perform line search}
                        \State{$\eta^{-1} \leftarrow \max_{s \in [k]} \{2\lambda_{p} + \sum_{n \in \supp(\bm{x}_{:, j})}\ell''(f_n, y_n)(\partial f_{\mathrm{FM}}(\bm{x}_n)/\partial p_{j,s})^2/N\}$;}
                    \Else
                        \State{$\eta^{-1} \leftarrow 2\lambda_p + \mu \sum_{n \in \supp (\bm{x}_{:, j})} \norm{\nabla_{\bm{p}_j} \ell_n}_2^2 / N$;}
                    \EndIf
                    \State{$\bm{q} \leftarrow \prox_{\tilde{\lambda}_{p}\eta \Omega}(\bm{P} - \eta \bm{e}^d_j \bm{d}^\top; j)$;\Comment{Apply for only $j$-th row vector}}
                    \State{$\bm{\delta} \leftarrow \bm{p}_{j} - \bm{q}$;}
                    \If{Perform line search}
                        \State{$L_{\mathrm{new}}\leftarrow L + \sum_{n \in \supp(\bm{x}_{:, j})}[\ell(f_n-\inner{\bm{f}'_n}{\bm{\delta}}, y_n)-\ell(f_n, y_n)]/N$;}
                        \State{$L_{\mathrm{new}}\leftarrow L_{\mathrm{new}} + \lambda_{p} (\norm{\bm{q}}_2^2 - \norm{\bm{p}_j}_2^2) + \tilde{\lambda}_{p}\left[\Omega(\bm{P} - \bm{e}^d_{j} \bm{\delta}^\top) - \Omega(\bm{P})\right]$;}
                        \State{$\alpha \leftarrow 1$;}
                        \While{not $L_{\mathrm{new}} - L \le \sigma \alpha \{-\inner{\bm{d}}{\bm{\delta}} + \tilde{\lambda}_p[\Omega({\bm{P} -\bm{e}^d_j \bm{\delta}^\top})-\Omega(\bm{P})]\}$}
                            \State{$L_{\mathrm{new}}\leftarrow L_{\mathrm{new}} + \sum_{n \in \supp(\bm{x}_{:, j})}\left[\ell(f_n-\inner{\bm{f}'_n}{\alpha\rho\bm{\delta}}, y_n) - \ell(f_n-\inner{\bm{f}'_n}{\alpha\bm{\delta}}, y_n)\right]/N$;}                    \State{$L_{\mathrm{new}}\leftarrow L_{\mathrm{new}} + \lambda_{p} \left[\norm{\bm{p}_j-\alpha\rho\bm{\delta}}_2^2 - \norm{\bm{p}_j-\alpha\bm{\delta}}_2^2\right]$;}
                            \State{$L_{\mathrm{new}}\leftarrow L_{\mathrm{new}} + \tilde{\lambda}_{p}\left[\Omega(\bm{P} - \alpha \rho \bm{e}^d_{j} \bm{\delta}^\top) - \Omega(\bm{P} - \alpha \bm{e}^d_{j} \bm{\delta}^\top)\right]$;}
                            \State{$\alpha \leftarrow \alpha \rho$;}
                        \EndWhile
                    \State{$L \leftarrow L_{\mathrm{new}}$;}
                    \State{$\bm{\delta} \leftarrow \alpha \bm{\delta}$;}
                    \EndIf
                    \State{$\bm{p}_{j} \leftarrow \bm{p}_j - \bm{\delta}$;}
                    \State{$\bm{a}_{n} \leftarrow \bm{a}_n - x_{n,j}\bm{\delta}$ for all $n \in \supp(\bm{x}_{:, j})$;}
                    \State{$f_{n} \leftarrow f_n -  \inner{\bm{f}'_n}{\bm{\delta}}$ for all $n \in \supp(\bm{x}_{:, j})$;}
                \EndFor
            \EndWhile
            \Output{Learned $\bm{P}$ and $\bm{w}$}
        \end{algorithmic}
    \end{algorithm}
    
    \paragraph{PSGD Algorithm for Sparse FMs.}
    The extension of~\cref{alg:sgd_fm} to PSGD algorithm for sparse FMs is straightforward: evaluates a proximal operator for a sparse regularization after updating $\bm{P}$ at each iteration.
    For $\tilde{\ell}_{1,2}^2$-sparse FMs (TI-sparse FMs), all parameters must be the latest values at each iteration, i.e., a lazy update technique cannot be used.

    \paragraph{Extension to HOFMs.}
    Here we describe the extension of above described algorithms for $M$-order HOFMs.
    The extended algorithms update $\bm{w}$, $\bm{P}^{(2)}, \ldots, \bm{P}^{(M)}$ sequentially.
    HOFMs are also multi-linear w.r.t $\bm{w}, \bm{p}_{1}^{(2)}, \ldots, \bm{p}_{d}^{(M)}$ (and clearly $w_{0}, \ldots, w_{d}, p_{1,1}^{(2)}, \ldots, p_{d, k}^{(M)}$)~\citep{blondel2016higher} and thus the output of HOFMs are written as $f_{\mathrm{HOFM}}^M(\bm{x}) = \inner{\bm{\theta}}{\nabla_{\bm{\theta}}f_{\mathrm{HOFM}}^M(\bm{x})} + \mathrm{const}$ for a parameter $\bm{\theta} \in \{\bm{w}, \bm{p}_1^{(2)}, \ldots, \bm{p}_d^{(M)}\}$.
    Given $\nabla_{\bm{\theta}} f_{\mathrm{HOFM}}^M(\bm{x})$, the parameter $\bm{\theta}$ and predictions can be updated efficiently in both CD (updates only one element in $\bm{\theta}$) and BCD algorithms.
    For each case, $\partial f_{\mathrm{HOFM}}^M(\bm{x})/\partial \theta_j$ is written as
    \begin{align}
        \frac{\partial f_{\mathrm{HOFM}}^M(\bm{x})}{\partial \theta_j} = \begin{cases}
        x_j & \text{if } \theta_j = w_j, \ j \in [d],\\
        x_j\anova^{m-1}(\bm{x}_{\lnot j}, (\bm{p}_{:, s})_{\lnot j}) & \text{if } \theta_j = p_{j, s}^{(m)},\ m \in \{2, \ldots, M\},\ j \in [k],
        \end{cases}
    \end{align}
    where $\bm{x}_{\lnot j} \in \real^{d-1}$ is the $(d-1)$-dimensional vector with $x_j$ removed.
    Thus, the replacement of  $f'_{n, s}$ in~\cref{alg:cd_fm} ($\bm{f}'_{n}$ in~\cref{alg:bcd_sfm}) with $\partial f_{\mathrm{HOFM}}^m(\bm{x}_n)/\partial p_{j,s}^{(m)}$ ($\nabla_{\bm{p}_j^{(m)}} f_{\mathrm{HOFM}}^m(\bm{x}_n)$) produces the CD (BCD) algorithm for HOFMs.
    \Cref{alg:sgd_fm} can be also extended to the SGD algorithm for HOFMs similarly.
    Here, the issue is clearly how to compute $\partial f_{\mathrm{HOFM}}^M(\bm{x}_n)/\partial p_{j,s}^{(m)}$ efficiently.
    Fortunately,~\citet{blondel2016higher} and~\citet{atarashi2020link} proposed efficient computation algorithms for $\anova^{m-1}(\bm{x}_{\lnot j}, (\bm{p}_{:, s})_{\lnot j})$.
    For more detail, please see~\citep{blondel2016higher,atarashi2020link}.

\section{Additional Experiments}
\subsection{Efficiency Comparison on Real-world Dataset}
    \label{subsec:efficiency_real_world}
    \begin{figure*}[t]
        \centering
        \subfloat[ML100K Dataset with $\tilde{\lambda}_p=10^{-4}$ (left) and $10^{-5}$ (right).]{
        \includegraphics[width=80mm]{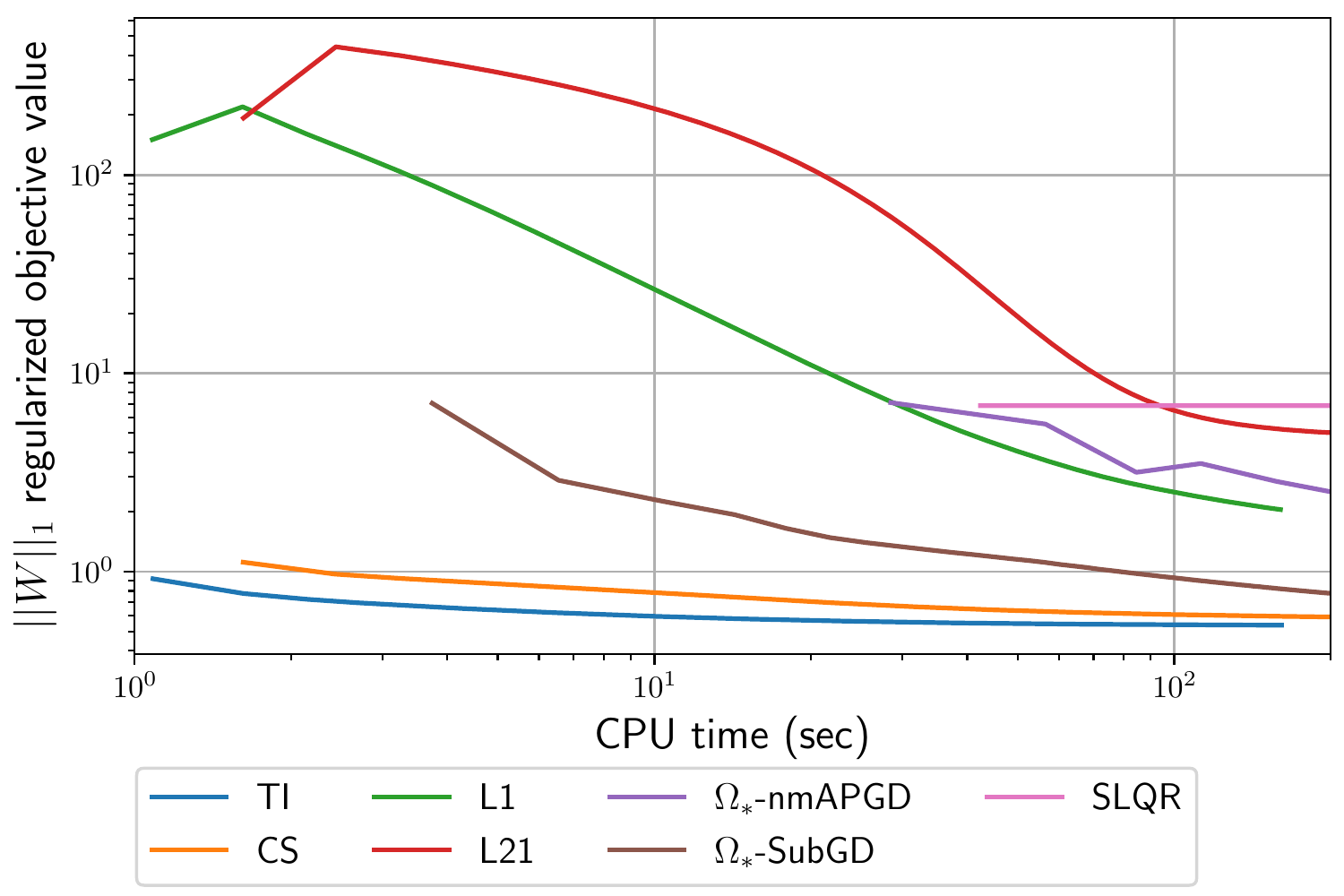}
        \includegraphics[width=80mm]{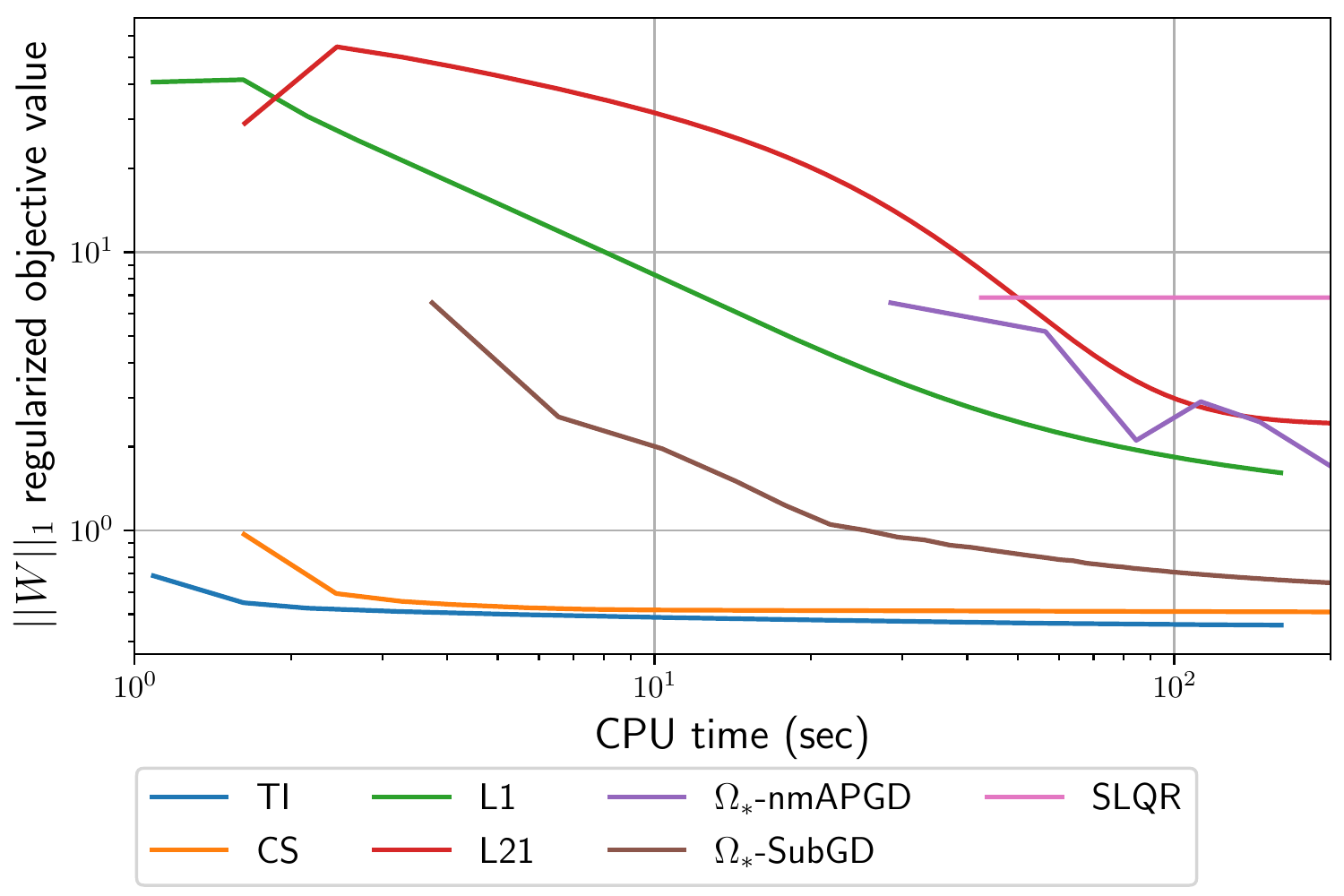}
        \label{fig:exact_ml100k}
        }\\
        \subfloat[a9a Dataset with $\tilde{\lambda}_p=10^{-4}$ (left) and $10^{-5}$ (right).]{
        \includegraphics[width=80mm]{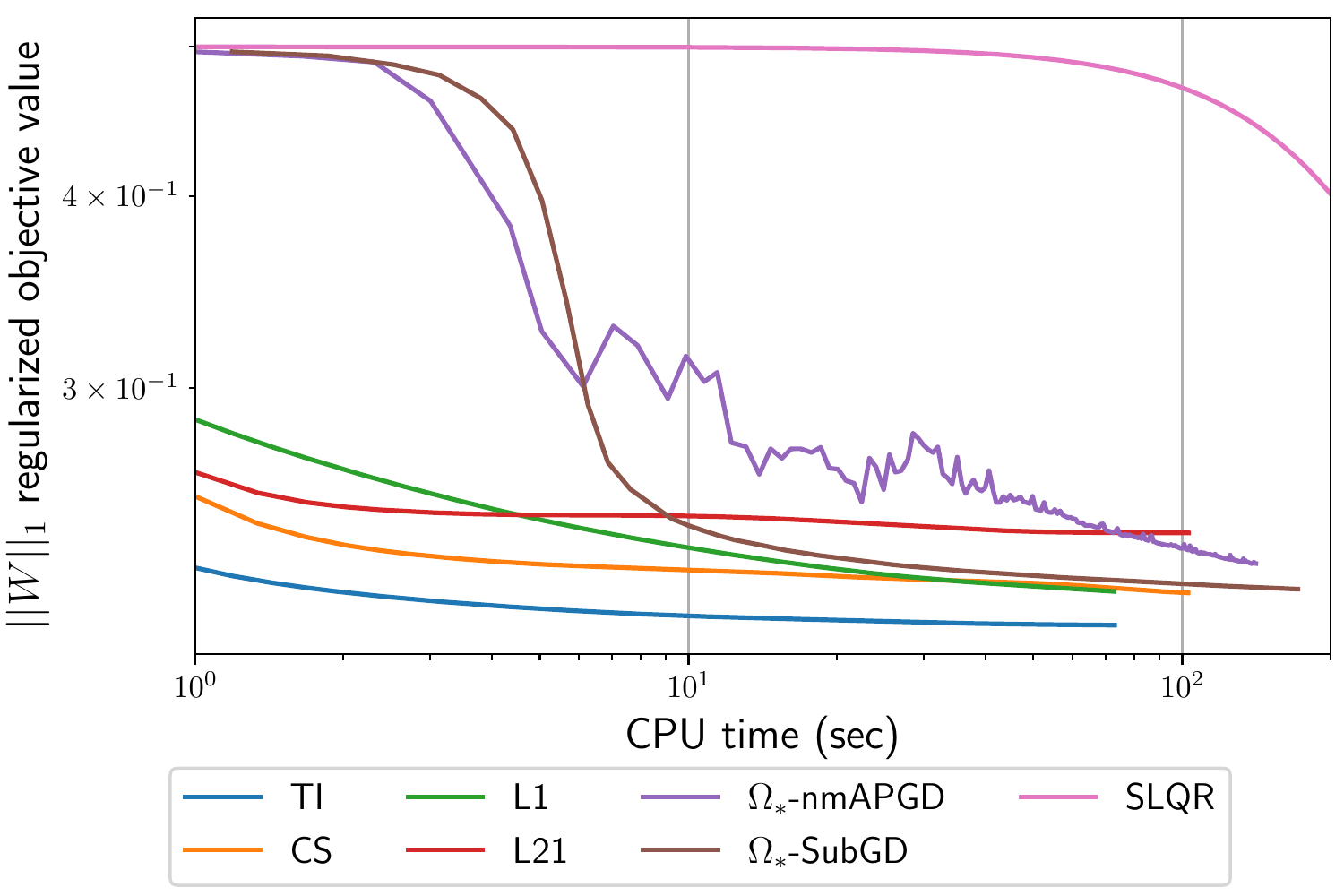}
        \includegraphics[width=80mm]{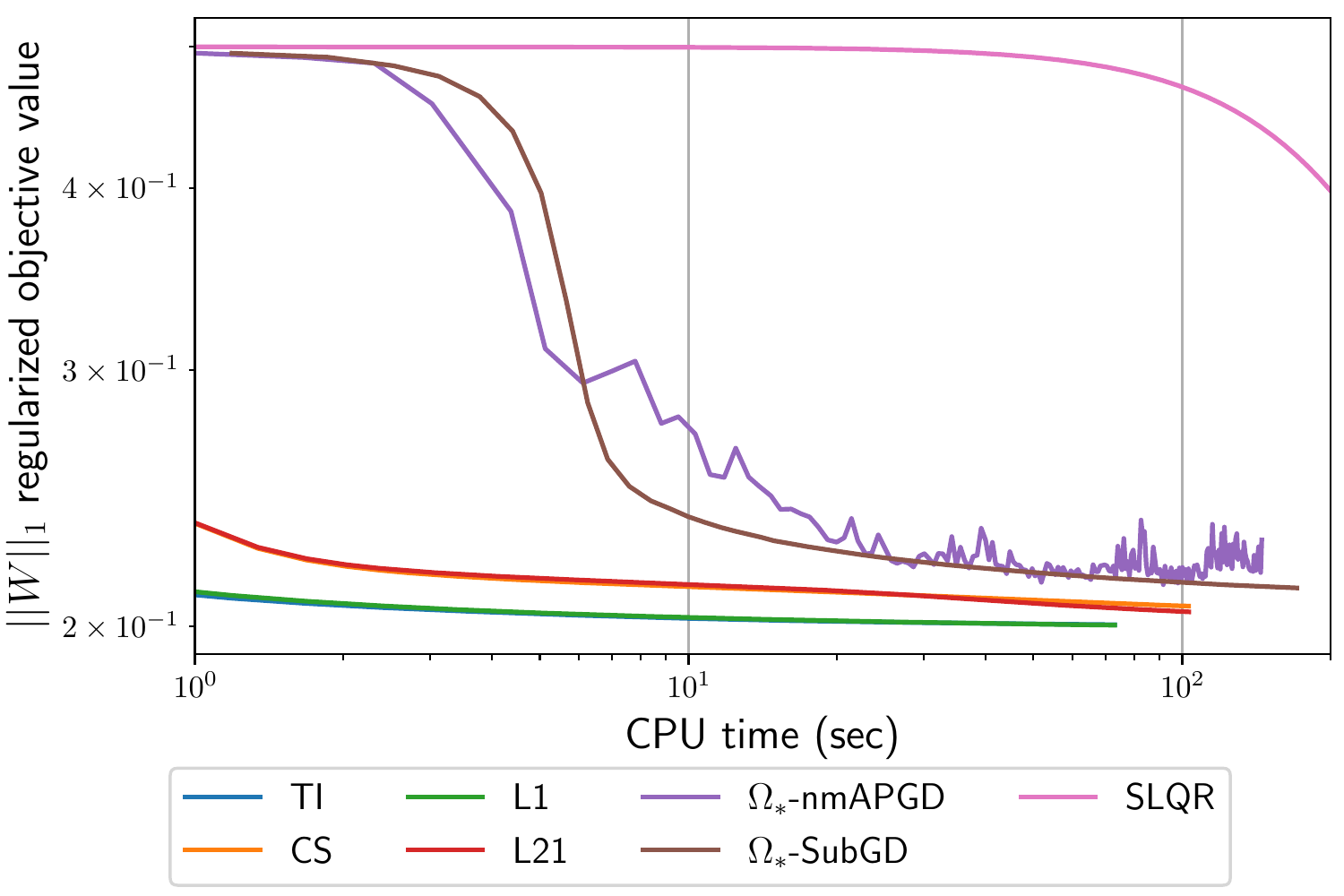}
        \label{fig:exact_a9a}
        }
        \caption{Trajectories of $\norm{\bm{W}}_1$ regularized objective value for \textbf{TI}, \textbf{CS}, \textbf{L21}, \textbf{L1}, \textbf{$\Omega_*$-nmAPGD}, \textbf{$\Omega_*$-SubGD}, and \textbf{SLQR} methods on (a) ML100K dataset and (b) a9a dataset.}
        \label{fig:exact_real}
    \end{figure*}

    We evaluated the efficiency of the proposed and existing methods on the ML100K dataset and the a9a dataset as in~\cref{subsubsec:efficiency}.
    We ran the experiment with $\tilde{\lambda}_p=10^{-4}$ and $10^{-5}$ on both datasets.
    We set $\lambda_p = \lambda_{\mathrm{tr}}=10^{-4}$ on the ML100K dataset and $\lambda_p=\lambda_{\mathrm{tr}}=10^{-3}$ on the a9a dataset.
    The other settings are the same those in~\cref{subsubsec:efficiency}.
    Because we compared the convergence speeds of objective values, we didn't separate datasets to training, validation, and testing datasets, i.e., we used $100,000$ and $48,642$ instances for training on the ML100K and the a9a datasets, respectively.
    
    As shown in~\cref{fig:exact_ml100k}, \textbf{TI} and \textbf{CS} converged much faster than the other methods in terms of $\Omega_*$ ($\norm{\bm{W}}_1$) regularized objective function on both the ML100K dataset and the a9a dataset.
    These results indicate that the proposed regularizers can be good alternative to $\Omega_*$ for not only synthetic datasets but also some real-world datasets.
    
\subsection{Optimization Methods Comparison}
    \label{subsec:solver}
    We compared the convergence speeds of the some algorithms for TI-sparse FMs and CS-sparse FMs on both synthetic and real-world datasets.
    We ran this experiment on an Arch Linux desktop with an Intel Core i7-4790 (3.60 GHz) CPU and 16 GB RAM.
    
    \paragraph{Methods Compared.}
    We compared the following algorithms for TI-sparse FMs and CS-sparse FMs:
    \begin{itemize}
        \item \textbf{PCD}: the proximal coordinate descent algorithm (only TI-sparse FMs).
        \item \textbf{PBCD}: the proximal block coordinate descent algorithm (only CS-sparse FMs).
        \item \textbf{APGD}: FISTA with restart.
        \item \textbf{nmAPGD}: the non-monotone APGD algorithm~\citep{li2015accelerated}.
        \item \textbf{PSGD}: the PSGD algorithm.
        \item \textbf{MB-PSGD}: the mini-batch PSGD algorithm (only on real-world datasets).
        \item \textbf{Katyusha}: the Katyusha algorithm proposed by~\citet{allen2017katyusha}. It is similar to an accelerated proximal stochastic variance reduction gradient algorithm~\citep{nitanda2014stochastic} but introduces an additional moment term, which is called Katyusha momentum.
        \item \textbf{MB-Katyusha}: the mini-batch Katyusha algorithm (only on real-world datasets).
    \end{itemize}
    We used line search techniques for the step size in \textbf{APGD} and \textbf{nmAPGD}~\citep{beck2009fast,li2015accelerated} but we did not use them in \textbf{PCD} and \textbf{PBCD} because their sub-problems are smooth.
    For the (mini-batch) PSGD-based algprithms, we ran the experiment using initial step size $\eta_{0}=1.0$, $0.1$ and $0.01$ but we show only the best results.
    In \textbf{PSGD} and \textbf{MB-PSGD}, rather than using a constant step size, we used a diminishing step size as proposed by~\citet{bottou2012stochastic}: $\eta = \eta_{0}(1+\eta_0 \lambda_p t)^{-1}$ at the $t$-th iteration.
    On the other hand, \textbf{Katyusha} and \textbf{MB-Katyusha} used a constant step size: $\eta=\eta_0$.
    In \textbf{MB-PSGD} and \textbf{MB-Katyusha}, we set the number of instances in one mini-batch $N_b$ to be $Nd / \nnz{\bm{X}}$, which reduces the computational cost per epoch from $O(NdK)$ to $O(\nnz{\bm{X}}k)$.
    For the evaluation of the proximal operator~\eqref{eq:prox_ti}, we used~\cref{alg:prox_ti_randomize}, which runs in $O(d)$ time in expectation.
    \begin{figure}[t!]
        \centering
        \subfloat[Feature interaction selection setting: $d_{\mathrm{true}}=80$, $b=8$ and $d_{\mathrm{noise}}=20$ with $N=200$ and $20,000$ for \textbf{TI} method.]{
        \includegraphics[width=80mm]{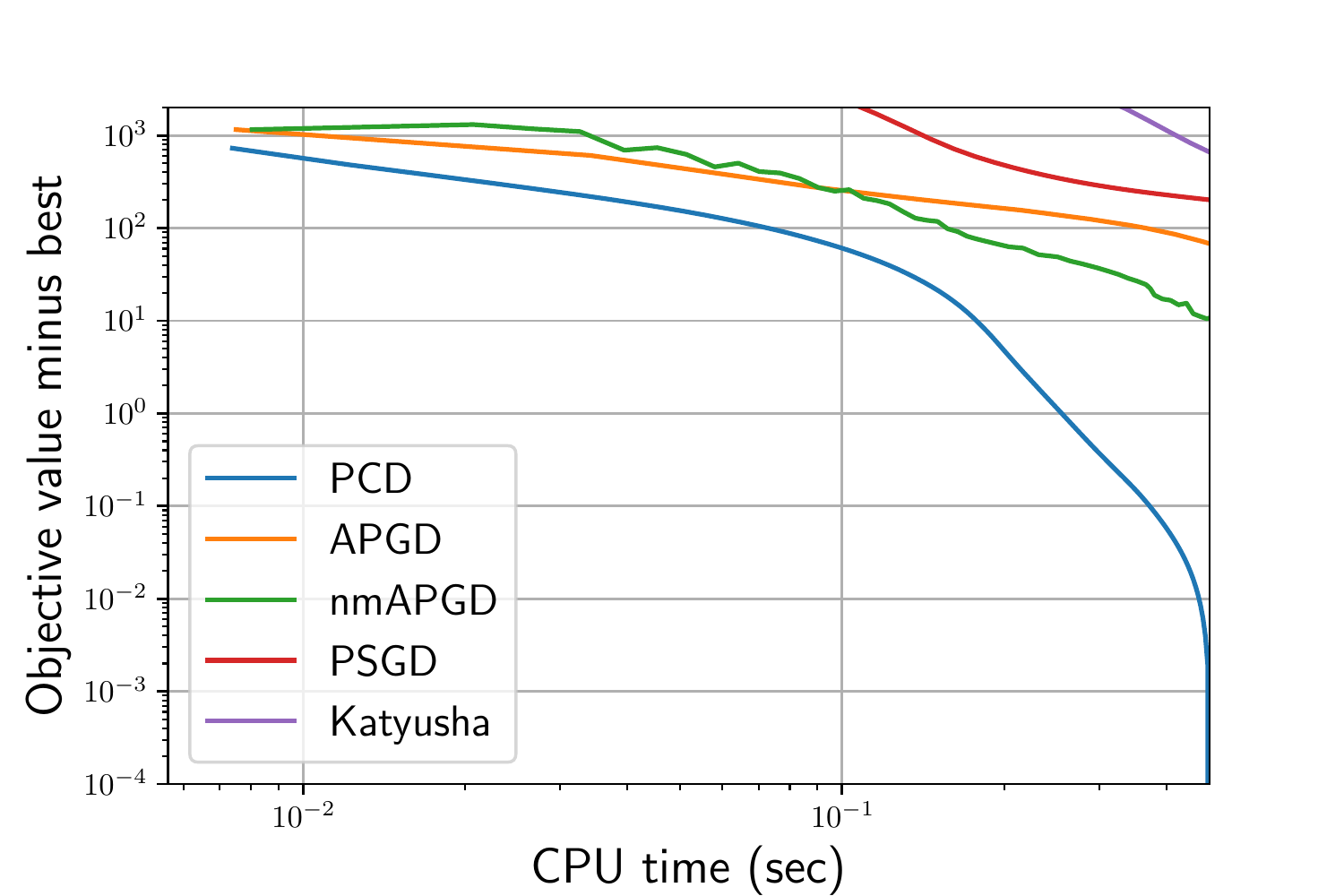} 
        \includegraphics[width=80mm]{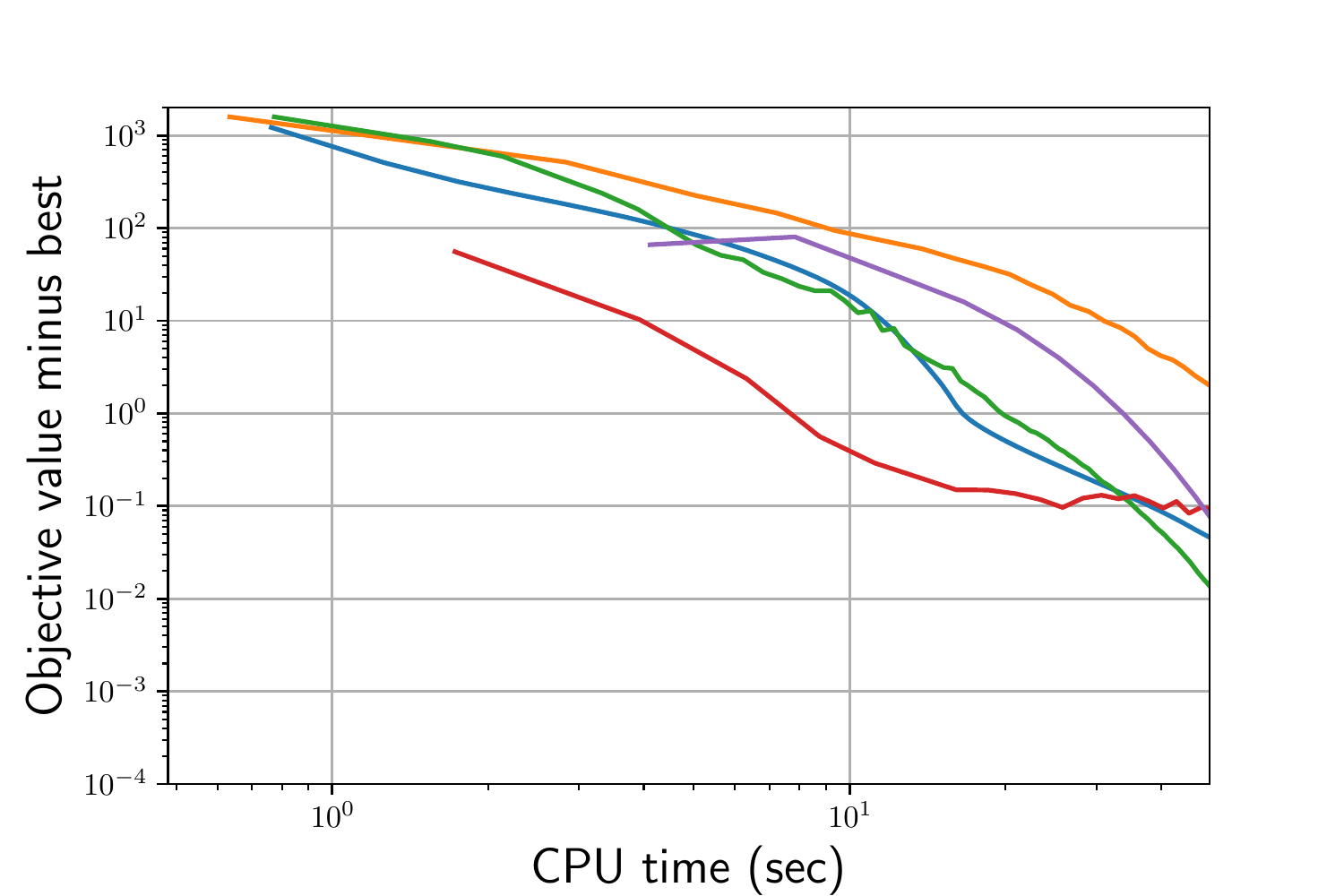}
        \label{fig:synthetic_solver_feature_interaction_selection}
        }\\
        \subfloat[Feature selection setting: $d_{\mathrm{true}}=20$, $b=1$ and $d_{\mathrm{noise}}=80$ with $N=200$ (left) and $20,000$ for \textbf{TI} method.]{
        \includegraphics[width=80mm]{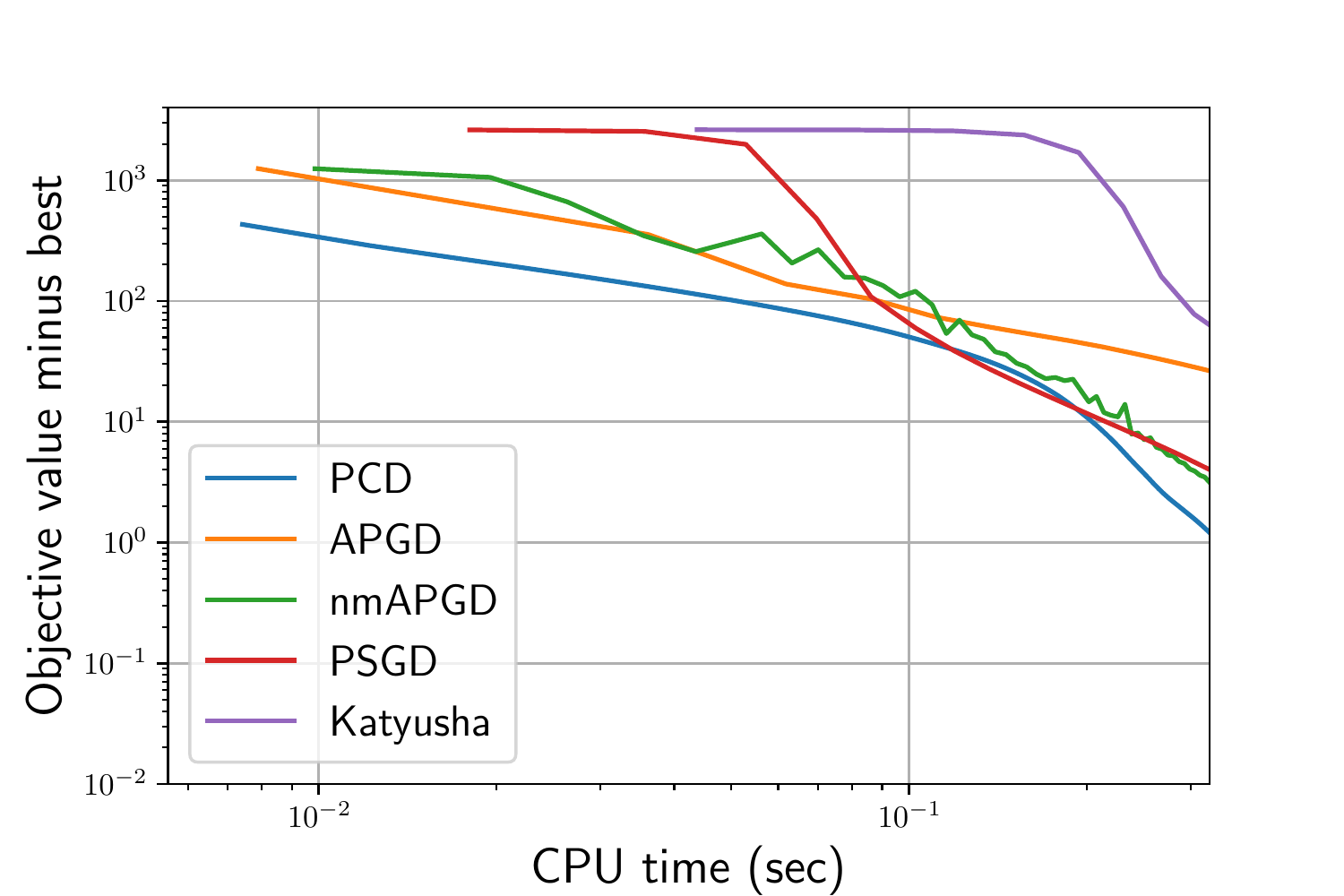} 
        \includegraphics[width=80mm]{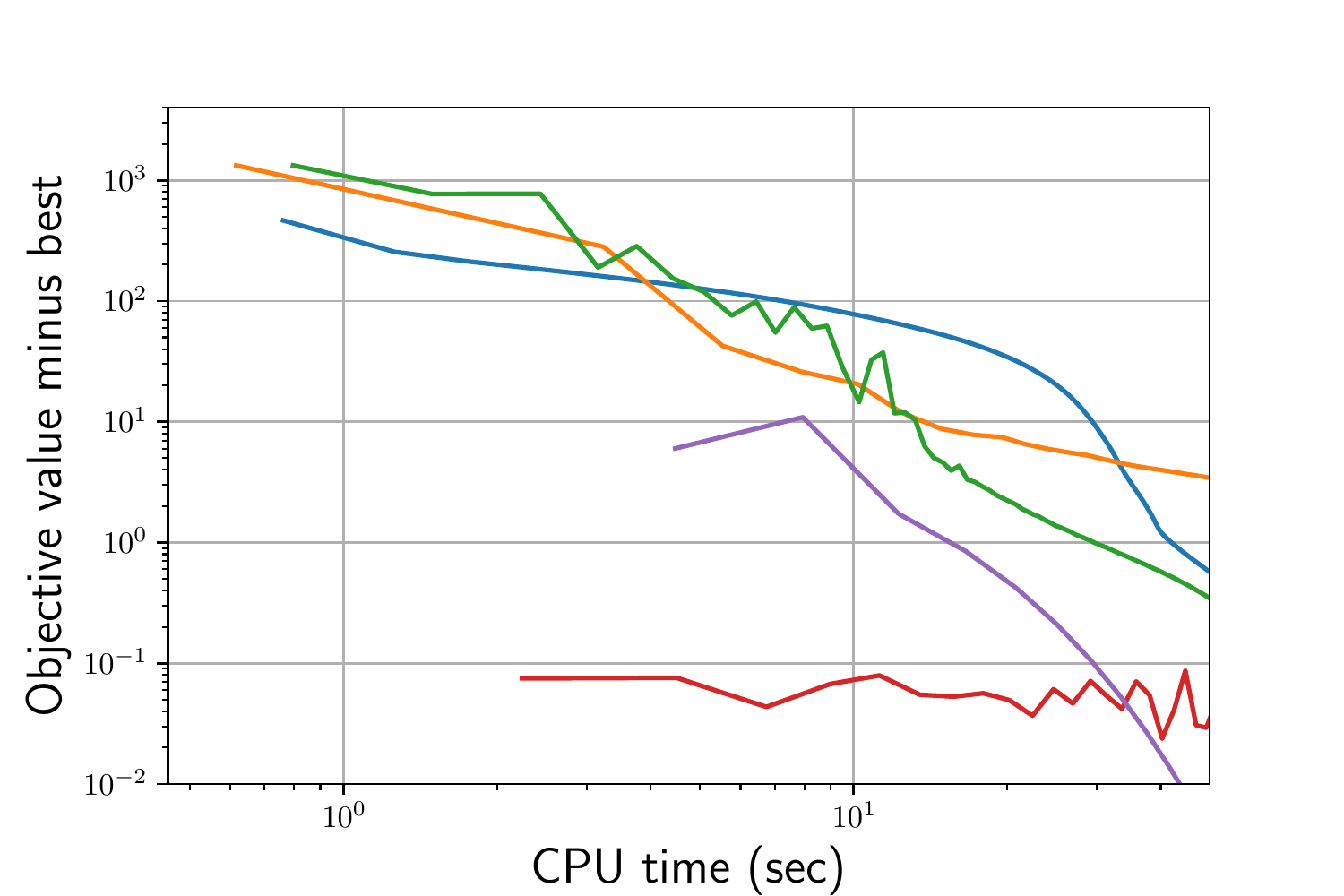}
        \label{fig:synthetic_solver_feature_selection}
        }\\
        \subfloat[Feature selection setting: $d_{\mathrm{true}}=20$, $b=1$ and $d_{\mathrm{noise}}=80$ with $N=200$ (left) and $20,000$ for \textbf{CS} method.]{
        \includegraphics[width=80mm]{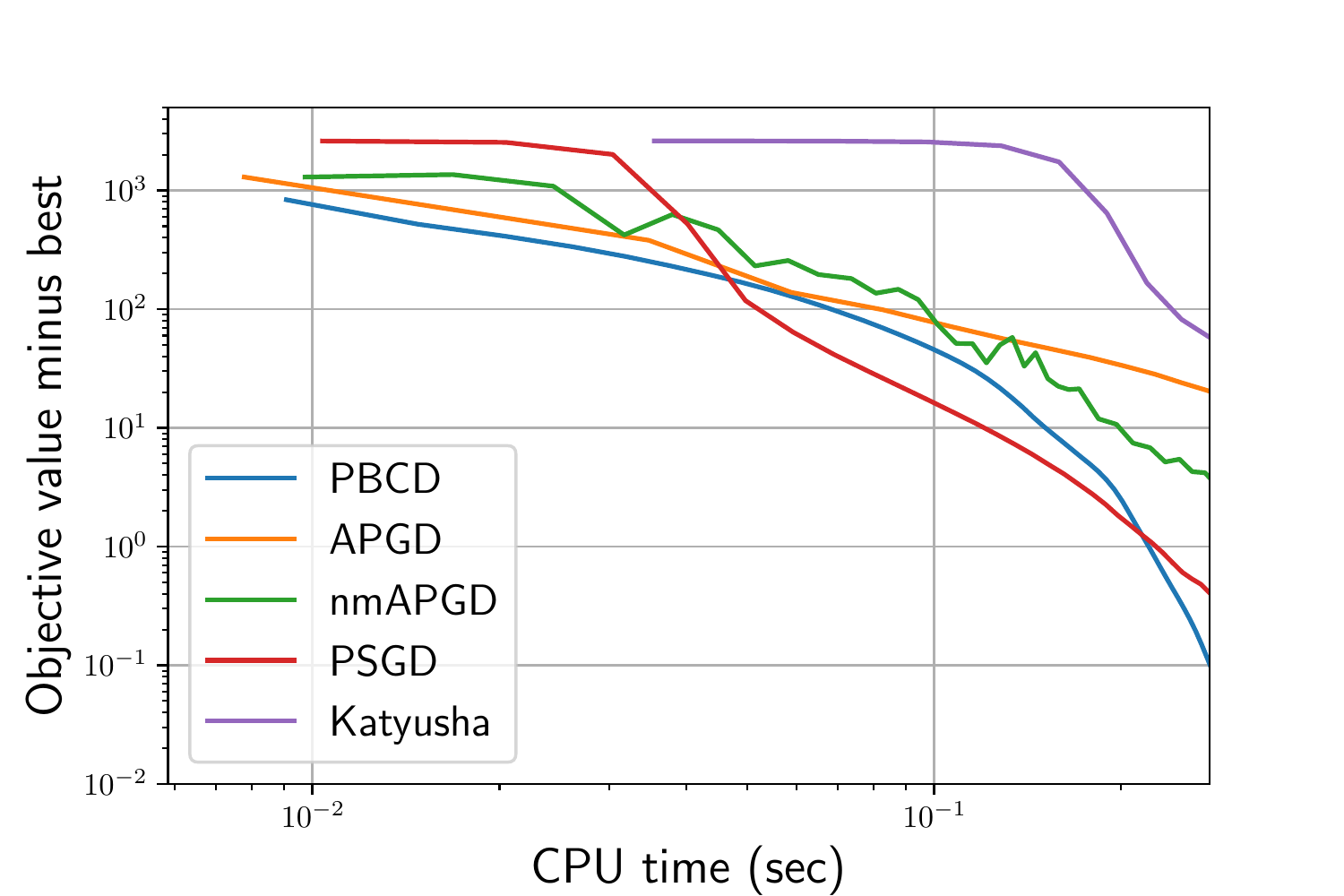} 
        \includegraphics[width=80mm]{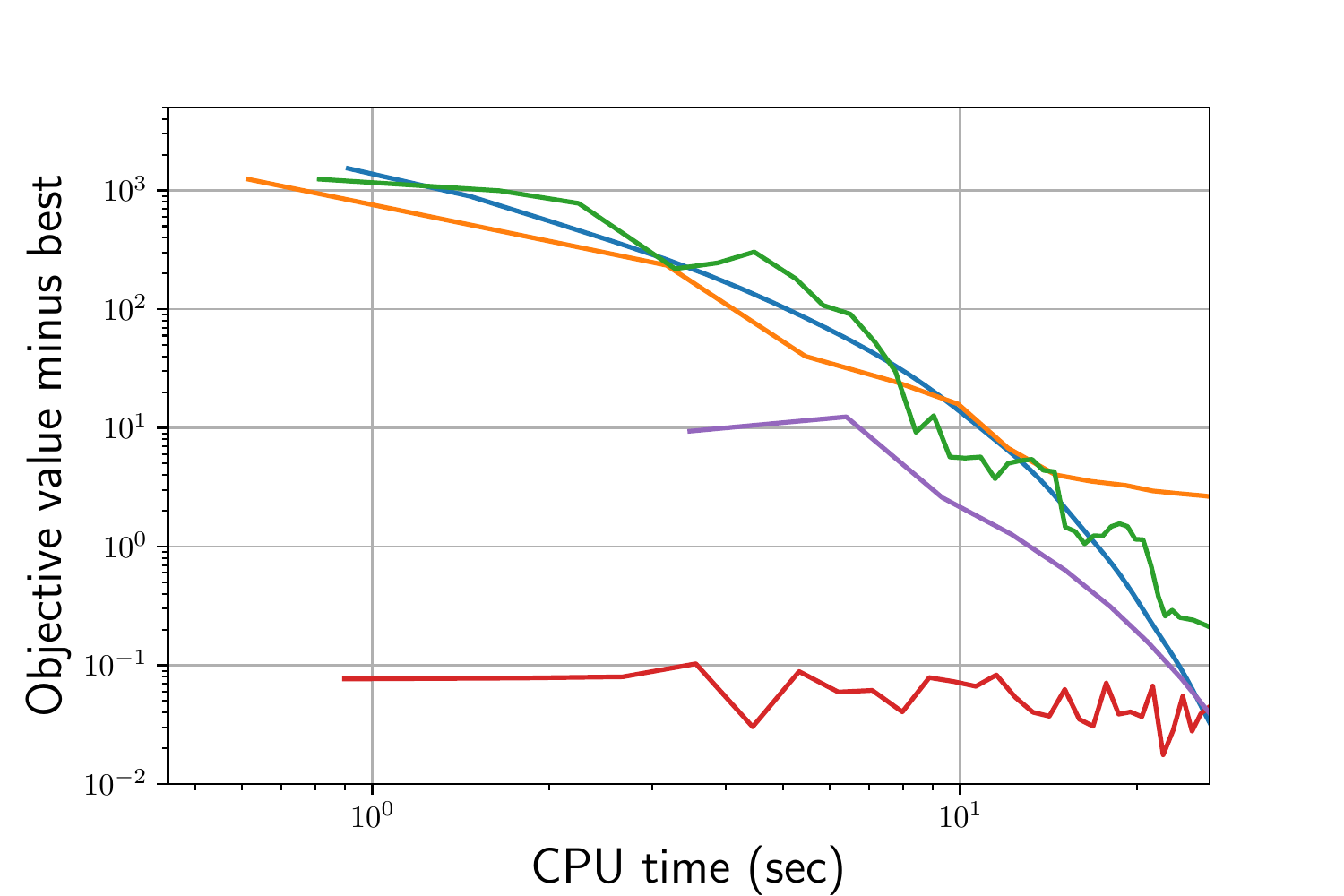}
        \label{fig:synthetic_solver_feature_selection_cs}
        }
        \caption{Runtime comparisons among \textbf{PCD}, \textbf{PBCD}, \textbf{APGD}, \textbf{nmAPGD}, \textbf{PSGD}, and \textbf{Katyusha} algorithms on synthetic datasets using different amounts of training data: (a) feature interaction selection setting datasets for \textbf{TI} method; (b) feature selection setting datasets for \textbf{TI} method; (c) feature selection setting datasets for \textbf{CS} method. Left and right graphs show results for datasets with $N=200$ and $20,000$, respectively.}
        \label{fig:synthetic_solver}
    \end{figure}
    
    \begin{figure}[t!]
        \centering
        \subfloat[ML100K dataset ($N=100,000$): \textbf{TI} method (left) and \textbf{CS} method (right) with $\lambda_w = 5\times 10^{-4}$, $\lambda_p = 5\times 10^{-5}$, and $\tilde{\lambda}_p = 5\times 10^{-5}$.]{
        \includegraphics[width=80mm]{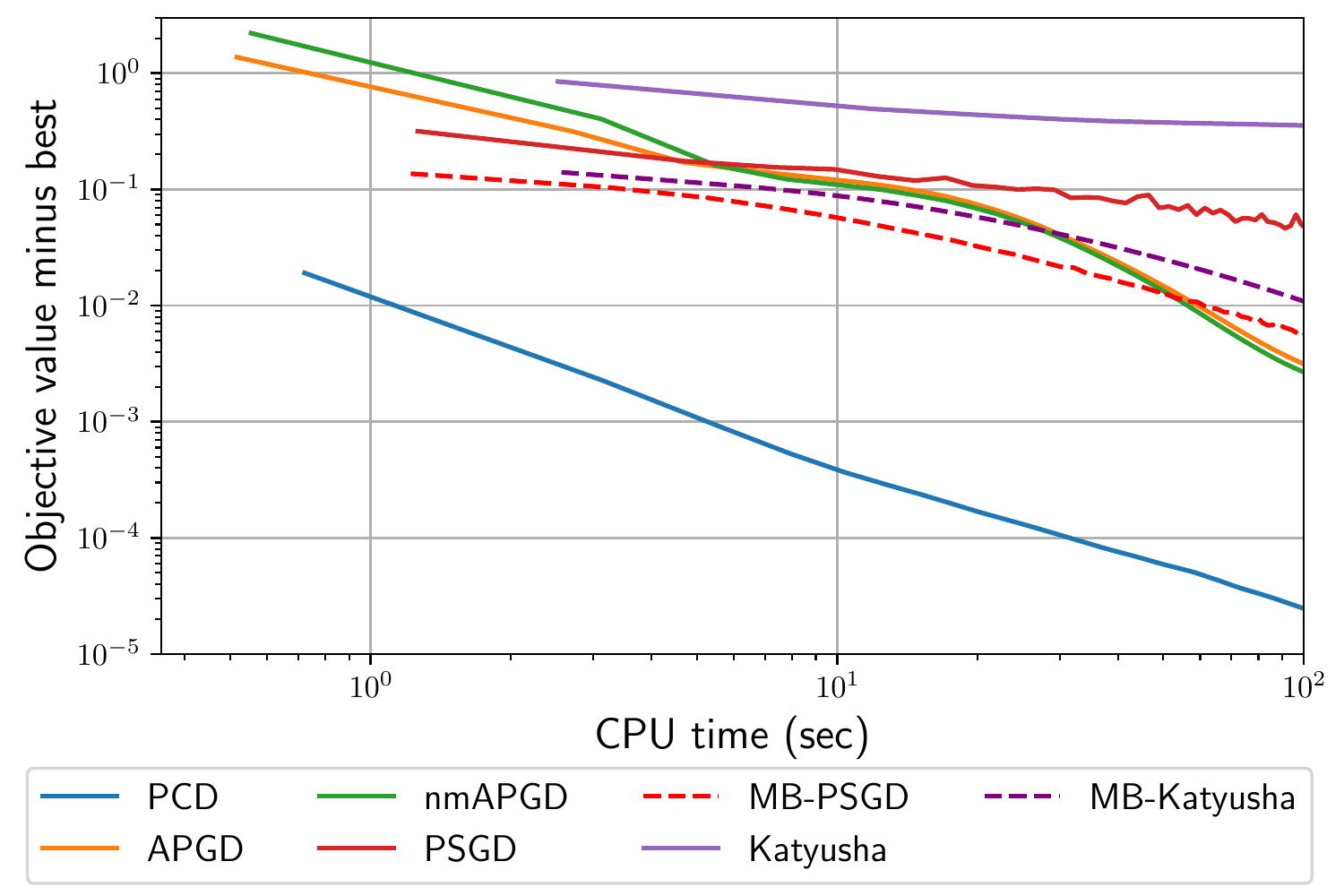} 
        \includegraphics[width=80mm]{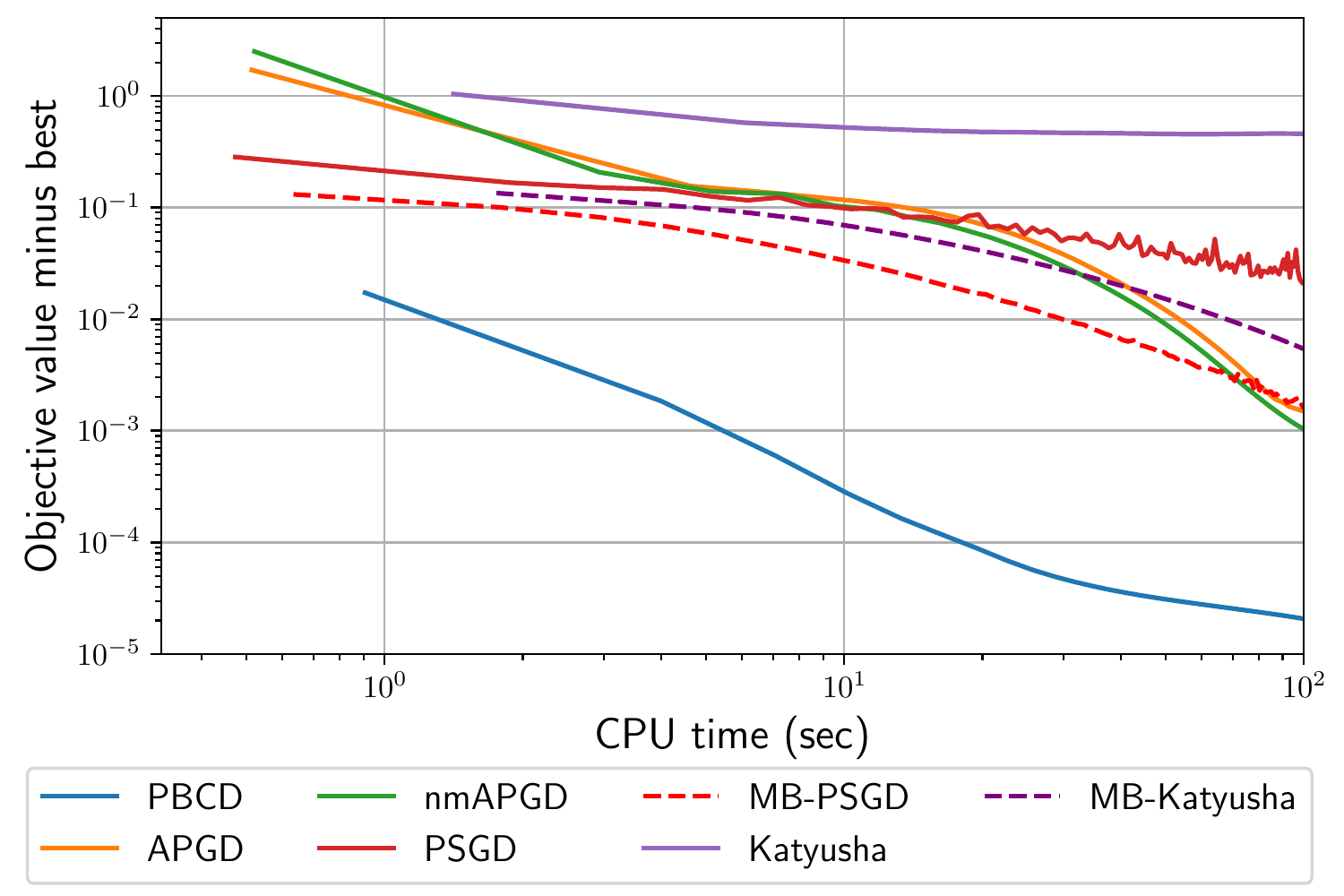}
        \label{fig:solver_ml100k}
        }\\
        \subfloat[a9a dataset ($N=48,842$): \textbf{TI} method (left) and \textbf{CS} method (right) with $\lambda_w = 5\times 10^{-2}$, $\lambda_p = 5\times 10^{-4}$, and $\tilde{\lambda}_p = 5\times 10^{-4}$.]{
        \includegraphics[width=80mm]{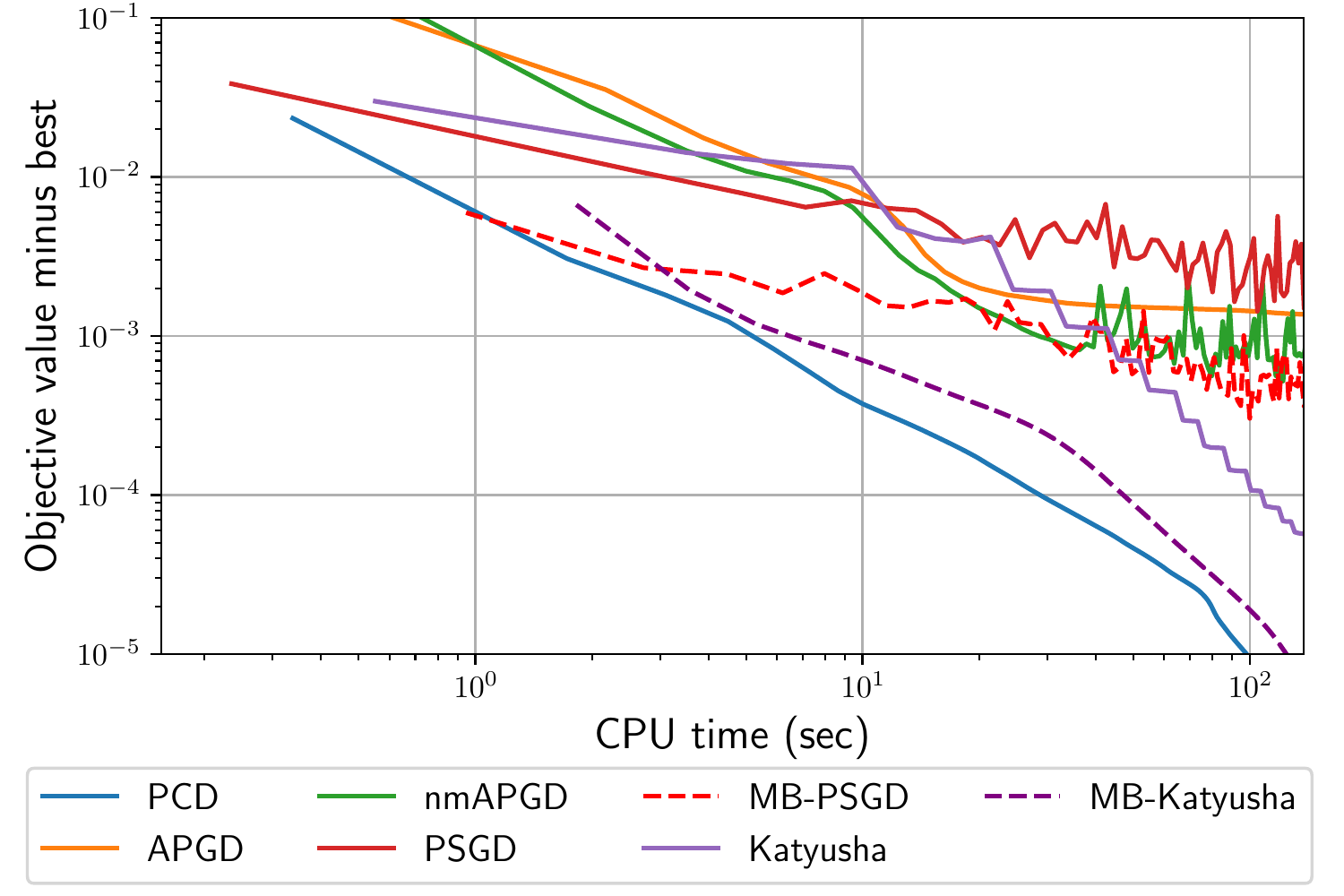} 
        \includegraphics[width=80mm]{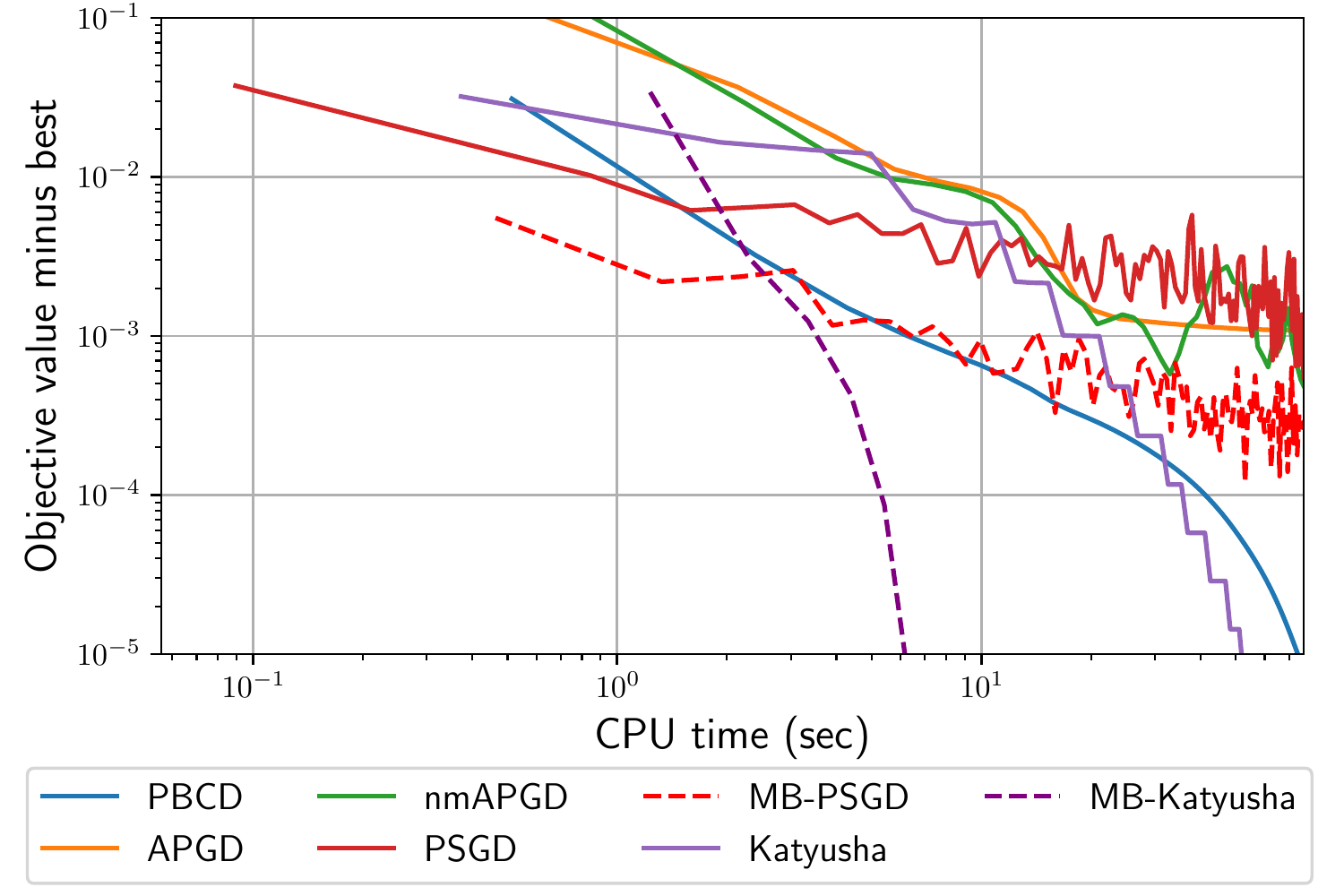}
        \label{fig:solver_a9a}
        }\\
        \caption{Runtime comparisons among \textbf{PCD}, \textbf{PBCD}, \textbf{APGD}, \textbf{nmAPGD}, \textbf{PSGD}, \textbf{MB-PSGD}, \textbf{Katyusha}, and \textbf{MB-Katyusha} algorithms on (a) ML100K dataset and (b) a9a dataset. Left and right graphs show results for \textbf{TI} and \textbf{CS} method, respectively.}
        \label{fig:solver_real}
    \end{figure}
    
    \paragraph{Datasets.}
    We used feature interaction selection setting datasets, feature selection setting datasets, the ML100K dataset, and the a9a dataset.
    We set the number of instances of synthetic datasets $200$ and $20,000$ in order to compare scalabilities of algorithms w.r.t $N$.
    As in~\cref{subsec:efficiency_real_world}, we used $100,000$ and $48,642$ instances for training on the ML100K and the a9a datasets, respectively.
    We ran the experiment ten times using different initial random seeds.
    On synthetic datasets, we set $\lambda_{p}=\tilde{\lambda_p}=0.1$ for the batch algorithms (\textbf{PCD}, \textbf{PBCD}, \textbf{APGD}, and \textbf{nmAPGD}).
    For the stochastic algorithms (\textbf{PSGD} and \textbf{Katyusha}), we first scaled the feature vectors and targets: we used $\bm{x}_n / \sqrt{d}$ and $y_n / d$ as feature vectors and targets and set $\lambda_p$ and $\tilde{\lambda}_p$ to $0.1/d^2$.
    Since we used the squared loss and $f_{\mathrm{FM}}(\bm{x}/\sqrt{d}; \bm{0}, \bm{P}) = f_{\mathrm{FM}}(\bm{x}; \bm{0}, \bm{P})/d$, the balance between the loss term and regularization term was the same as that in the batch algorithms.
    On the ML100K dataset, we set $\lambda_w=0.5 \times 10^{-4}$, $\lambda_p=0.5 \times 10^{-4}$, and $\tilde{\lambda}_p = 0.5 \times 10^{-4}$ for all methods.
    On the a9a dataset, we set $\lambda_w=0.5 \times 10^{-2}$, $\lambda_p=0.5 \times 10^{-3}$, and $\tilde{\lambda}_p = 0.5 \times 10^{-3}$ for all methods.
    They were chosen based on the results in~\cref{subsec:real}.
    
    \paragraph{Results: batch vs stochastic.}
    As shown in~\cref{fig:synthetic_solver}, on synthetic datasets, when the number of training instances was $20,000$, the stochastic algorithms (\textbf{PSGD} and \textbf{Katyusha}) were faster than the batch algorithms (\textbf{PCD}, \textbf{PBCD}, \textbf{APGD}, and \textbf{nmAPGD}).
    This indicates that stochastic algorithms can be more useful than batch algorithms for large-scale dense datasets, as described in~\cref{subsec:pgd_ti} and~\cref{subsec:pgd_cs}.
    However, the stochastic algorithms take $O(dk)$ time at each iteration even if a sampled feature vector is sparse.
    On synthetic datasets, the feature vectors were dense since they were generated from Gaussian distributions, so the stochastic algorithms might be relatively slower in some real-world applications.
    Indeed, as shown in~\cref{fig:solver_real}, on real-world datasets, the completely stochastic algorithms (\textbf{PSGD} and \textbf{Katyusha}) were not faster than batch algorithms although although $N \gg d$ on both datasets.
    The use of appropriate size mini-batch improved the convergence speed as our expected: \textbf{MB-PSGD} and \textbf{MB-Katyusha} were faster than their completely stochastic versions.
    Nevertheless, such mini-batch algorithms were slower than \textbf{PCD} and \textbf{PBCD} on the ML100K dataset and the \textbf{PCD} on the a9a dataset.
    Moreover, performances of (mini-batch) stochastic algorithms were sensitive w.r.t the choice of the step size hyperparameter and the objective values usually diverged with $\eta_0=1.0$.
    Thus, as our analysis in~\cref{subsec:pgd_ti}, the PSGD-based algorithms should be used only when $\nnz{\bm{X}}/d$ is large.
    Note that $\nnz{\bm{X}}/d$ of the ML100K and the a9a datasets in this experiment are $338$ and $5,506$, respectively (clearly, on synthetic datasets $\nnz{\bm{X}}/d = N$).

    \paragraph{Results: \textbf{PCD}/\textbf{PBCD} vs \textbf{APGD}/\textbf{nmAPGD}.}
    On synthetic datasets (\cref{fig:synthetic_solver}), \textbf{PCD}/\textbf{PBCD}, \textbf{APGD}, and \textbf{nmAPGD} tended to show similar results but \textbf{PCD} was much faster than \textbf{APGD} and \textbf{nmAPGD} on the feature interaction selecting dataset with $N=200$.
    On real-world datasets (\cref{fig:solver_real}), \textbf{PCD} and \textbf{PBCD} were much faster than \textbf{APGD} and \textbf{nmAPGD}.
    Strictly speaking, not PCD-based algorithms but PGD/PSGD-based algorithms should be used since \textbf{TI} (\textbf{CS}) regularizer is not seperable w.r.t each coordinate (each row vector) in $\bm{P}$.
    However, our results indicate that \textbf{PCD} and \textbf{PBCD} can work better than \textbf{nmAPGD} and \textbf{APGD} practically.
    Moreover, again, \textbf{PCD} and \textbf{PBCD} have some important practical advantages: (i) easy to implement, (ii) easy to extend to related models (as shown in~\cref{sec:extension}), and (iii) having few hyperparameters.
    
\subsection{Comparison of~\cref{alg:prox_ti_sort} and~\cref{alg:prox_ti_randomize}}
    We compared two algorithms for the proximal operator~\eqref{eq:prox_ti},~\cref{alg:prox_ti_sort} (\textbf{Sort}) and~\cref{alg:prox_ti_randomize} (\textbf{Random}) proposed by~\citet{filipe2011online}.
    For a $d$-dimensional vector, \textbf{Sort} and \textbf{Random} run in $O(d \log d)$ time and $O(d)$ time, respectively.
    We evaluated the runtime of two algorithms for a $d$-dimensional vector generated from a Gaussian distribution, $\mathcal{N}(\bm{0}, \sigma^2 \eye_{d})$ with varying $d \in \{2^{3}, 2^{4}, \ldots, 2^{14}\}$.
    We set $\sigma=1.0$ and $\sigma=10.0$ and the regularization strength $\lambda=0.001$, $\lambda=0.1$, and $\lambda=10.0$.
    We ran the experiment $100$ times with different initial random seeds and report the average runtimes.
    For \textbf{Sort}, we used Nim's standard \texttt{sort} procedure, which is an implementation of merge sort.
    
    \begin{figure}[t]
        \centering
        \subfloat[$\sigma=1.0$.]{
        \includegraphics[width=52mm]{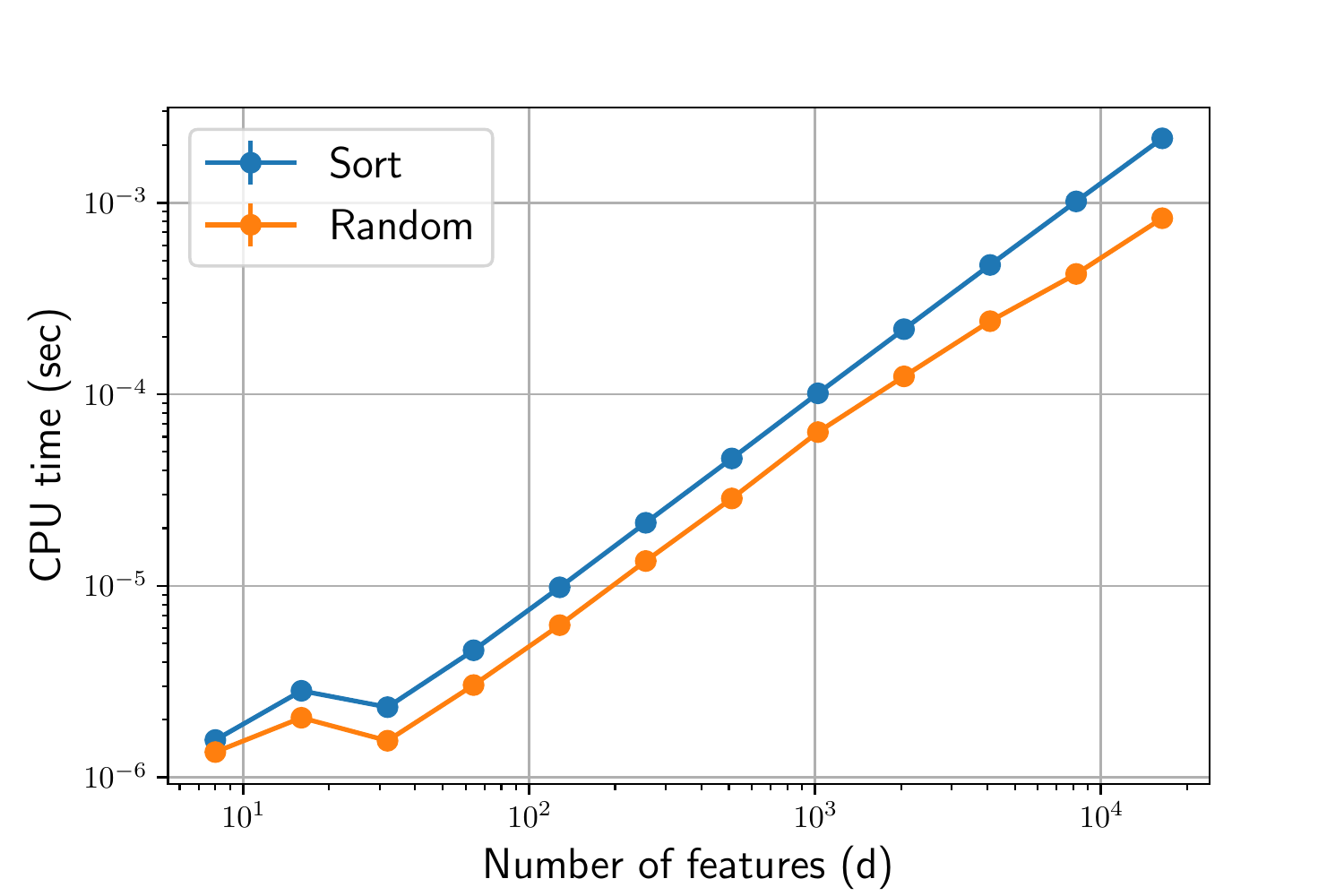}
        \includegraphics[width=52mm]{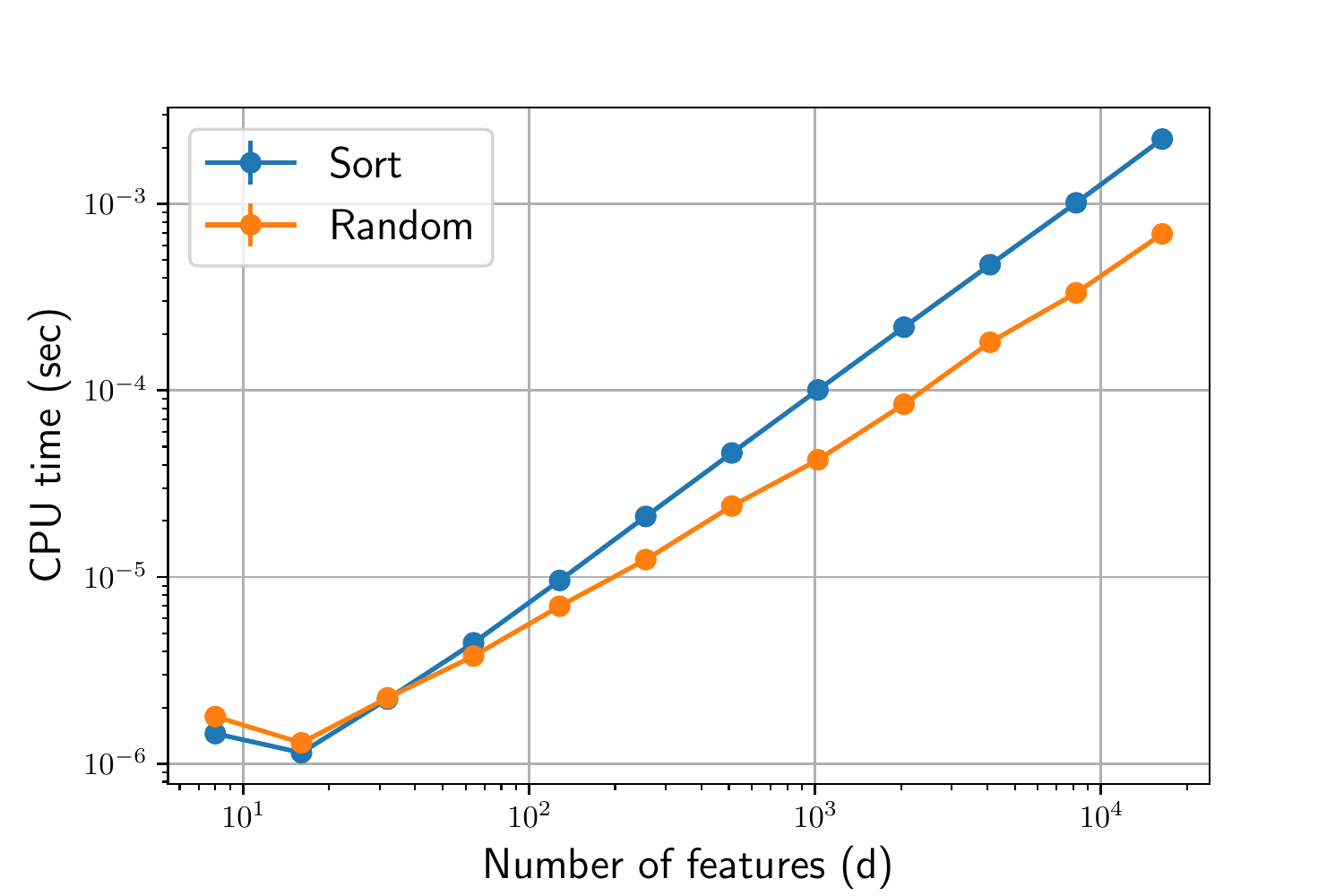}
        \includegraphics[width=52mm]{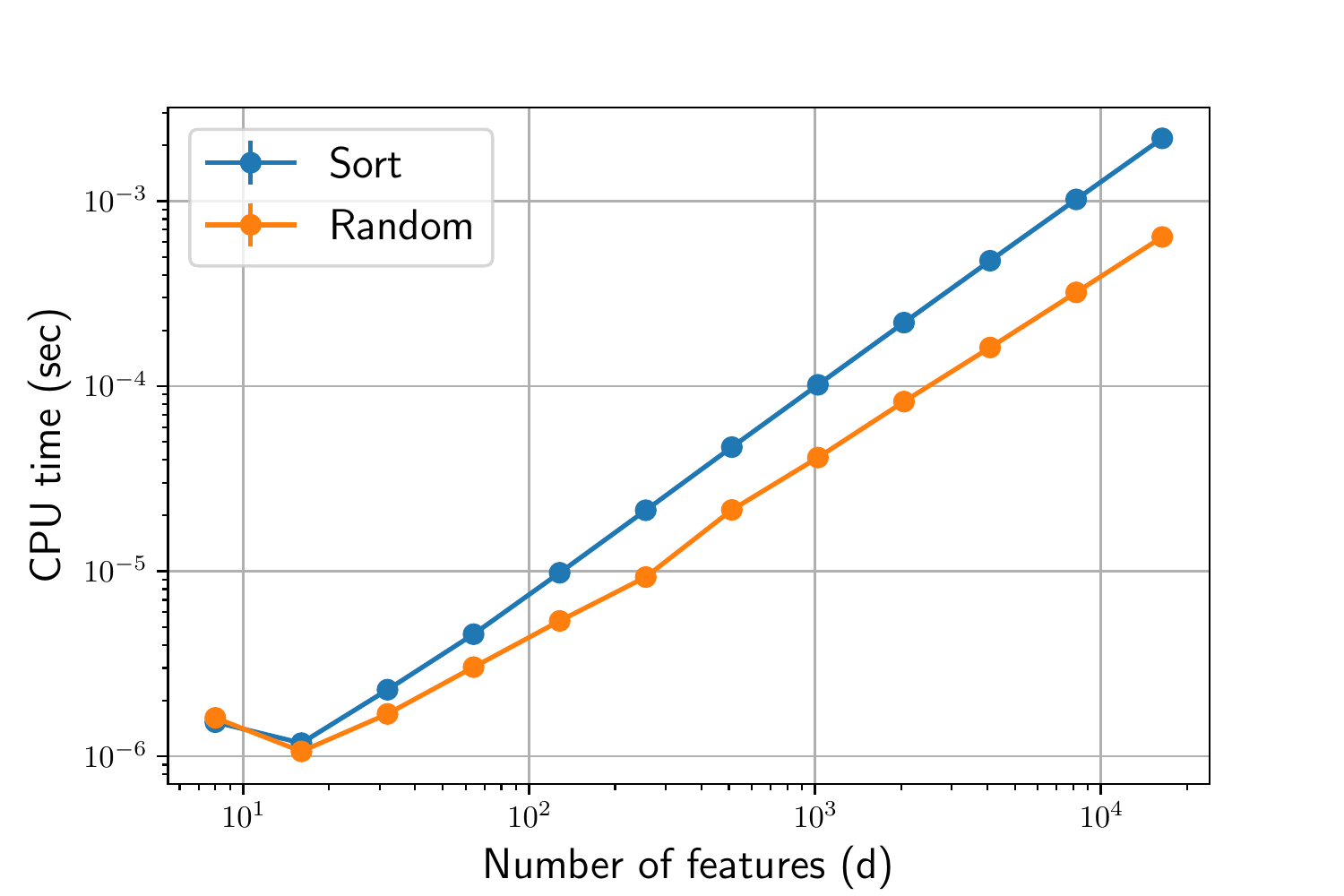}
        \label{fig:prox_sort_vs_random_std1.0}
        }\\
        \subfloat[$\sigma=10.0$.]{
        \includegraphics[width=52mm]{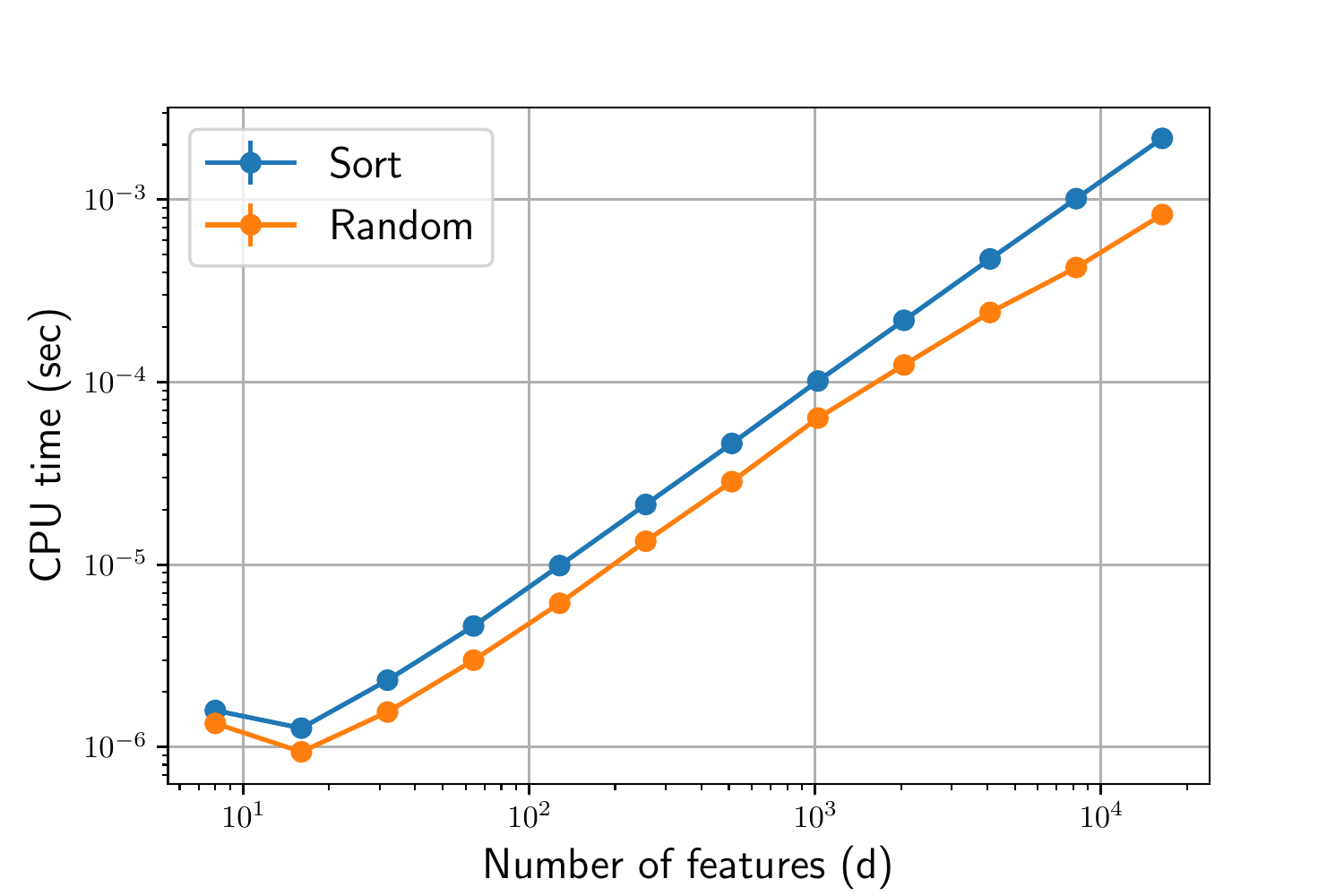}
        \includegraphics[width=52mm]{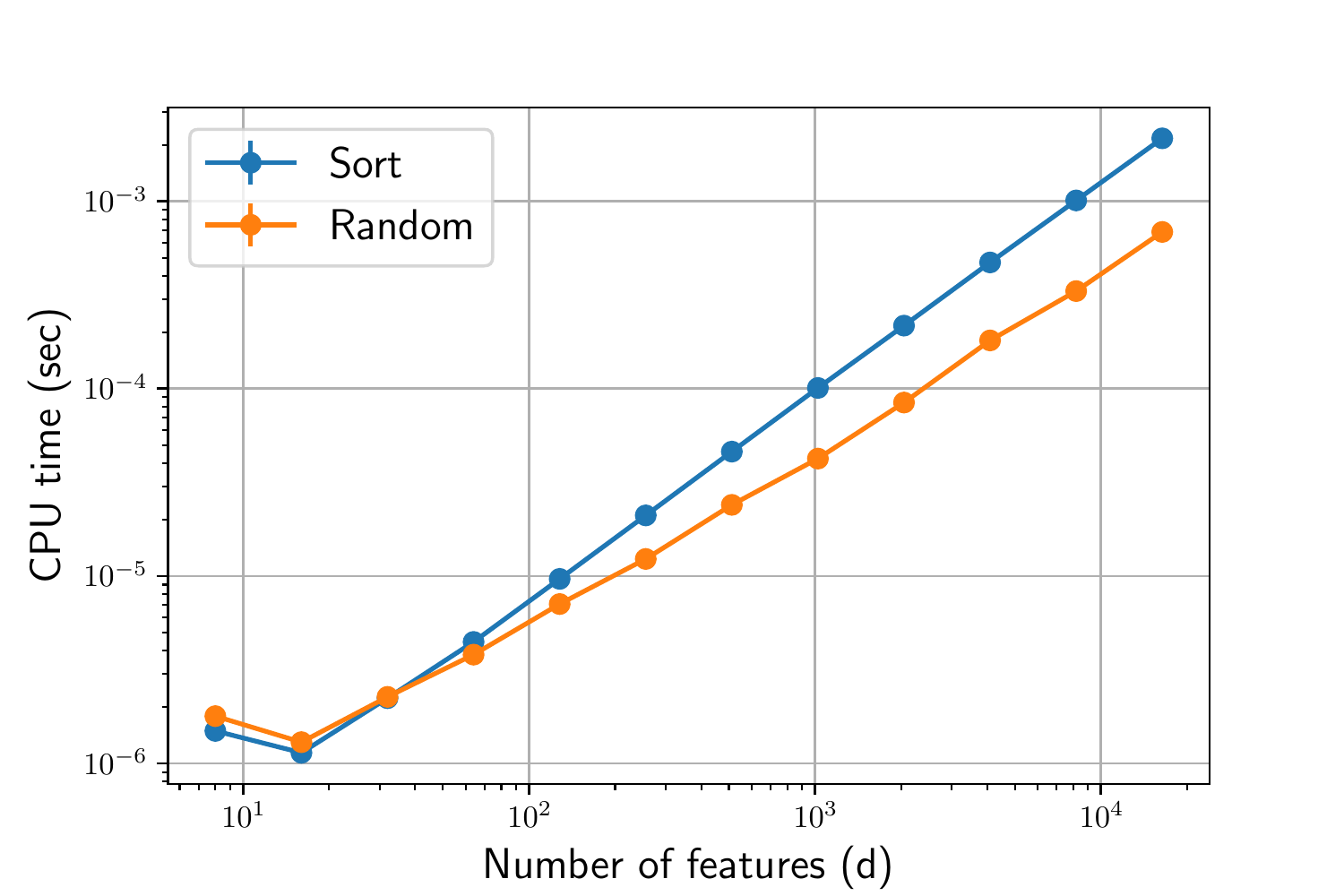}
        \includegraphics[width=52mm]{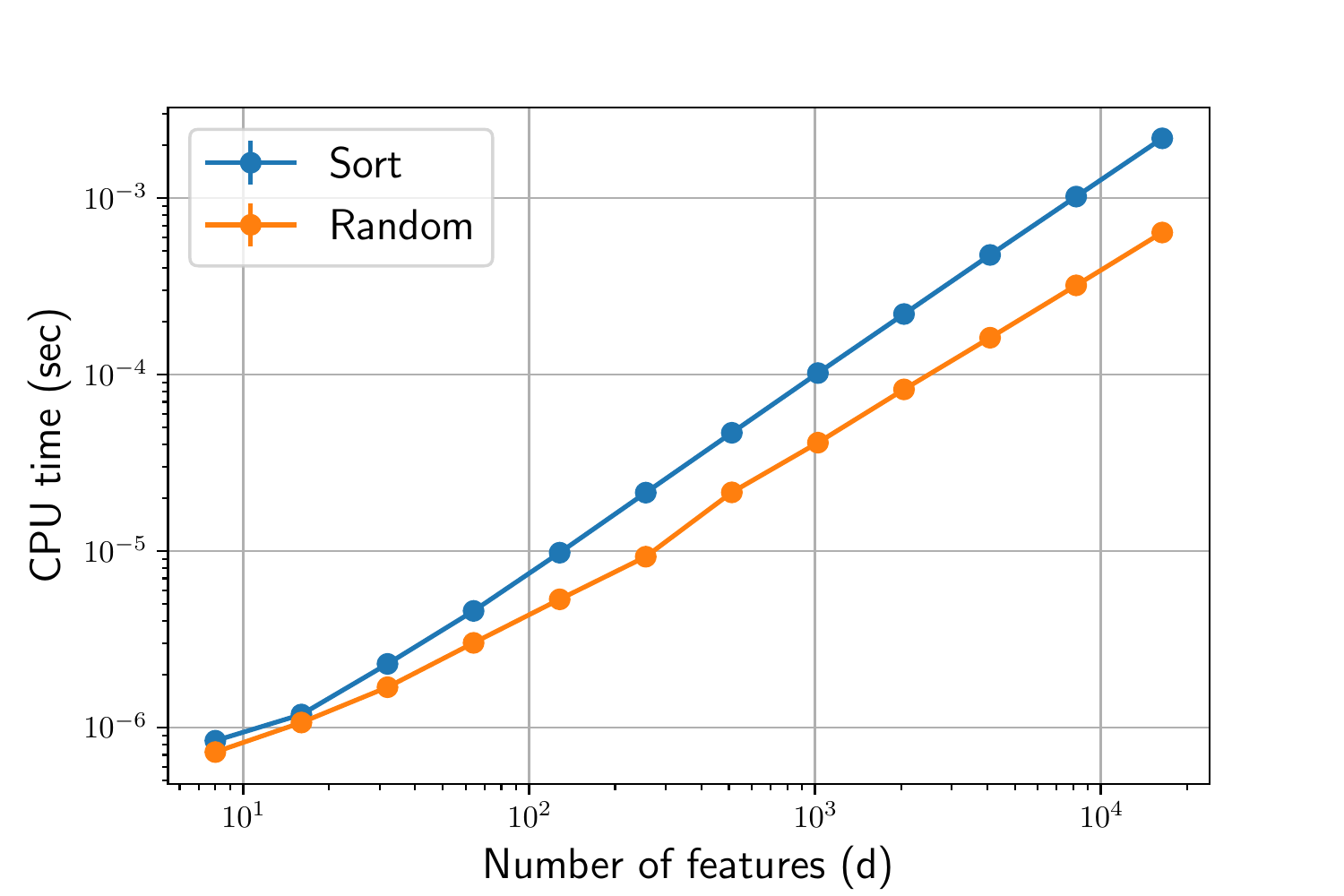}
        \label{fig:prox_sort_vs_random_std10.0}
        }
        \caption{Comparison of~\cref{alg:prox_ti_sort} (\textbf{Sort}) and~\cref{alg:prox_ti_randomize} (\textbf{Random}) with $\lambda=0.001$ (left), $\lambda=0.1$ (center), and $\lambda=10.0$ (right).}
        \label{fig:prox_sort_vs_random}
    \end{figure}
    As shown in~\cref{fig:prox_sort_vs_random}, \textbf{Random} outperformed \textbf{Sort} in most cases.
    \textbf{Sort} sometimes ran faster than \textbf{Random} only when $d$ was very small ($d \le 2^5 =32$).
    Thus, we basically recommend \textbf{Random} (\cref{alg:prox_ti_randomize}) rather than \textbf{Sort} (\cref{alg:prox_ti_sort}) for the evaluation of the proximal operator~\eqref{eq:prox_ti}.

\end{document}